\documentclass[10pt,journal,compsoc]{IEEEtran}
\usepackage{amssymb,amsthm,amsmath}
\usepackage{booktabs}
\usepackage{bm,bbm}
\usepackage{cite}
\usepackage{color}
\usepackage{graphicx}
\usepackage{hyperref}
\usepackage{latexsym}
\usepackage{multirow}
\usepackage{newtxmath}
\usepackage{overpic}
\usepackage{url}
\newtheorem{definition}{Definition}
\newtheorem{assumption}{Assumption}
\newtheorem{theorem}{Theorem}
\newtheorem{proposition}{Proposition}
\newtheorem{lemma}{Lemma}
\newtheorem{claim}{Claim}

\DeclareMathOperator{\diag}{diag}
\DeclareMathOperator{\Conv}{Conv}
\DeclareMathOperator{\cl}{cl}

\newcommand{\bq}{{\bm{q}}}
\newcommand{\bu}{{\bm{u}}}
\newcommand{\bv}{{\bm{v}}}
\newcommand{\bx}{{\bm{x}}}
\newcommand{\by}{{\bm{y}}}
\newcommand{\bz}{{\bm{z}}}
\newcommand{\bw}{{\bm{w}}}
\newcommand{\bA}{{\mathbf{A}}}
\newcommand{\bB}{{\mathbf{B}}}
\newcommand{\bI}{{\mathbf{I}}}
\newcommand{\bS}{{\mathbf{S}}}
\newcommand{\bG}{{\mathbf{G}}}
\newcommand{\calC}{{\mathcal{C}}}
\newcommand{\calG}{{\mathcal{G}}}

\newcommand{\calS}{{\mathcal{S}}}
\newcommand{\calT}{{\mathcal{T}}}
\newcommand{\calU}{{\mathcal{U}}}

\newcommand{\calV}{{\mathcal{V}}}
\newcommand{\calN}{{\mathcal{N}}}
\newcommand{\calP}{{\mathcal{P}}}
\newcommand{\calQ}{{\mathcal{Q}}}
\newcommand{\calR}{{\mathcal{R}}}
\newcommand{\bbI}{{\mathbbm{1}}}
\newcommand{\bbN}{{\mathbb{N}}}
\newcommand{\bbR}{{\mathbb{R}}}
\newcommand{\hyt}[1]{{\rm(\hypertarget{#1}{#1})}}
\newcommand{\hyl}[1]{{\rm(\hyperlink{#1}{#1})}}

\begin{document}
\title{Convergence Analysis of Blurring Mean Shift}
\author{Ryoya Yamasaki, and Toshiyuki Tanaka, \IEEEmembership{Member, IEEE}%
\IEEEcompsocitemizethanks{\IEEEcompsocthanksitem%
This work has been submitted to the IEEE for possible publication. Copyright may be transferred without notice, after which this version may no longer be accessible.}}
\markboth{PREPRINT VERSION}
{Shell \MakeLowercase{\textit{et al.}}: Convergence Analysis of Blurring Mean Shift}
\IEEEtitleabstractindextext{%
\begin{abstract}
Blurring mean shift (BMS) algorithm, a variant of the mean shift algorithm,
is a kernel-based iterative method for data clustering,
where data points are clustered according to
their convergent points via iterative blurring.
In this paper, we analyze convergence properties of the BMS algorithm
by leveraging its interpretation as an optimization procedure, 
which is known but has been underutilized in existing convergence studies.
Whereas existing results on convergence properties 
applicable to multi-dimensional data only cover the case 
where all the blurred data point sequences converge to a single point, 
this study provides a convergence guarantee 
even when those sequences can converge to multiple points, 
yielding multiple clusters. 
This study also shows that the convergence of the BMS algorithm is fast
by further leveraging geometrical characterization of the convergent points.
\end{abstract}
}
\maketitle
\IEEEdisplaynontitleabstractindextext
\IEEEpeerreviewmaketitle
\IEEEraisesectionheading{%
\section{Introduction}
\label{sec:Intro}}
\IEEEPARstart{B}{lurring} mean shift (BMS) algorithm 
\cite{fukunaga1975estimation, cheng1995mean, Cheng95, 
carreira2006fast, carreira2008generalised, chen2015convergence} 
(also called dynamic mean shift \cite{zhang2006accelerated}), 
a variant of the mean shift (MS) algorithm \cite{cheng1995mean, comaniciu2002mean}, 
is a kernel-based iterative method for data clustering.
In this study, we analyze convergence properties of the BMS algorithm.

The MS algorithm has been widely used 
for data clustering \cite{chacon2018mixture}, 
image segmentation \cite{comaniciu2002mean}, 
edge detection \cite{guo2005mean}, 
object tracking \cite{Coma2003OT}, and 
mode estimation \cite{yamasaki2023optimal}, to mention a few.
It essentially seeks a local maximizer (mode) of 
the kernel density estimate (KDE) \cite{silverman1986density}:
for $n$ given data points $\bx_1,\ldots,\bx_n\in\bbR^d$, 
the KDE with a kernel $K:\bbR^d\to\bbR$ 
and a bandwidth $h\in(0,\infty)$ is given by 
\begin{align}
	P(\bu)\coloneq\frac{1}{n h^d}\sum_{i=1}^n K\biggl(\frac{\bu-\bx_i}{h}\biggr).
\end{align}
Assume that the kernel $K$, at any point of its differentiability, satisfies 
\begin{align}
  \label{eq:KG}
	\nabla K(\bu)=-\bu G(\bu)
\end{align}
with another function $G:\bbR^d\to\bbR$. 
Motivated by the stationarity condition of the KDE $P$,
\begin{align}
	\begin{split}
	&\nabla P(\bu)
	=-\frac{1}{n h^{d+2}}\sum_{i=1}^n(\bu-\bx_i)G\biggl(\frac{\bu-\bx_i}{h}\biggr)
	=\bm{0}_d
\\
	&\Rightarrow
	\bu
	=\frac{\sum_{i=1}^n G\bigl(\frac{\bu-\bx_i}{h}\bigr)\bx_i}
	{\sum_{i=1}^n G\bigl(\frac{\bu-\bx_i}{h}\bigr)}
	\end{split}
\end{align}
with the all-0 vector $\bm{0}_d\in\bbR^d$,
the MS algorithm given an initial estimate $\by_1\in\bbR^d$
iterates the update rule (fixed-point iteration)
\begin{align}
\label{eq:updaMS}
	\by_{t+1}
	=\frac{\sum_{j=1}^n G\bigl(\frac{\by_t-\bx_j}{h}\bigr)\bx_j}
	{\sum_{j=1}^n G\bigl(\frac{\by_t-\bx_j}{h}\bigr)}
\end{align}
for $t\in\bbN$.
For example, MS-based data clustering runs the MS algorithm $n$ times, 
each with an initial estimate $\by_1=\bx_i$, $i\in[n]\coloneq\{1,\ldots,n\}$.
Convergence properties of the MS algorithm have been well studied:
while the MS algorithm typically achieves 
the linear convergence when it converges, 
\cite{comaniciu1999mean, Huang2018} proved 
that the algorithm with the Epanechnikov kernel 
$K(\bu)\propto(1-\|\bu\|^2)_+$ converges in a finite number of iterations,
where $(\cdot)_+\coloneq\max\{\cdot,0\}$,
and \cite{yamasaki2019ms, yamasaki2023ms} clarified mild sufficient conditions 
on the kernel $K$ for the sequence $(\by_t)_{t\in\bbN}$ 
to converge to a stationary point of the corresponding KDE $P$, 
and provided a detailed study of the convergence rate.

The BMS algorithm has also been used for data clustering 
\cite{cheng1995mean, carreira2008generalised, meila2010exponential}
and derivative applications such as
image segmentation \cite{carreira2006fast, vsurkala2011hierarchical},
motion segmentation \cite{zhang2009two}, 
object tracking \cite{bradski1998computer},
and data denoising \cite{wang2010manifold}.
The MS algorithm \eqref{eq:updaMS} keeps the data points 
$\bx_1,\ldots,\bx_n$ fixed during the iterations, 
whereas the BMS algorithm \emph{blurs} data points:
the BMS algorithm for data clustering iterates the update rule 
\begin{align}
\label{eq:updaBMS}
	\by_{t+1,i}
	=\frac{\sum_{j=1}^n G\bigl(\frac{\by_{t,i}-\by_{t,j}}{h}\bigr)\by_{t,j}}
	{\sum_{j=1}^n G\bigl(\frac{\by_{t,i}-\by_{t,j}}{h}\bigr)},
	\quad\forall i\in[n]
\end{align}
for $t\in\bbN$, with the initialization 
\begin{align}
\label{eq:initBMS}
	\by_{1,i}=\bx_i,\quad\forall i\in[n],
\end{align}
where the right-hand side of the BMS update rule~\eqref{eq:updaBMS} uses
\emph{blurred data points} $\by_{t,1},\ldots,\by_{t,n}$ in place of
the original data points $\bx_1,\ldots,\bx_n$ in the MS update rule~\eqref{eq:updaMS}. 
In applications of the BMS algorithm, it is expected that, 
as the number $t$ of iterations tends to infinity, 
the blurred data points $\by_{t,1},\ldots,\by_{t,n}$ 
obtained by the BMS algorithm form multiple clusters,
$\{\bx_i\mid\lim_{t\to\infty}\by_{t,i}=\bz_m\}_{i\in[n]}$ for $m=1,\ldots,M$
with $M\ge2$ different points $\bz_1,\ldots,\bz_M$.
Considering such usage, it is important to clarify conditions 
under which the \emph{blurred data point sequences}
$(\by_{t,1})_{t\in\bbN},\ldots,(\by_{t,n})_{t\in\bbN}$ converge, 
which are applicable to cases where these sequences converge to multiple points.
However, as far as the authors know, there are no 
existing results for multi-dimensional settings on this issue.

The main contributions of this paper are that 
we prove the convergence of the BMS algorithm 
for a wide class of functions $G$ 
with convergence rate bounds, which improve upon existing ones 
and further include, as a special case, the finite-time convergence of the algorithm
for the truncated-flat function $G(\bu)\propto\bbI(\|\bu\|\le1)$,
where $\bbI(c)$ is 1 if the condition $c$ is true and 0 otherwise.
These results cover the case where the blurred data point sequences 
$(\by_{t,1})_{t\in\bbN},\ldots,(\by_{t,n})_{t\in\bbN}$ can converge to 
multiple points, and 
suggest that the BMS algorithm typically allows 
for more efficient data clustering than the MS algorithm.

The rest of this paper is organized as follows.
Section~\ref{sec:Preparation} introduces several assumptions
used throughout this paper, as well as the BMS graph and related notions,
which turn out to be quite useful in analyzing convergence properties
of the BMS algorithm. 
Section~\ref{sec:Review} provides a brief review of previous work
on convergence analysis of the BMS algorithm. 
Section~\ref{sec:Analysis} presents the main results of this paper.
Proofs of theorems and propositions are described in Section~\ref{sec:Proof}. 
Section~\ref{sec:Conclusion} concludes this paper.
Supplementary material presents proofs of theoretical results
not covered in Section~\ref{sec:Proof},
as well as supplementary description regarding 
the review of previous work.

\section{Preparation}
\label{sec:Preparation}
\subsection{Basic Assumptions on Kernel}
\label{sec:BA-Ker}
It turns out that introduction of the `kernel' $K$
that is related with $G$ via \eqref{eq:KG} 
(or more precisely, \eqref{eq:funcG} below)
is helpful for the analysis in this paper,
although $K$ itself does not appear in the BMS algorithm at all.%
\footnote{%
The kernel $K$ is sometimes called the \emph{shadow} kernel of $G$
in the context of the MS algorithm~\cite{cheng1995mean}.}
We begin by introducing the kernel $K$ along with basic assumptions on it. 
Let $\|\cdot\|$ denote the Euclidean norm and
$f'(u\pm)\coloneq\lim_{v\to u\pm0}\frac{f(v)-f(u)}{v-u}$ denote 
the right and left derivatives of the function $f:[0,\infty)\to\bbR$ at $u\in[0,\infty)$ if exist. 
We make the following assumptions throughout the paper, unless stated otherwise: 
\begin{assumption}
\label{asm:RS}%
The kernel $K:\bbR^d\to\bbR$ is bounded, continuous, non-negative, normalized,
and radially symmetric so that it can be represented as $K(\bu)=k(\|\bu\|^2/2)$
for any $\bu\in\bbR^d$ with what is called the \emph{profile} $k:[0,\infty)\to\bbR$ of $K$.
Also, the profile $k$ is convex and non-increasing, and satisfies $k'(0+)>-\infty$.
\end{assumption}
These assumptions exclude, for example, 
the uniform kernel $K(\bu)\propto\bbI(\|\bu\|\le1)$ 
and the tricube kernel $K(\bu)\propto\{(1-\|\bu\|^3)_+\}^3$ 
from our consideration.
However, many commonly-used kernels satisfy Assumption~\ref{asm:RS};
see Table~\ref{tab:Kernel}.

Under Assumption~\ref{asm:RS}, we define the function $G$ 
appearing in the BMS update rule \eqref{eq:updaBMS} as 
\begin{align}
\label{eq:funcG}
	G(\bu)=g\biggl(\frac{\|\bu\|^2}{2}\biggr)
	\text{ with }
	g(u)
	\begin{cases}
	= -k'(0+)&\text{for }u=0,\\
	\in-\partial k(u)&\text{for }u>0,
	\end{cases}
\end{align}
with the profile $k$ of the kernel $K$,
where $\partial k(u)$ is the subdifferential of $k$ at $u\in[0,\infty)$ 
that is defined as the set of values $c\in\bbR$ satisfying 
$k(u)-k(v)\ge c(u-v)$ for any $v\in[0,\infty)$.
This definition is a generalization of \eqref{eq:KG}
that allows for a non-differentiable kernel with a subdifferentiable profile.

These functions have the following properties:
\begin{proposition}[{\cite[Section 24]{rockafellar1997convex}}]
\label{prop:KG}
Assume Assumption~\ref{asm:RS}.
Then, for the profile $k$ of the kernel $K$ and the function 
$g$ in \eqref{eq:funcG}, one has 
\begin{itemize}
\item[\hyt{a1}]
$\partial k(u)\neq\emptyset$ for any $u\in[0,\infty)$,
and hence $g$ is well-defined.
\item[\hyt{a2}]
$\partial k(0)=(-\infty,k'(0+)]$, and 
$\partial k(u)=[k'(u-),k'(u+)]$ for any $u\in(0,\infty)$.
\item[\hyt{a3}]
$\partial k(0)\not\ni0$,
$k(u)<k(0)$ for $u\in(0,\infty)$,
and $g(0)\in(0,\infty)$.
\item[\hyt{a4}]
$\partial k(u)$ is non-decreasing in the sense that 
$\max\partial k(u)\le\min\partial k(v)$ holds for any $u, v\in[0,\infty)$ such that $u<v$,
and hence $g$ is non-increasing.
\item[\hyt{a5}]
$\max\partial k(u)\le0$ for any $u\in[0,\infty)$,
and hence $g$ is non-negative.
\item[\hyt{a6}]
$k(u)-k(v)\ge -g(v)(u-v)$ for any $u,v\in[0,\infty)$.
\item[\hyt{a7}]
$k$ is Lipschitz-continuous, i.e., 
there exists a constant $c\in(0,\infty)$ such that 
$|k(u)-k(v)|\le c|u-v|$ for any $u,v\in[0,\infty)$
(one can let $c=-k'(0+)$).
\item[\hyt{a8}]
$\{u\in[0,\infty)\mid k(u)>0\}=
\{u\in[0,\infty)\mid \partial k(u)\not\ni0\}\subseteq
\{u\in[0,\infty)\mid g(u)>0\}$,
and hence $k(u)=0$ must hold for $g(u)=0$ to hold,
where the equality holds instead of $\subseteq$ 
if either $k(u)>0$ for any $u\in[0,\infty)$
or $k$ is differentiable
at $\min\{u\in[0,\infty)\mid k(u)=0\}$.
\end{itemize}
\end{proposition}

We further introduce the notion of the truncated kernel 
and those related to how it is truncated:
\begin{definition}[{Truncated kernel and related notions}]
\label{def:TK}
Assume that the kernel $K$ is radially symmetric with the profile $k$.
\begin{itemize}
\item
The \emph{truncation point} $\beta$ of $K$ is defined as 
$\beta\coloneq\sup\{u\in[0,\infty)\mid k(u^2/2)\neq 0\}\in(0,\infty]$.
\item
A \emph{truncated kernel} is a kernel with $\beta<\infty$, 
and a \emph{non-truncated kernel} is a kernel with $\beta=\infty$.
\item
A \emph{smoothly truncated kernel} is a truncated kernel with profile $k$ 
for which $k(u^2/2)$ is differentiable at $u=\beta$. 
A \emph{non-smoothly truncated kernel} is a truncated kernel 
for which $k(u^2/2)$ is not differentiable at $u=\beta$. 
\end{itemize}
\end{definition}
For a truncated kernel $K$ satisfying 
Assumption~\ref{asm:RS}, 
one has $\{u\in[0,\infty)\mid k(u^2/2)>0\}=[0,\beta)$
and $\{u\in[0,\infty)\mid k(u^2/2)=0\}=[\beta,\infty)$
for its profile $k$ and truncation point $\beta$.

\subsection{Graph-Theoretic Representation}
\label{sec:Term-DC}
In this subsection, we introduce several notions and terms 
that assist our analysis of the BMS algorithm.

We first define the BMS graph, a graph-theoretic representation of 
the state of the blurred data points $\by_{t,1},\ldots,\by_{t,n}$.
\begin{definition}[Configuration space]
  A collection of $n$ points $\bu_1,\ldots,\bu_n$
  in $\bbR^d$ defines a \emph{configuration}
  $\bu=(\bu_1^\top,\ldots,\bu_n^\top)^\top$
  in the \emph{configuration space} $\bbR^{nd}$.
\end{definition}
\begin{definition}[BMS graph]
\label{def:CGC}
Assume that a configuration
$\bu=(\bu_1^\top,\ldots,\bu_n^\top)^\top\in\bbR^{nd}$,
a function $G:\bbR^d\to\bbR$ which is radially symmetric,
and a bandwidth $h\in(0,\infty)$ are specified.
Then, the \emph{BMS graph} $\calG_\bu$ of $\bu$ is defined as
the (undirected, unweighted, and simple) graph consisting of 
the vertex set $[n]$ and the edge set
\begin{align}
	\mathcal{E}\coloneq\left\{\{i,j\}\in[n]^2\Bigm| i\not=j\ 
	\mbox{and}\ G\left(\frac{\bu_i-\bu_j}{h}\right)\not=0\right\}.
\end{align}
We say that vertices $i$ and $j$ are \emph{joined} in $\calG_\bu$ if and only if $\{i,j\}\in\mathcal{E}$.
\end{definition}
When $G$ is defined from a kernel $K$ satisfying Assumption~\ref{asm:RS}
via \eqref{eq:funcG}, a BMS graph 
can be regarded as a unit ball graph \cite{ClarkColbournJohnson1990},
which is an intersection graph of equal-radius balls. 
It has been studied in relation to
various applications such as modeling of 
wireless communication networks \cite{huson1995broadcast}
and continuum percolation theory \cite{MeesterRoy1996}.
We consider the BMS graph $\calG_{\by_t}$ of the configuration 
$\by_t\coloneq(\by_{t,1}^\top,\ldots,\by_{t,n}^\top)^\top\in\bbR^{nd}$ 
generated by the BMS algorithm.

We further introduce the notions `closed/open', 
`singular/non-singular', and `stable/unstable',
which define an important categorization of the BMS graph. 
In the following, we use common graph-theoretic
notions such as a subgraph, a component (also called a maximal connected subgraph),
and a complete graph in which all the distinct vertices are joined \cite{bondy1976graph}. 
\begin{definition}[{Characterization of the BMS graph}]
\label{def:CIS}
Given a configuration $\bu\in\mathbb{R}^{nd}$, we say that:
\begin{itemize}
\item
The BMS graph $\calG_\bu$ of $\bu$ is \emph{closed}
when all of its components are complete. 
Otherwise, it is \emph{open}.
\item
A vertex $i\in[n]$ is \emph{isolated} in $\calG_\bu$
when $i$ is not joined to any $j\in[n]$ with $\bu_i\neq\bu_j$.
\item
The BMS graph $\calG_\bu$ is \emph{singular}
when all of its vertices are isolated in $\calG_\bu$.
Otherwise, it is \emph{non-singular}.
\item
The BMS graph $\calG_\bu$ is \emph{stable}
if there exists an open neighborhood $\calT\subseteq\bbR^{nd}$ of $\bu$
such that the BMS graphs $\calG_\bu$ and $\calG_{\bu'}$
are the same for any $\bu'\in\calT$.
Otherwise, it is \emph{unstable}.
\end{itemize}
\end{definition}
It is obvious from the definition
that a singular BMS graph is closed. 
As Figure~\ref{fig:Graph} demonstrates,
the BMS graph $\calG_\bu$ can be categorized into 
one of `singular', `closed and non-singular', or `open'. 
On the other hand, the notion `stable/unstable' refers to the stability
of $\calG_\bu$ against small perturbations of $\bu$.

The following proposition is an immediate consequence of the definition
of the BMS graph and the BMS update rule. 
\begin{proposition}
\label{prop:BMScomp}
For a \emph{configuration sequence} $(\by_t)_{t\in\bbN}$ generated by the BMS algorithm, 
let $\calG_{\by_t,*}$ be a component of the BMS graph $\calG_{\by_t}$,
and let $\calV_{\by_t,*}$ be its vertex set.
One then has
\begin{align}
	\by_{t+1,i}
	=\frac{\sum_{j\in\calV_{\by_t,*}}G_{t,i,j}\by_{t,j}}{\sum_{j\in\calV_{\by_t,*}}G_{t,i,j}},
	\quad \forall i\in\calV_{\by_t,*}
\end{align}
with $G_{t,i,j}\coloneq G(\frac{\by_{t,i}-\by_{t,j}}{h})$.
Moreover, assuming Assumption~\ref{asm:RS} and
letting $\calP\coloneq\{\by_{t,i}\}_{i\in\calV_{\by_t,*}}$
and $\calP'\coloneq\{\by_{t+1,i}\}_{i\in\calV_{\by_t,*}}$,
one has $\Conv(\calP)\supseteq\Conv(\calP')$, 
where $\Conv(\calS)$ for a set $\calS\subseteq\bbR^d$ is the convex hull of $\calS$.
\end{proposition}
In other words, a single BMS update of a blurred data point
with its index in a component of the BMS graph $\calG_{\by_t}$ 
is affected only by those blurred data points with indices
belonging to the same component, and the convex hull
of the blurred data points with indices in a component
`shrinks' by a single BMS update.

Given a BMS graph $\calG_\bu$,
we let $M_\bu$ denote the number of components of $\calG_\bu$, 
let $\calG_{\bu,1}, \ldots,\calG_{\bu,M_\bu}$ be those components, 
and let $\calV_{\bu,1},\ldots,\calV_{\bu,M_\bu}$
be the vertex sets of those components. 
The number $M_\bu$ of components has the following upper bound:
\begin{proposition}[{Upper bound of the number of components of the BMS graph}]
\label{prop:Number}
Assume Assumption~\ref{asm:RS} 
and that the kernel $K$ has the truncation point $\beta\in(0,\infty]$. 
Then, for $\bu\in\bbR^{nd}$ and
$\gamma=\max\{\|\bu_i-\bu_j\|\}_{i,j\in[n]}$,
$M_\bu\le\min\{n,(1+\frac{2\gamma}{\beta h})^d\}$.
\end{proposition}

\begin{figure}
\centering%
\begin{tabular}{cccc}
{\hskip-4pt}\begin{overpic}
	[height=2cm, bb=0 0 173 159]{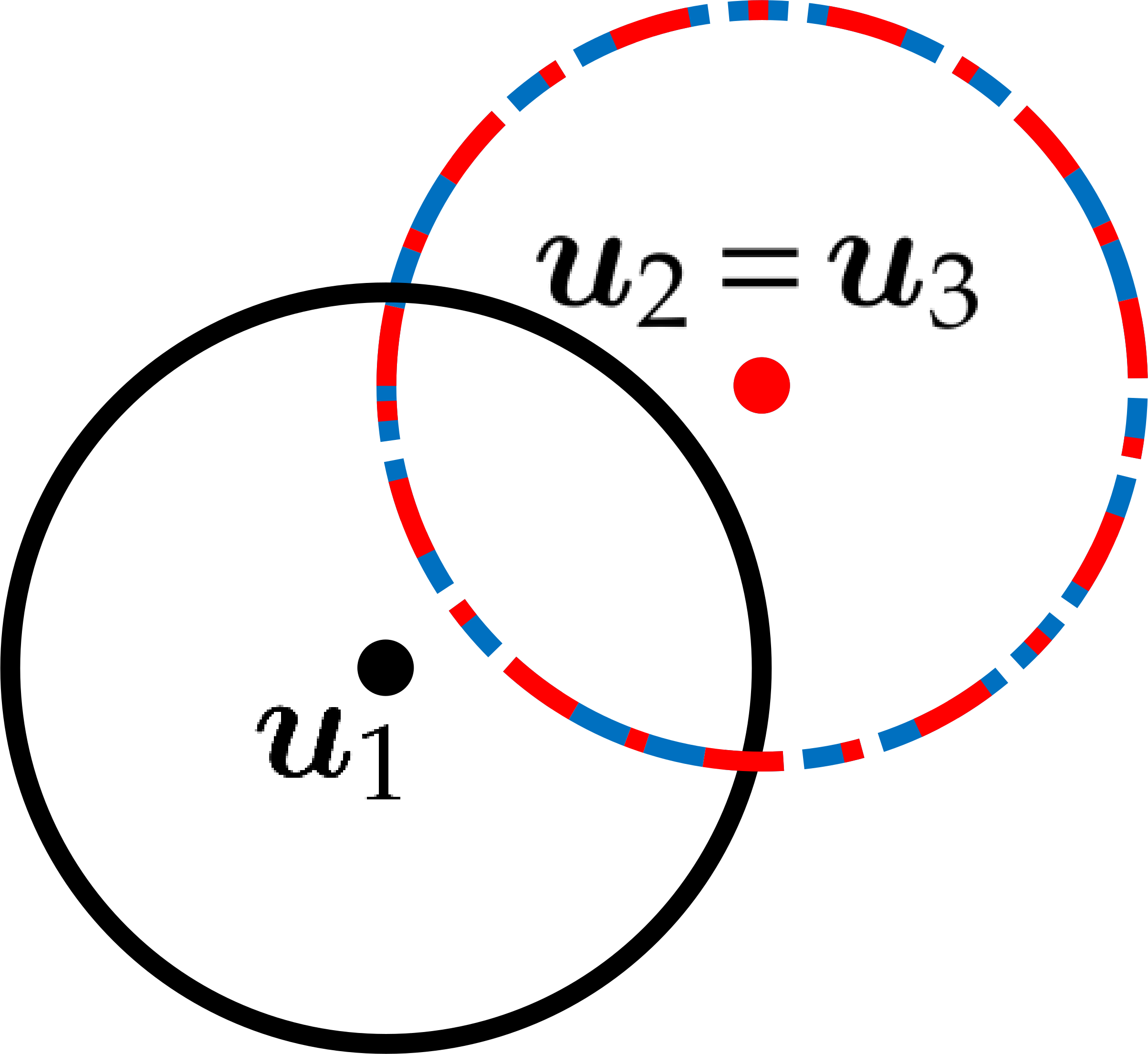}\put(1,85){\scriptsize\hyt{i}}
\end{overpic}&{\hskip-11pt}\begin{overpic}
	[height=2cm, bb=0 0 187 159]{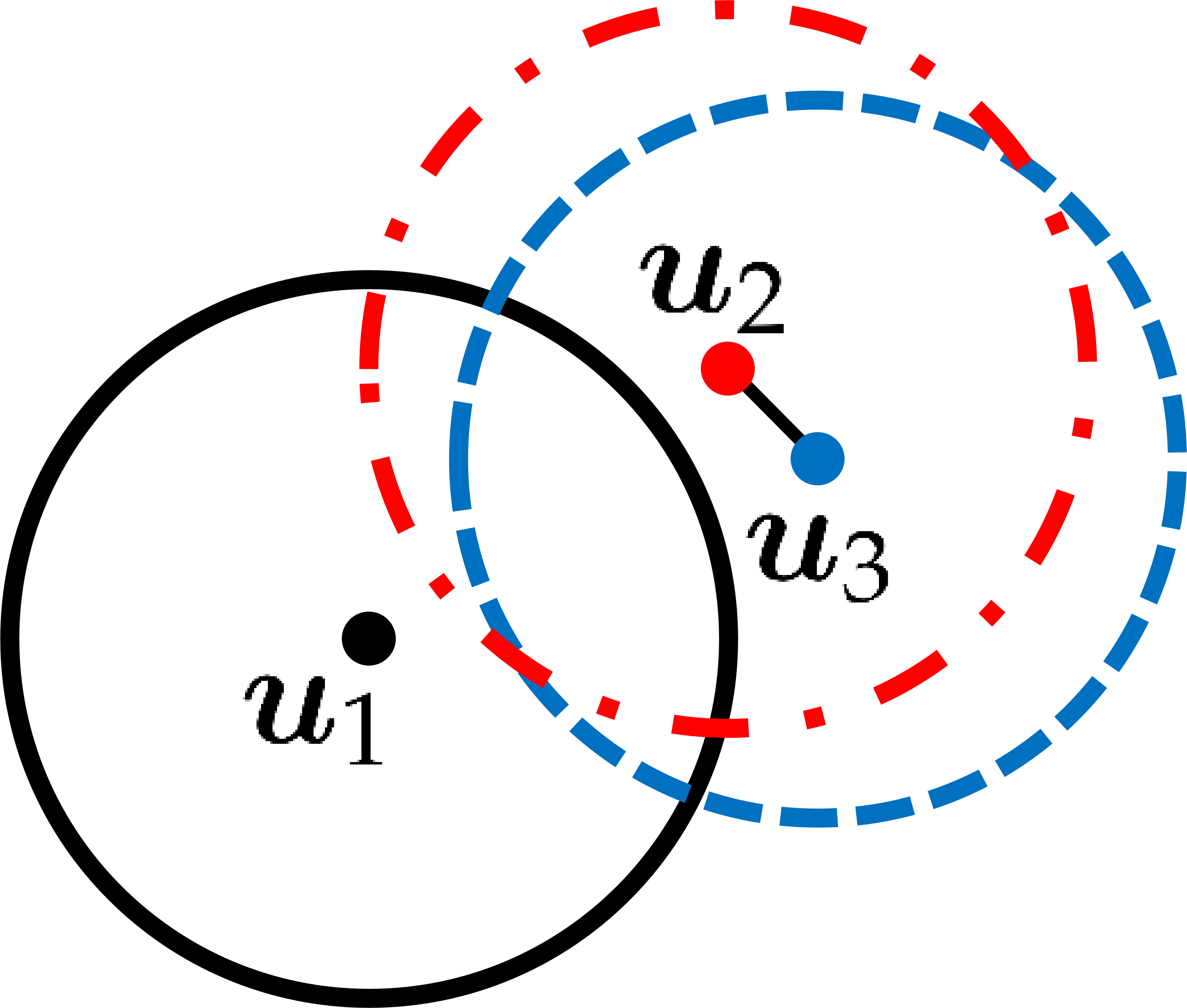}\put(1,78.64){\scriptsize\hyt{ii}}
\end{overpic}&{\hskip-11pt}\begin{overpic}
	[height=2cm, bb=0 0 176 159]{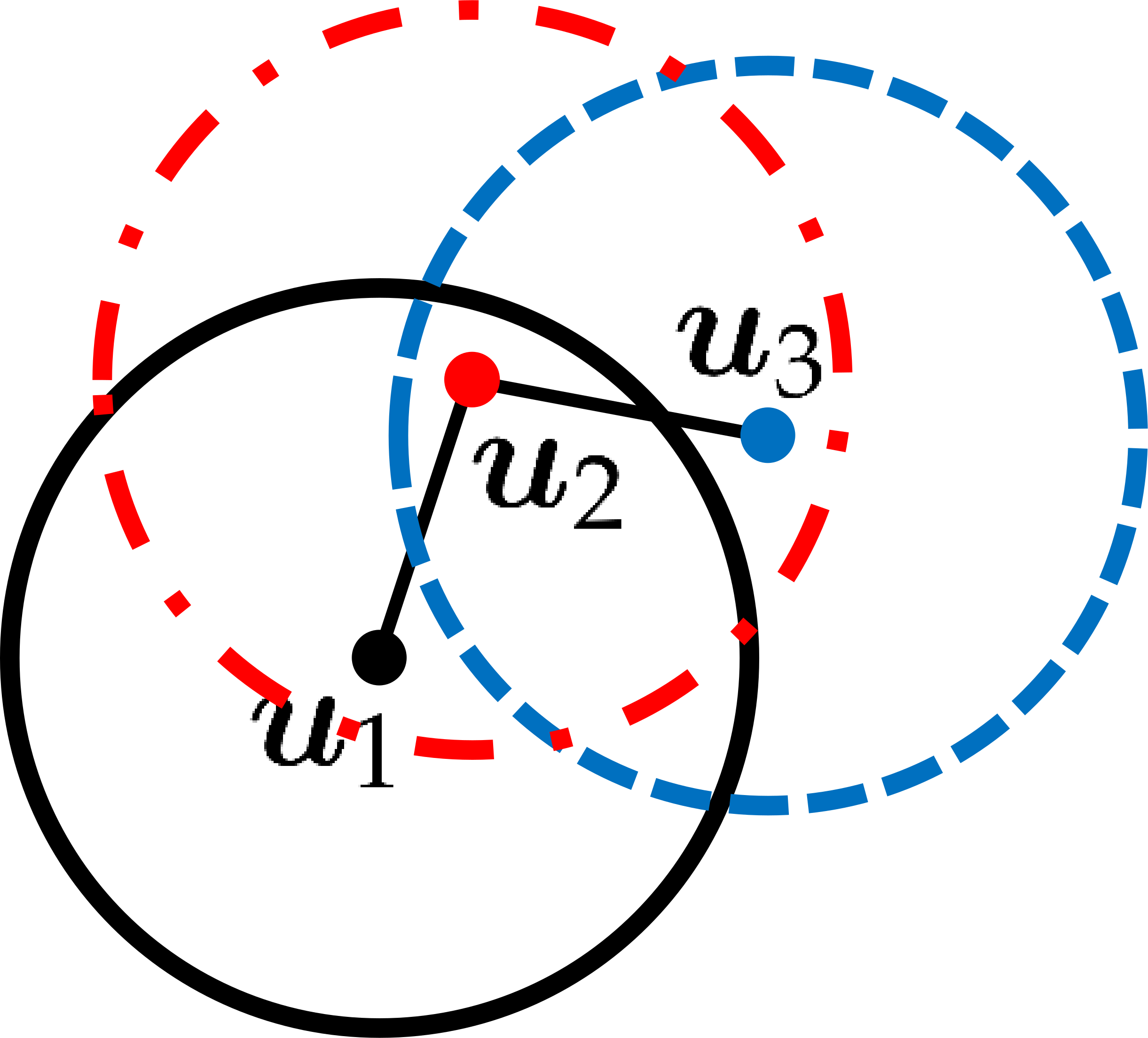}\put(1,83.55){\scriptsize\hyt{iii}}
\end{overpic}&{\hskip-11pt}\begin{overpic}
	[height=2cm, bb=0 0 154 159]{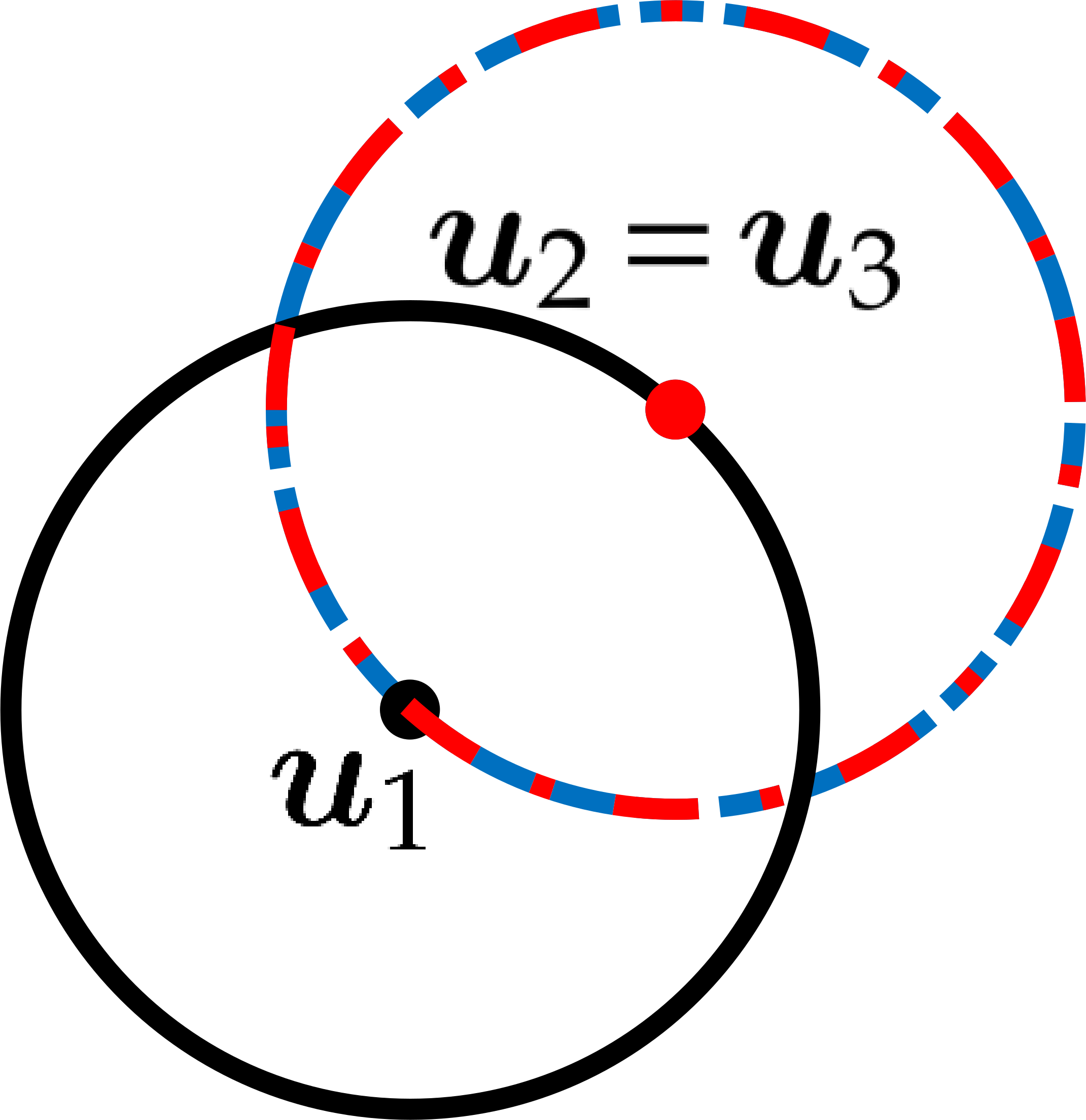}\put(1,89.87){\scriptsize\hyt{iv}}
\end{overpic}
\end{tabular}
\caption{%
Illustration of the BMS graphs.
Each dot represents $\bu_i\in\bbR^2$, 
and each circle represents the outer edge of the ball 
$\{\bv\in\bbR^2\mid G(\frac{\bu_i-\bv}{h})>0\}$, with $i\in[3]$.
For example, the BMS graph $\calG_\bu$ of $\bu=(\bu_i)_{i\in[3]}$
in \protect\hyl{ii} is such that
$2$ and $3$ are joined and that $1$ is isolated. 
%
Also, $\calG_\bu$ is 
`singular' in \protect\hyl{i} and \protect\hyl{iv}, 
`closed and non-singular' in \protect\hyl{ii}, 
`open' in \protect\hyl{iii},
stable in \protect\hyl{i}--\protect\hyl{iii}, and 
unstable in \protect\hyl{iv}.}
\label{fig:Graph}
\end{figure}

\section{Review of Previous Work}
\label{sec:Review}
\subsection{Convergence Guarantee to a Single Point}
\label{sec:Rev-Th3}
In Section~\ref{sec:Review}, we review existing 
convergence analysis of the BMS algorithm.
Although the previous studies 
\cite{cheng1995mean, Cheng95,zhang2006accelerated, carreira2006fast, 
carreira2008generalised, chen2015convergence}%
\footnote{%
The previous studies \cite{cheng1995mean, chen2015convergence} 
discussed the BMS algorithm under a weighted formulation 
that we also mention in Section~\ref{sec:Conclusion}.
In this paper, including the review in Section~\ref{sec:Review}, 
we proceed with the discussion under the non-weighted formulation, 
in favor of notational simplicity.}
have tackled this issue,
those results that we consider important are 
\cite[Claim 3 and Theorem 3]{cheng1995mean},
\cite[Theorem 14.2]{Cheng95}, 
\cite{carreira2006fast}, and \cite{carreira2008generalised}.
This section describes reviews of these results, 
and reviews of other existing results 
are presented in the supplementary material.

In convergence analysis of the BMS algorithm,
an important existing result for multi-dimensional data is 
\cite[Claim 3 and Theorem 3]{cheng1995mean} that gives 
a condition under which the blurred data point sequences 
$(\by_{t,1})_{t\in\bbN},\ldots,(\by_{t,n})_{t\in\bbN}$ 
converge to a single point in the exponential rate:
\begin{theorem}[{Convergence guarantee and rate bound 
for non-truncated kernels or a large bandwidth;
\cite[Claim 3 and Theorem 3]{cheng1995mean}}]
\label{thm:Cheng}
Assume Assumption~\ref{asm:RS}.
Then, letting
\begin{align}
\label{eq:ab}
	a_{t,\bu}\coloneq\min\{\bu^\top\by_{t,i}\}_{i\in[n]},\quad
	b_{t,\bu}\coloneq\max\{\bu^\top\by_{t,i}\}_{i\in[n]}
\end{align}
with $\bu\in\bbR^d$, it holds that
\begin{align}
\label{eq:Claim1}
	[a_{1,\bu},b_{1,\bu}]\supseteq
	[a_{2,\bu},b_{2,\bu}]\supseteq\cdots
\end{align}
for any $\bu$, 
and that 
\begin{align}
\label{eq:Claim3}
	\Conv(\{\by_{1,i}\}_{i\in[n]})\supseteq
	\Conv(\{\by_{2,i}\}_{i\in[n]})\supseteq\cdots.
\end{align}
Also, for the diameter 
\begin{align}
\label{eq:defdia}
	d_t\coloneq\max_{\|\bu\|=1}(b_{t,\bu}-a_{t,\bu})
\end{align}
of the convex hull $\Conv(\{\by_{t,i}\}_{i\in[n]})$, it holds that
\begin{align}
\label{eq:diamrate}
	\frac{d_{t+1}}{d_t}
	\le 1-\frac{g((d_t/h)^2/2)}{4g(0)}
	\le 1-\frac{g((d_1/h)^2/2)}{4g(0)}
\end{align}
for any $t\in\bbN$.
Moreover, if $g((d_1/h)^2/2)>0$, then $(d_t)_{t\in\bbN}$ converges to zero,
and the blurred data point sequences $(\by_{t,1})_{t\in\bbN},\ldots,(\by_{t,n})_{t\in\bbN}$ 
converge to a single point $\bz_1\in\bbR^d$, in the exponential rate:
for $q=1-\frac{g((d_1/h)^2/2)}{4g(0)}\in[\frac{3}{4},1)$,
it holds that $d_t=O(q^t)$ and
$\|\by_t-\bz_1\otimes\bm{1}_n\|=O(q^t)$
(i.e., $\|\by_{t,i}-\bz_1\|=O(q^t)$ for all $i\in[n]$),
where $\otimes$ is the Kronecker product
and $\bm{1}_n\in\bbR^n$ is the all-1 vector.
\end{theorem}

The condition $g((d_1/h)^2/2)>0$ in the above theorem implies that 
the BMS graph $\calG_{\by_1}$ of the configuration $\by_1$
is closed with one complete component.
Under Assumption~\ref{asm:RS}, 
this condition holds either when the kernel $K$ is non-truncated, 
or when the bandwidth $h$ is sufficiently large
(even though $K$ is truncated). 
However, this theorem just gives a condition 
where the BMS algorithm outputs one cluster, 
and is not helpful for situations that we expect in applications of the BMS-based data clustering:
In the application point of view, 
convergence analysis of the BMS algorithm requires
consideration of situations where the blurred data point sequences 
$(\by_{t,1})_{t\in\bbN},\ldots,(\by_{t,n})_{t\in\bbN}$ 
can converge to multiple points.

\subsection{Conditional Convergence Guarantee}
\label{sec:Rev-Naive}
As argued briefly in \cite{cheng1995mean},
when a truncated kernel is used, 
one can expect that a well-isolated group of data points
will eventually converge to one point,
while the other data points may converge to different points. 
In the following theorem,
we provide, in terms of the BMS graph, 
a mathematically rigorous sufficient condition
for such a situation to happen.
\begin{theorem}[{Conditional convergence guarantee and rate bound;
extension of Theorem~\ref{thm:Cheng}}]
\label{thm:variant}
Assume Assumption~\ref{asm:RS} and 
that there exists some $\tau\in\bbN$ at which the BMS graph $\calG_{\by_\tau}$ is closed,
and take $M=M_{\by_\tau}$.
Then, for every $m\in[M]$, 
the blurred data point sequences $(\by_{t,i})_{t\in\bbN}, i\in\calV_{\by_\tau,m}$
converge to a single point in the exponential rate:
there exist $q\in[\frac{3}{4},1)$, $\bar{\by}\coloneq\lim_{t\to\infty}\by_t$, 
and $M$ different points $\bz_1,\ldots,\bz_M\in\bbR^d$ such that
$\|\by_t-\bar{\by}\|=O(q^t)$ and
$\|\by_{t,i}-\bz_m\|=O(q^t)$ for all $i\in\calV_{\by_\tau,m}$ for all $m\in[M]$.
\end{theorem}

In the previous studies \cite{cheng1995mean,chen2015convergence}
providing such arguments as the above, however,
it has not been discussed whether there exists $\tau\in\bbN$ at which
the BMS graph $\calG_{\by_\tau}$ is closed.

\subsection{One-Dimensional Convergence Guarantee}
\label{sec:CW95}
There is a convergence guarantee that is only applicable
to the one-dimensional case ($d=1$), while covering situations where 
the blurred data point sequences converge to multiple points:
\cite[Theorem 14.2]{Cheng95} states that 
the BMS algorithm based on the Epanechnikov kernel $K$
(or the truncated-flat function $G$)
converges in a finite number of iterations.

In the one-dimensional case with the Epanechnikov kernel,
they proved that the blurred data point sequences 
do not change their order through iterations.
In the same way as in Theorem~\ref{thm:Cheng}, one can prove 
that the blurred data point sequences starting from the outermost points
are bounded and monotone. 
Also, they proved that the distance with which each of the outermost points moves
in each iteration is either zero or larger than a positive constant 
that is independent of the number of iterations so far.
These results constitute the proof of the finite-time convergence of the outermost points.
By repeating the same considerations for 
the inner points except for the converged outermost points, 
one can prove the convergence of all the blurring data point sequences.

However, their proof strategy is restricted to 
the one-dimensional case with the Epanechnikov kernel. 
In particular, it should be difficult to extend the argument
based on the order-preserving property to the multi-dimensional case.

\subsection{Cubic Convergence of Gaussian Population}
\label{sec:MACP}
\cite{carreira2006fast, carreira2008generalised} showed 
that the BMS algorithm achieves the cubic convergence 
under a certain limited situation with three assumptions:%
\footnote{%
A similar analysis is also provided in \cite{zhang2006accelerated, chen2015convergence}.
However, \cite{zhang2006accelerated} proved only up to 
the point that the convergence rate is super-linear.}
the Gaussian data distribution, the Gaussian kernel,
and the population limit.

In the population limit $n\to\infty$,
the BMS update rule applied to a blurred data point $\bu_t\in\bbR^d$
can be written in terms of the density $P_t$ underlying 
the blurred data points $\by_{t,1},\ldots,\by_{t,n}$ as
\begin{align}
\label{eq:PV}
  \bu_t\mapsto\bu_{t+1}
  =\frac{\int P_t(\by) G(\frac{\bu_t-\by}{h}) \by\,d\by}
	{\int P_t(\by) G(\frac{\bu_t-\by}{h})\,d\by}.
\end{align}
Since the BMS algorithm is invariant under translation and rotation,
without loss of generality we may assume that 
the Gaussian population density at $t=1$ has
mean zero and a diagonal covariance matrix,
so that we 
let $P_1(\bu)=\calN(\bu; \bm{0}_d, \diag((s_{1,j}^2)_{j\in[d]}))$ with $s_{1,1},\ldots,s_{1,d}\in(0,\infty)$,
where $\diag((s_{1,j}^2)_{j\in[d]})$ is the $d\times d$ diagonal matrix 
in which the $j$th diagonal element is $s_{1,j}^2$ for $j\in[d]$.
With the Gaussian kernel assumption, we may let 
$G(\bu)=\calN(\bu; \bm{0}_d,\bI_d)$,
where $\bI_d$ is the $d\times d$ identity matrix.
These assumptions allow us to calculate \eqref{eq:PV} with $t=1$ as
\begin{align}
	\bu_2=\diag\biggl(\biggl(\frac{s_{1,j}^2}{s_{1,j}^2+h^2}\biggr)_{j\in[d]}\biggr)\bu_1.
\end{align}
Therefore, the population density $P_2$ also becomes Gaussian: 
$P_2(\bu)=\calN(\bu; \bm{0}_d, \diag((s_{2,j}^2)_{j\in[d]}))$
with $s_{2,j}\coloneq s_{1,j}^3/(s_{1,j}^2+h^2)$. 
Repeating the above calculation,
one can see that, for any $t$, $P_t$ is zero-mean Gaussian with a diagonal
covariance matrix. 
Letting the covariance matrix of $P_t$ be $\diag((s_{t,j}^2)_{j\in[d]})$,
one has $s_{t+1,j}=s_{t,j}^3/(s_{t,j}^2+h^2)$.
This relation shows that $(s_{t,j})_{t\in\bbN}$ is decreasing and converges to 0
and that $s_{t+1,j}/s_{t,j}^3\to 1/h^2\in(0,\infty)$ as $t\to\infty$,
for each $j\in[d]$.
Because the quantity $s_{t,j}$ has a meaning as the size of the residual of $\by_{t,i}$ 
from the convergent point of $(\by_{t,i})_{t\in\bbN}$ along the $j$th coordinate, 
this consideration shows the cubic convergence.

Although this result might suggest
very fast convergence of the BMS algorithm in more general settings,
the validity of the result itself is quite limited
due to the assumptions made.
It is therefore not certain as to whether the fast convergence
demonstrated in this argument is valid in
more practical finite-sample settings.

\subsection{Optimization View in Configuration Space}
\label{sec:View}
The work \cite{cheng1995mean} presented 
an interpretation that the BMS algorithm updating the configuration 
$\by_t=(\by_{t,1}^\top,\ldots,\by_{t,n}^\top)^\top$ can be viewed as 
an optimization procedure in the configuration space $\bbR^{nd}$
for (locally) maximizing the objective function
\begin{align}
	L(\bu)\coloneq\sum_{i,j=1}^n K\biggl(\frac{\bu_i-\bu_j}{h}\biggr)
\end{align}
defined in terms of the configuration 
$\bu\coloneq(\bu^\top_1,\ldots,\bu^\top_n)^\top\in\bbR^{nd}$.
We here describe this optimization view.

Assume for a moment that the kernel $K$ is differentiable 
for the simplicity of the description.
The BMS update rule \eqref{eq:updaBMS} is then rewritten as 
\begin{align}
\label{eq:BMS-k}
	\begin{split}
	\by_{t+1,i}
	&=\by_{t,i}+\frac{h^2}{2\sum_{j=1}^nG_{t,i,j}}
	\biggl(-\frac{2}{h^2}\sum_{j=1}^n(\by_{t,i}-\by_{t,j})G_{t,i,j}\biggr)\\
	&=\by_{t,i}+\frac{h^2}{2\sum_{j=1}^nG_{t,i,j}}
	\frac{\partial L(\bu)}{\partial\bu_i}\biggr|_{\bu=\by_t}.
	\end{split}
\end{align}
In the configuration space, it can further be represented as
\begin{align}
\label{eq:BMS-all}
	\by_{t+1}=\by_t+\frac{h^2}{2}\bS_t^{-1}\nabla L(\by_t)
\end{align}
with the abbreviation
\begin{align}
	\bS_t\coloneq
	\diag\biggl(\biggl(\sum_{j=1}^n G_{t,i,j}\biggr)_{i\in[n]}\biggr)
	\otimes\bI_d\in\bbR^{nd\times nd}.
\end{align}
It shows that the BMS algorithm can be viewed 
as a gradient ascent method for the objective function $L$ with 
the initial estimate $\by_1=(\bx_1^\top,\ldots,\bx_n^\top)^\top$
determined by \eqref{eq:initBMS}.
The matrix $\frac{h^2}{2}\bS_t^{-1}$ in \eqref{eq:BMS-all}
corresponds to the step size, 
and the BMS algorithm is interpretable as employing
$\by_t$-dependent step sizes 
for the coordinates corresponding to each data point.

The existing studies \cite{fukunaga1975estimation, cheng1995mean,Cheng95, 
carreira2006fast, carreira2008generalised, chen2015convergence}
have not leveraged the above optimization view
for convergence analysis of the blurred data point sequences 
$(\by_{t,1})_{t\in\bbN},\ldots,(\by_{t,n})_{t\in\bbN}$
or that of the configuration sequence $(\by_t)_{t\in\bbN}$.
However, as will be shown in the following section, 
the objective function $L$ plays a key role 
in convergence analysis of these sequences.

\section{Convergence Analysis}
\label{sec:Analysis}
\subsection{Properties of Objective Function}
\label{sec:PO}
In this subsection 
we explore properties of the objective function $L$
in preparation for the convergence analysis of the BMS algorithm
in this paper.

First, the function $L$ has the following invariance:
\begin{proposition}[{Invariance of the objective function}]
\label{prop:TI}
It holds that $L(\bu)=L(\bv\otimes\bm{1}_n\pm\bu)$
for any $\bu\in\bbR^{nd}$ and $\bv\in\bbR^d$.
Also, for a permutation $\sigma$ of $[n]$
we let $\bu^\sigma\coloneq(\bu_{\sigma(1)}^\top,\ldots,\bu_{\sigma(n)}^\top)^\top$.
Then, for any permutation $\sigma$ of $[n]$,
it holds that $L(\bu^\sigma)=L(\bu)$. 
Assuming that $K(\bu)=K(-\bu)$ for any $\bu\in\bbR^d$
(which holds under Assumption~\ref{asm:RS}),
it holds that 
$L(\bu)=2\sum_{1\le i\le j\le n}K(\frac{\bu_i-\bu_j}{h})-nK(\bm{0}_d)$.
\end{proposition}

Also, the function $L$ has a trivial global maximizer:
\begin{proposition}[{Global maximizer of the objective function}]
\label{prop:TGM}
Assume Assumption~\ref{asm:RS}.
Then it holds that $\arg\max_{\bu\in\bbR^{nd}}
L(\bu)=\{\bv\otimes\bm{1}_n\mid\bv\in\bbR^d\}$.
\end{proposition}
When the configuration sequence $(\by_t)_{t\in\bbN}$ converges to this global maximizer, 
the BMS-based data clustering results in just one cluster.
It suggests that non-trivial clustering results would be obtained from 
the BMS algorithm only when the sequence $(\by_t)_{t\in\bbN}$ 
converges to a local maximizer of $L$ other than the global maximizer.

\subsection{Convergence Guarantee of Objective Sequence}
\label{sec:COS}
\subsubsection{Minorize-Maximize Algorithm}
\label{sec:COS-Pre}
In this subsection, we provide a sufficient condition under which 
the \emph{objective sequence} $(L(\by_t))_{t\in\bbN}$
is non-decreasing, which in turn implies its convergence.
For convergence analysis of the objective sequence 
$(L(\by_t))_{t\in\bbN}$, we apply the framework of 
the minorize-maximize (MM) algorithm 
\cite{yamasaki2023ms, yamasaki2019kernel, de2009sharp, lange2016mm}.

We first introduce the notion of minorizers:
\begin{definition}[{Minorizer and quadratic minorizer}]\hfill
\label{def:MQM}
\begin{itemize}
\item
For a function $f:\calS\to\bbR$ with $\calS\subseteq\bbR^p$,
a function $\bar{f}(\cdot|\bu')$ is called
a \emph{minorizer} of the function $f$ at $\bu'\in \calS$, 
if it satisfies $\bar{f}(\bu'|\bu')=f(\bu')$ and 
$\bar{f}(\bu|\bu')\le f(\bu)$ for any $\bu\in \calS$.
\item
Also, a minorizer $\bar{f}(\bu|\bu')$ that is quadratic or constant
in $\bu$ is called a \emph{quadratic minorizer}.
\end{itemize}
\end{definition}

The MM algorithm solves a hard optimization problem for 
an original objective function $f$ by iteratively constructing
a minorizer of $f$ at an estimate
and optimizing that minorizer to obtain a new estimate. 
A quadratic minorizer is often employed 
to ease the subsequent optimization.

\subsubsection{Convergence Guarantee}
\label{sec:COS-CG}
Under Assumption~\ref{asm:RS}, 
we can construct a quadratic minorizer of 
the kernel $K$ at $\bu'\in\bbR^d$ as
\begin{align}
\label{eq:MM}
	H(\bu|\bu')
	=\frac{G(\bu')}{2}(\|\bu'\|^2-\|\bu\|^2)
	+K(\bu'),
\end{align}
which satisfies
\begin{align}
	H(\bu'|\bu')=K(\bu');\quad
	H(\bu|\bu')\le K(\bu),\quad\forall\bu\in\bbR^d,
\end{align}
as can be verified by substituting 
$\bu=\bu'$ into $H(\bu|\bu')$ and 
$(u,v)=(\|\bu\|^2/2,\|\bu'\|^2/2)$ into Proposition~\ref{prop:KG} \hyl{a6}.
As a sum of minorizers at a point is a minorizer of a sum of the original 
functions at the same point, we can construct a minorizer of the original 
objective function $L$ at $\by_t$ as 
\begin{align}
	\begin{split}
	&R(\bu|\by_t)
	\coloneq\sum_{i,j=1}^n H\biggl(\frac{\bu_i-\bu_j}{h}\biggl|\frac{\by_{t,i}-\by_{t,j}}{h}\biggr)\\
	&=-\sum_{i,j=1}^n\frac{G_{t,i,j}}{2h^2}\|\bu_i-\bu_j\|^2
	+\text{($\bu$-independent constant)}.
	\end{split}
\end{align}
This way of constructing a minorizer of the function $L$
is similar to that of constructing a minorizer of the KDE $P$
in the MS algorithm \cite{yamasaki2023ms, lange2016mm}.

As $R(\cdot|\by_t)$ is a minorizer of $L$ at $\by_t$,
if we can prove the inequality $R(\by_t|\by_t)\le R(\by_{t+1}|\by_t)$,
we obtain the relation
\begin{align}
\label{eq:MI}
	L(\by_t)=R(\by_t|\by_t)\le R(\by_{t+1}|\by_t)\le L(\by_{t+1}),
\end{align}
showing the ascent property of the objective sequence $(L(\by_t))_{t\in\bbN}$,
which will prove its convergence as $L$ is bounded. 
Although the inequality $R(\by_t|\by_t)\le R(\by_{t+1}|\by_t)$ would
immediately follow for the MM algorithm,
it does not for the BMS algorithm since the BMS algorithm 
\emph{cannot} be regarded as the MM algorithm with the minorizer $R$: 
Maximizing the minorizer $R(\cdot|\by_t)$ yields
trivial optima $\bv\otimes\bm{1}_n$ with any $\bv\in\bbR^d$,
which also maximize $L$. 
We have found, however, that the update rule \eqref{eq:updaBMS} of 
the BMS algorithm also improves the minorizer $R(\cdot|\by_t)$:
\begin{theorem}[{Ascent property and convergence guarantee of objective sequence}]
\label{thm:COS}
Assume Assumption~\ref{asm:RS}.
Then one has that $R(\by_t|\by_t)\le R(\by_{t+1}|\by_t)$ for any $t\in\bbN$,
and the objective sequence $(L(\by_t))_{t\in\bbN}$ is non-decreasing and converges.
Also, for any $t\in\bbN$, $L(\by_t)=L(\by_{t+1})$ implies $\by_t=\by_{t+1}$.
\end{theorem}

\subsection{Convergence Guarantee and Rate Bounds of Configuration Sequence for Smooth Kernels}
\label{sec:CPSS}
\subsubsection{{\L}ojasiewicz Property}
\label{sec:CPSS-Pre}
In Sections~\ref{sec:CPSS} and \ref{sec:CPSN}, we transform 
a convergence guarantee of the objective sequence $(L(\by_t))_{t\in\bbN}$ 
to that of the configuration sequence $(\by_t)_{t\in\bbN}$.
In this subsection, we assume the {\L}ojasiewicz inequality/property 
\cite{lojasiewicz1965ensembles, kurdyka1994wf, bolte2007lojasiewicz, bolte2007clarke}
for the objective function $L$, and study convergence 
of $(\by_t)_{t\in\bbN}$ by relying on that property and 
on existing abstract convergence theorems in optimization theory 
\cite{attouch2009convergence, frankel2015splitting} which exploit that property.
We firstly review the discussion on the {\L}ojasiewicz property, 
and important classes of functions having that property.

The {\L}ojasiewicz inequality gives a bound of 
the flatness of a function around its stationary point.
Its definition, together with those of related notions, is given as follows: 
\begin{definition}[{{\L}ojasiewicz property/inequality/exponent}]\hfill
\label{def:Loj}
\begin{itemize}
\item
A function $f:\calS\to\bbR$ with $\calS\subseteq\bbR^p$
is said to have the \emph{{\L}ojasiewicz property} at $\bu'\in \calS$ 
with an exponent $\theta$, if there exists $\epsilon>0$ 
such that $f$ is differentiable on the intersection
$\calU(\bu',f,\calS,\epsilon)\coloneq
\{\bu\in \calS\mid \|\bu'-\bu\|<\epsilon, f(\bu')-f(\bu)\ge0\}$
of the $\epsilon$-neighborhood of $\bu'$ and
the lower level set $\{\bu\in\calS\mid f(\bu)\le f(\bu')\}$, 
and satisfies the \emph{{\L}ojasiewicz inequality}
\begin{align}
\label{eq:Lojasiewicz-ineq}
	\|\nabla f(\bu)\|\ge c \{f(\bu')-f(\bu)\}^\theta
\end{align}
with some $c>0$ and $\theta\in[0,1)$ and any $\bu\in \calU(\bu',f,\calS,\epsilon)$,
where we adopt the convention $0^0=0$ on the right-hand side of \eqref{eq:Lojasiewicz-ineq}
following \cite[Remark 4]{attouch2009convergence}.
\item
Also, $f$ is said to have the \emph{{\L}ojasiewicz property} 
on $\calT\subseteq \calS$ (when $\calT=\calS$, we omit `on $\calT$'),
if $f$ is differentiable on $\calT$ and there exists $\epsilon>0$ 
such that $f$ satisfies the {\L}ojasiewicz inequality 
\eqref{eq:Lojasiewicz-ineq} with some $c>0$ and $\theta\in[0,1)$ and 
any $(\bu',\bu)$ such that $\bu'\in \calT,\bu\in\calU(\bu',f,\calT,\epsilon)$. 
\item
Moreover, the minimum value of $\theta$, 
with which $f$ has the {\L}ojasiewicz property at $\bu'$, 
is called the \emph{{\L}ojasiewicz exponent} of $f$ at $\bu'$.
\end{itemize}
\end{definition}

Not every function has the {\L}ojasiewicz property:
for example, \cite[p.\,14]{absil2005convergence, palis2012geometric}
and \cite{yamasaki2023ms} respectively present Mexican-hat function and 
$f(\bu)=-e^{-\|\bu\|^{-\gamma}}\bbI(\|\bu\|\neq0), \gamma>0$ 
that are in the $C^\infty$ class as counterexamples.
\cite{lojasiewicz1965ensembles} proved
that analytic functions have the {\L}ojasiewicz property,
and thereafter, \cite{kurdyka1994wf} generalized that result 
to the class of $C^1$ functions with o-minimal structure
(see also \cite{van1996geometric}),
which particularly includes $C^1$ globally subanalytic functions:%
\footnote{%
\label{fn:Bolte}
\cite{bolte2007lojasiewicz, bolte2007clarke} extended the definition of 
the {\L}ojasiewicz inequality for non-smooth functions, 
and proved that continuous globally subanalytic functions
satisfy that generalized {\L}ojasiewicz inequality.
However, our convergence analysis of the BMS algorithm described in 
Section~\ref{sec:CPSS} eventually requires the smoothness assumption 
(assumption \hyl{b1} in Theorem~\ref{thm:BMS-GCG} or
Assumption~\ref{asm:LCG} in Theorem~\ref{thm:BMS-CG}).
Therefore, in Section~\ref{sec:CPSS}, we adopt a simple framework that 
supposes the smoothness even if it can be generalized to the non-smooth case.}
\begin{proposition}[{\cite{lojasiewicz1965ensembles, kurdyka1994wf}}]
\label{prop:SA}
A function $f:\calS\to\bbR$ with $\calS\subseteq\bbR^p$ 
has the {\L}ojasiewicz property, 
if $f$ is analytic or if $f$ is $C^1$ globally subanalytic.
\end{proposition}

Here, the global subanalyticity is defined as follows,
together with several related notions that serve 
as sufficient conditions for the global subanalyticity;
see also \cite{bolte2007lojasiewicz, valette2022subanalytic}.
\begin{definition}[{Global subanalyticity and related notions}]\hfill
\label{def:GSF}
\begin{itemize}
\item%
A set $\calS\subseteq\bbR^p$ is called \emph{semialgebraic},
if there exist a finite number of polynomial functions $f_{ij}:\bbR^p\to\bbR$ such that 
$\calS=\bigcup_{i=1}^q\bigcap_{j=1}^r\{\bu\in\bbR^p\mid f_{ij}(\bu)\,\sigma_{ij}\,0\}$
with relational operators $\sigma_{ij}\in\{<,>,=\}$.
\item%
A set $\calS\subseteq\bbR^p$ is called \emph{semianalytic},
if for each point $\bu'\in\bbR^p$ there exist a neighborhood $\calT$ of $\bu'$ 
and a finite number of analytic functions $f_{ij}:\calT\to\bbR$ such that 
$\calS\cap \calT=\bigcup_{i=1}^q\bigcap_{j=1}^r\{\bu\in \calT\mid f_{ij}(\bu)\,\sigma_{ij}\,0\}$
with relational operators $\sigma_{ij}\in\{<,>,=\}$.
\item%
A set $\calS\subseteq\bbR^p$ is called \emph{subanalytic},
if for each point $\bu'\in\bbR^p$ there exist a neighborhood $\calT$ of $\bu'$ 
and a bounded semianalytic set $\calU\subseteq\bbR^{p+p'}$ with $p'\in\bbN$
such that $\calS\cap \calT=\{\bu\in\bbR^p\mid (\bu,\bv)\in \calU\}$.
\item
A set $\calS\subseteq\bbR^p$ is called \emph{globally semianalytic} 
or \emph{globally subanalytic}, if its image under the map
$\bu\mapsto(u_1/(1+u_1^2)^{1/2},\ldots,u_p/(1+u_p^2)^{1/2})$
is a semianalytic or subanalytic subset of $\bbR^p$, respectively.
\item%
A function $f:\calS\to\bbR$ with $\calS\subseteq\bbR^p$ is called \emph{semialgebraic},
\emph{semianalytic}, \emph{subanalytic}, \emph{globally semianalytic}, 
or \emph{globally subanalytic}, if its graph $\{(\bu,v)\in \calS\times\bbR\mid v=f(\bu)\}$ 
is semialgebraic, semianalytic, subanalytic, globally semianalytic, or globally subanalytic
subset of $\bbR^{p+1}$, respectively.
\end{itemize}
\end{definition}
These related notions are useful in practice, 
because they are easier to verify than the global subanalyticity.
For example, the class of semialgebraic functions includes 
polynomial, rational, and piecewise polynomial functions,
and the class of semianalytic functions includes 
a piecewise analytic function defined over a semianalytic partition 
\cite{bierstone1988semianalytic, dedieu1992penalty};
\cite{yamasaki2023ms} clarified that these function classes include 
many commonly used kernels shown in Table~\ref{tab:Kernel}.

\begin{figure}[!t]
\centering%
\includegraphics[width=8.5cm, bb=0 0 551 216]{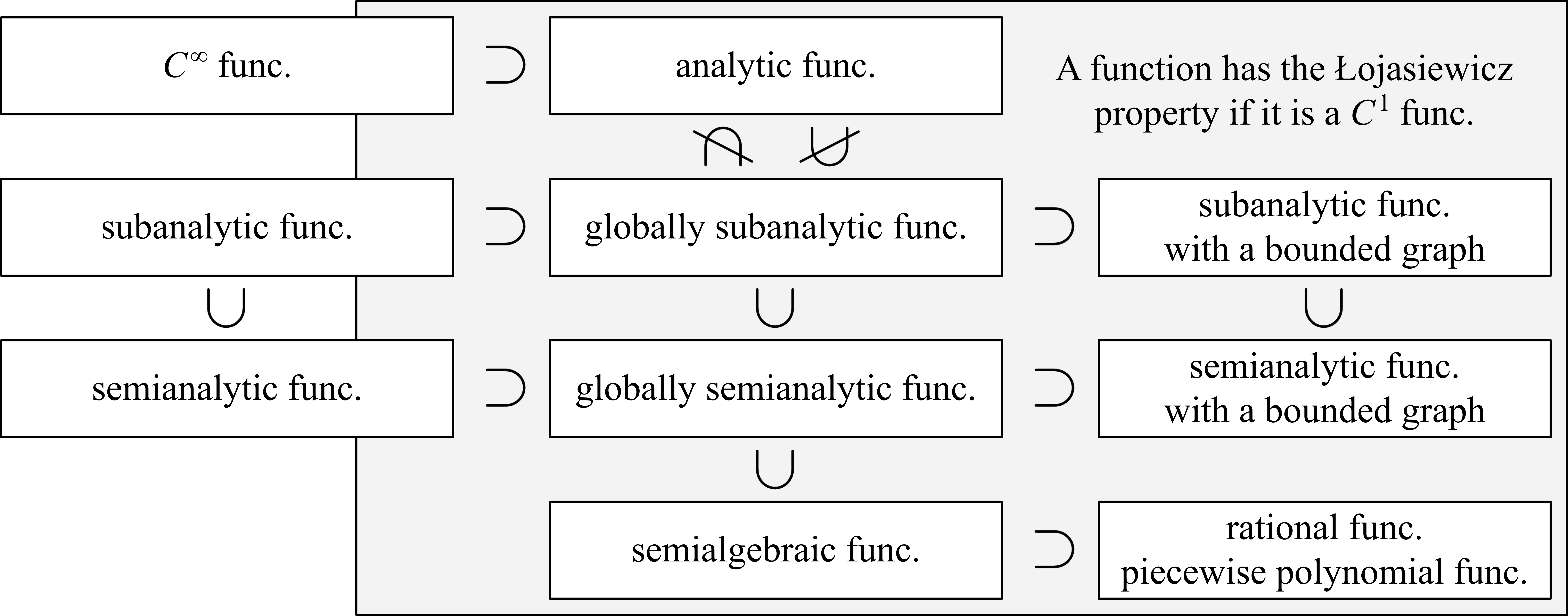}
\caption{%
Inclusion relation among important function classes 
relevant to the discussion on the {\L}ojasiewicz property.}
\label{fig:relation}
\end{figure}

The inclusion relation of these function classes 
(see Figure~\ref{fig:relation} and Section~\ref{sec:Proof-PropSAGSProp})
and the stability of the global subanalyticity under 
the summation \cite[Properties 1.1.8]{valette2022subanalytic} 
summarized below are important in our analysis:
\begin{proposition}
\label{prop:GSprop}
Any semialgebraic or globally semianalytic functions, 
any semianalytic or subanalytic functions with a bounded graph,
and the summation of any globally subanalytic functions
are globally subanalytic.
\end{proposition}

\subsubsection{Convergence Guarantee}
\label{sec:CPSS-CG}
Several recent studies on the optimization theory, such as
\cite{absil2005convergence, attouch2009convergence, 
attouch2013convergence, noll2014convergence, 
frankel2015splitting, bolte2016majorization, yamasaki2023ms}, 
exploit the {\L}ojasiewicz inequality to prove 
the convergence of various optimization algorithms.
Similarly to those studies, by applying abstract convergence theorems,
\cite[Theorem 3.2]{attouch2013convergence}
and \cite[Theorem 3.1]{frankel2015splitting}, 
to the BMS algorithm, we obtain the following theorem:
\begin{theorem}[{Conditional convergence guarantee for smooth kernels}]
\label{thm:BMS-GCG}
Assume Assumption~\ref{asm:RS},
and, for the closure $\calT=\cl(\Conv(\{\by_t\}_{t\ge\tau}))$ of 
the convex hull $\Conv(\{\by_t\}_{t\ge\tau})$ of 
$\{\by_t\}_{t\ge\tau}$ with some $\tau\in\bbN$, that
\begin{itemize}
\item[\hyt{b1}]
the function $L$ is differentiable and has a Lipschitz-continuous gradient on $\calT$
(i.e., there exists a constant $c\in[0,\infty)$ 
such that $\|\nabla L(\bu)-\nabla L(\bu')\|\le c\| \bu-\bu'\|$ 
for any $\bu,\bu'\in \calT$);
\item[\hyt{b2}]
the function $L$ has the {\L}ojasiewicz property on $\calT$.
\end{itemize}
Then the configuration sequence $(\by_t)_{t\in\bbN}$ converges 
to a stationary point $\bar{\by}\in\bbR^{nd}$ of the function $L$.
\end{theorem}

We next argue sufficient conditions on the kernel $K$
which make the function $L$ to satisfy the conditions~\hyl{b1} and \hyl{b2}. 
The conditions~\hyl{b1} and \hyl{b2} require 
the function $L$ to have a Lipschitz-continuous gradient 
(smoothness) and the {\L}ojasiewicz property around 
the trajectory of the sequence $(\by_t)_{t\ge\tau}$.
Simply because the summation retains the smoothness in the above-mentioned sense,
the function $L$ satisfies the condition~\hyl{b1}
if the kernel $K$ satisfies the following assumption:
\begin{assumption}
\label{asm:LCG}%
The kernel $K$ is differentiable with a Lipschitz-continuous gradient.
\end{assumption}
This assumption implies $K\in C^1$.
Thus, when the kernel $K$ is analytic or globally subanalytic,
Propositions~\ref{prop:SA} and \ref{prop:GSprop} show 
that the function $L$ has the {\L}ojasiewicz property.
However, in order to have the function $L$ to satisfy the condition \hyl{b2}, 
one finds that 
requiring only the subanalyticity, not the global subanalyticity, 
for the kernel $K$ is sufficient due to 
a characteristic behavior of the BMS algorithm, as explained below: 
Since Theorem~\ref{thm:Cheng} under Assumption~\ref{asm:RS}
shows $\by_t\in\Conv(\{\bx_i\}_{i\in[n]})^n$ for any $t\in\bbN$,
one can restrict the domain of every kernel $K(\frac{\bu_i-\bu_j}{h})$ 
as a function of $\bu$ to $\Conv(\{\bx_i\}_{i\in[n]})^n$ without any problems.
Also, Assumption~\ref{asm:RS} ensures the boundedness of $K(\frac{\bu_i-\bu_j}{h})$.
Thus, for a subanalytic kernel $K$,
the restriction of $K(\frac{\bu_i-\bu_j}{h})$ to $\Conv(\{\bx_i\}_{i\in[n]})^n$ 
is a subanalytic function with a bounded graph
and accordingly globally subanalytic due to Proposition~\ref{prop:GSprop},
and hence the restriction of the corresponding function $L$ is 
globally subanalytic and has the {\L}ojasiewicz property due to 
Propositions~\ref{prop:SA} and \ref{prop:GSprop} and Assumption~\ref{asm:LCG}.
Accordingly, under Assumptions~\ref{asm:RS} and \ref{asm:LCG}, 
we can employ the following assumption instead of the condition \hyl{b2}:
\begin{assumption}
\label{asm:LP}%
The kernel $K$ is analytic or subanalytic.
\end{assumption}

Thus, the following holds as a corollary of Theorem \ref{thm:BMS-GCG}:
\begin{theorem}[{Convergence guarantee for smooth kernels;
corollary of Theorem \ref{thm:BMS-GCG}}]
\label{thm:BMS-CG}
Assume Assumptions \ref{asm:RS}, \ref{asm:LCG}, and \ref{asm:LP}.
Then the configuration sequence $(\by_t)_{t\in\bbN}$ converges to 
a stationary point $\bar{\by}\in\bbR^{nd}$ of the function $L$.
\end{theorem}
The main significance of Theorem \ref{thm:BMS-CG} is 
that it establishes the convergence of the BMS 
algorithm for several smoothly truncated kernels 
including the biweight and triweight kernels; 
see also Table~\ref{tab:Kernel}.
For these truncated kernels,
the objective function $L$ can have a stationary point 
that is different from a trivial global maximizer,
and one may obtain multiple clusters via the BMS algorithm.

\subsubsection{Convergence Rate Bounds}
On the basis of an abstract convergence theorem
\cite[Theorem 3.5]{frankel2015splitting},
which leverages the {\L}ojasiewicz property,
we can obtain the following convergence rate bounds. 
%
\begin{theorem}[{Convergence rate bounds for smooth kernels}]
\label{thm:BMS-CRE-GEN}
Assume assumptions of Theorem~\ref{thm:BMS-GCG} or \ref{thm:BMS-CG},
and that the function $L$ has the {\L}ojasiewicz exponent $\theta$ at $\bar{\by}=\lim_{t\to\infty}\by_t$.
Then one has that
\begin{itemize}
\item[\hyt{c1}]%
if $\theta\in[0,\frac{1}{2})$ 
then the BMS algorithm converges in a finite number of iterations
(i.e., there exist $\tau\in\bbN$ and $M=M_{\by_\tau}$ different points $\bz_1, \ldots, \bz_M\in\bbR^d$ such that 
$\by_t=\bar{\by}$, 
$\by_{t,i}=\bz_m$ for all $i\in\calV_{\by_\tau,m}$ for all $m\in[M]$,
and $L(\by_t)=L(\bar{\by})$ for any $t\ge\tau$);
\item[\hyt{c2}]%
if $\theta=\frac{1}{2}$
then the BMS algorithm achieves the exponential rate convergence
(i.e., there exist $q\in(0,1)$ and $M=M_{\bar{\by}}$ different points $\bz_1,\ldots,\bz_M\in\bbR^d$ such that 
$\|\by_t-\bar{\by}\|=O(q^t)$, 
$\|\by_{t,i}-\bz_m\|=O(q^t)$ for all $i\in\calV_{\bar{\by},m}$ for all $m\in[M]$,
and $L(\bar{\by})-L(\by_t)=O(q^{2t})$);
\item[\hyt{c3}]%
if $\theta\in(\frac{1}{2},1)$
then the BMS algorithm achieves the polynomial rate convergence
(i.e., there exist $M=M_{\bar{\by}}$ different points $\bz_1,\ldots,\bz_M\in\bbR^d$ such that 
$\|\by_t-\bar{\by}\|=O(t^{-\frac{1-\theta}{2\theta-1}})$,
$\|\by_{t,i}-\bz_m\|=O(t^{-\frac{1-\theta}{2\theta-1}})$ 
for all $i\in\calV_{\bar{\by},m}$ for all $m\in[M]$,
and $L(\bar{\by})-L(\by_t)=O(t^{-\frac{1}{2\theta-1}})$).
\end{itemize}
\end{theorem}

Calculation of the {\L}ojasiewicz exponent $\theta$ at 
a convergent stationary point $\bar{\by}$ is difficult in general, 
but the followings are known about $\theta$:
When $\bar{\by}$ is a local minimizer of $L$, $\theta$ becomes 0.
Also, when $\bar{\by}$ is a critical point of $L$
that is not a local minimizer, if $L$ is twice differentiable and the Hessian of $L$ is non-degenerate at $\bar{\by}$,
one has $\theta=1/2$, which is the typical value
of the {\L}ojasiewicz exponent. 
Furthermore, when the BMS algorithm uses a $C^1$ piecewise polynomial kernel,
we can apply known upper bounds of the {\L}ojasiewicz exponent for 
polynomial functions \cite{d2005explicit, kurdyka2014separation}.
One of such bounds show $\theta\le1-1/\max\{l(3l-4)^{nd-1}, 2l(3l-3)^{nd-2}\}$
when the kernel $K$ is a $C^1$ piecewise polynomial 
kernel with maximum degree $l\ge2$
(e.g., $l=4$ for the Biweight kernel);
see these studies and a similar discussion for 
the MS algorithm in \cite{yamasaki2023ms}.

It should be noted that the convergence rate bounds given
in Theorem~\ref{thm:BMS-CRE-GEN} may be loose, 
since they are based on 
generic theoretical guarantees for gradient-based algorithms, 
and do not reflect the structures specific to the BMS algorithm. 

\subsection{Refined Convergence Rate Bounds for Smooth Kernels}
\label{sec:CPSS-CRE}
Besides the convergence rate bounds 
of the BMS algorithm with a smooth kernel that we have just obtained via 
the {\L}ojasiewicz property in Theorem~\ref{thm:BMS-CRE-GEN}, 
we also reveal a case where refined convergence rate bounds 
can be obtained from a discussion specific to the BMS algorithm.
In this subsection we describe those refined convergence rate bounds.

According to the discussion in the previous subsection, we find that, 
under the assumptions of Theorem~\ref{thm:BMS-GCG} or \ref{thm:BMS-CG}, 
a convergent point $\bar{\by}$ of the configuration sequence $(\by_t)_{t\in\bbN}$
is a stationary point of the function $L$.
%
Stationary points of the function $L$ are fully characterized in terms of the BMS graph,
as in the following theorem. 
\begin{theorem}[{Geometrical characterization of a stationary point}]
\label{thm:Trivial}
Assume Assumption~\ref{asm:RS}.
Then the following two statements are equivalent:
\begin{itemize}
\item[\hyt{d1}]
$\sum_{j=1}^n(\bu_i-\bu_j)G(\frac{\bu_i-\bu_j}{h})=\bm{0}_d$ for all $i\in[n]$.
\item[\hyt{d2}]
The BMS graph $\calG_\bu$ is singular
(i.e., either $\bu_i=\bu_j$ or $G(\frac{\bu_i-\bu_j}{h})=0$ holds for every pair $\{i,j\}\in[n]^2$).
\end{itemize}
\end{theorem}
The statement \hyl{d1} implies that $\bu$ is a fixed point of the BMS update
(i.e., \eqref{eq:initBMS} with $\by_t=\bu$ gives $\by_{t+1}=\bu$).
It further implies that $\nabla L(\bu)=\bm{0}_{nd}$ holds 
if $K(\frac{\bu_i-\bu_j}{h})$ for all $\{i,j\}\in[n]^2$ are differentiable at $\bu$.

This geometrical characterization of the fixed point $\bar{\by}$,
to which the configuration sequence $(\by_t)_{t\in\bbN}$ converges 
under the assumptions of Theorem~\ref{thm:BMS-GCG} or \ref{thm:BMS-CG},
indicates that the BMS graph $\calG_{\bar{\by}}$
of the convergent point $\bar{\by}$ is singular.
Taking this observation into account, 
the convergent point $\bar{\by}$ satisfies either of the two conditions: 
\begin{itemize}
\item[\hyt{e1}]%
the BMS graph $\calG_{\bar{\by}}$ of $\bar{\by}$ is singular and stable;
\item[\hyt{e2}]%
the BMS graph $\calG_{\bar{\by}}$ of $\bar{\by}$ is singular and unstable,
\end{itemize}
as exemplified 
in Figure~\ref{fig:Graph} \hyl{i} and \hyl{iv}.

From our experience, we conjecture
that a convergent point with an unstable BMS graph
would not have a finite-volume basin of attraction in the configuration space, 
and that the instability of the BMS graph of a convergent point 
will not happen with a generic initial configuration.
As we have so far not been successful in proving or disproving this conjecture, 
we would leave the study of this conjecture for future research. 
For the case \hyl{e2}
we can only rely on Theorem~\ref{thm:BMS-CRE-GEN} 
to obtain convergence rate bounds, 
while for the case \hyl{e1} we can give refined convergence rate bounds 
by leveraging geometrical properties of the BMS graph $\calG_{\bar{\by}}$
of the convergent point $\bar{\by}$.

For the case \hyl{e1},
a convergence guarantee of the sequence $(\by_t)_{t\in\bbN}$ 
(Theorem~\ref{thm:BMS-GCG} or \ref{thm:BMS-CG})
shows that there exists $\tau\in\bbN$ such that
$\by_\tau,\by_{\tau+1},\ldots$ belong to an open neighborhood
of the convergent point $\bar{\by}$ so that
the BMS graphs $\calG_{\by_\tau},\calG_{\by_{\tau+1}},\ldots$
are the same as the BMS graph $\calG_{\bar{\by}}$,
which is singular and hence closed.
Therefore, the BMS graph $\calG_{\by_\tau}$ is closed with some $\tau\in\bbN$, 
and Theorem~\ref{thm:variant} is applicable to the case \hyl{e1}
to prove the exponential rate convergence.

%
Looking into the proof of Theorem~\ref{thm:variant}
in the supplementary material,
one observes that, among properties of $g$ available to us,
only the non-negativity and monotonicity of $g$ are used there. 
Under the geometrical information \hyl{e1} of the stationary point,
if one is allowed to assume stronger regularity of $g(u)$
in the vicinity of $u=0$,
as in Assumption~\ref{asm:DDif} below, 
then one can prove an improved convergence rate bound. 
%
%
%
\begin{assumption}
\label{asm:DDif}%
The kernel $K$ has the profile $k$ that is differentiable with 
a Lipschitz-continuous derivative on some neighborhood of the origin,
$(0,v]$ with some $v>0$
(so that, under Assumption~\ref{asm:RS}, 
there exists a constant $c\in(0,\infty)$ such that 
$g$ in \eqref{eq:funcG} satisfies $|g(u)-g(0)|\le c u$ for any $u\in[0,v]$).
\end{assumption}
It should be noted that Assumption~\ref{asm:DDif}
holds under Assumption~\ref{asm:LCG}.

This assumption allows us to apply a local approximation of 
the function $G$ in the BMS algorithm \eqref{eq:updaBMS} 
around the convergent point that satisfies \hyl{e1}.
This approximation yields a tighter evaluation of the time evolution of 
the diameter for components of the BMS graph, 
defined similarly to \eqref{eq:defdia}, 
and then leads to a tighter convergence rate bound, for the case \hyl{e1}.

These arguments 
are summarized as follows:
\begin{theorem}[{Convergence rate bounds for smooth kernels}]
\label{thm:BMS-CRE-EXP}
Assume the assumptions of Theorem~\ref{thm:BMS-GCG} 
and \hyl{e1} with $\bar{\by}=\lim_{t\to\infty}\by_t$. 
Then the BMS algorithm achieves the exponential rate convergence:
There exists $\tau\in\bbN$ such that 
all the BMS graphs $(\calG_{\by_t})_{t\ge\tau}$ are the same as $\calG_{\bar{\by}}$,
and there exist $q\in[\frac{3}{4},1)$ and
$M=M_{\by_\tau}$ different points $\bz_1,\ldots,\bz_M\in\bbR^d$ such that
$\|\by_t-\bar{\by}\|=O(q^t)$, 
$\|\by_{t,i}-\bz_m\|=O(q^t)$ for all $i\in\calV_{\by_\tau,m}$ for all $m\in[M]$,
and $L(\bar{\by})-L(\by_t)=O(q^{2t})$.
Assume further Assumption~\ref{asm:DDif}.
Then the BMS algorithm achieves the cubic convergence:
Take $\tau$ and define $M$ as above. 
Then there exist $\tau'\in\bbN$, $p\in(0,\infty)$, $q\in(0,1)$, and
$M$ different points $\bz_1,\ldots,\bz_M\in\bbR^d$ such that
$\|\by_{t+1}-\bar{\by}\|\le p\|\by_t-\bar{\by}\|^3$ for any $t\ge\tau'$,
$\|\by_t-\bar{\by}\|=O(q^{3^t})$,
$\|\by_{t,i}-\bz_m\|=O(q^{3^t})$ for all $i\in\calV_{\by_\tau,m}$ for all $m\in[M]$,
and $L(\bar{\by})-L(\by_t)=O(q^{2\cdot 3^t})$.
\end{theorem}

Note that the cubic convergence $\|\by_{t+1}-\bar{\by}\|\le p\|\by_t-\bar{\by}\|^3$ 
leads to all the remaining convergence rate bounds described above, but we could not prove
$\|\by_{t+1,i}-\bz_m\|\le p\|\by_{t,i}-\bz_m\|^3$.

Theorem~\ref{thm:BMS-CRE-EXP} clarifies that 
the cubic convergence shown by \cite{carreira2006fast, carreira2008generalised} 
in a certain limited situation (see Section~\ref{sec:MACP}) should rather be regarded 
as the typical convergence rate for the BMS algorithm for a variety of kernels
in the finite-sample setting.
Indeed, one can empirically observe that,
once the BMS graph becomes closed,
one would only need a few more iterations
for the blurred data point sequences, represented in a 
floating-point format on a computer, to converge;
see Sections~S4 and S5 of the supplementary material.

\subsection{Convergence Guarantee and Rate Bounds
  for Non-Smoothly Truncated Kernels}
\label{sec:CPSN}
Unlike the arguments in Section~\ref{sec:CPSS} on smooth kernels,
one cannot apply the abstract convergence theorems 
\cite{attouch2013convergence, frankel2015splitting} 
to ensure the convergence of the configuration sequence 
$(\by_t)_{t\in\bbN}$ by the BMS algorithm with a non-smooth kernel.
It turns out, however, that convergence properties of 
the BMS algorithm with a non-smoothly truncated kernel,
such as the cosine and Epanechnikov kernels, can still be analyzed
by leveraging geometrical structures inherent in the algorithm. 

The non-smooth truncation of a kernel $K$ implies that 
the function $g$ has discontinuity at the truncation point.
In other words, 
\begin{equation}
\label{eq:c2}
	\alpha\coloneq\inf\{g(u)/g(0), u\in[0,\infty)\mid g(u)\neq0\}>0
\end{equation}
holds. 
On the basis of this property and
the convergence guarantee of the objective sequence 
$(L(\by_t))_{t\in\bbN}$ (Theorem~\ref{thm:COS}), 
one can prove that
the BMS graph $\calG_{\by_\tau}$ becomes closed with a sufficiently large $\tau$.
This allows us to employ the proof strategy of
Theorems~\ref{thm:variant} and \ref{thm:BMS-CRE-EXP}
to derive the following convergence theorem:
\begin{theorem}[{Convergence guarantee and rate bounds for non-smoothly truncated kernels}]
\label{thm:NT}
Assume Assumption~\ref{asm:RS}
and that the kernel $K$ is a non-smoothly truncated kernel.
Then the BMS algorithm converges and achieves the exponential rate convergence:
There exist $\bar{\by}=\lim_{t\to\infty}\by_t$ and $\tau\in\bbN$ such that 
all the BMS graphs $(\calG_{\by_t})_{t\ge\tau}$ are the same as $\calG_{\bar{\by}}$,
and there exist $q\in[\frac{3}{4},1)$ and $M=M_{\by_\tau}$ different points
$\bz_1,\ldots,\bz_M\in\bbR^d$ such that
$\|\by_t-\bar{\by}\|=O(q^t)$, 
$\|\by_{t,i}-\bz_m\|=O(q^t)$ for all $i\in\calV_{\by_\tau,m}$ for all $m\in[M]$,
and $L(\bar{\by})-L(\by_t)=O(q^{2t})$.
Assume further Assumption~\ref{asm:DDif}.
Then the BMS algorithm achieves the cubic convergence:
For $\bar{\by}=\lim_{t\to\infty}\by_t$, take $\tau$ and define $M$ as above. 
Then there exist $\tau'\in\bbN$, $p\in(0,\infty)$, $q\in(0,1)$, and 
$M$ different points $\bz_1,\ldots,\bz_M\in\bbR^d$ such that
$\|\by_{t+1}-\bar{\by}\|\le p\|\by_t-\bar{\by}\|^3$ for any $t\ge\tau'$,
$\|\by_t-\bar{\by}\|=O(q^{3^t})$,
$\|\by_{t,i}-\bz_m\|=O(q^{3^t})$ for all $i\in\calV_{\by_\tau,m}$ for all $m\in[M]$,
and $L(\bar{\by})-L(\by_t)=O(q^{2\cdot 3^t})$.
\end{theorem}

Additionally, under this convergence guarantee, 
we can further ensure the finite-time convergence 
of the BMS algorithm for certain non-smoothly 
truncated kernels including the Epanechnikov kernel:
\begin{theorem}[{Convergence rate bound for, e.g., Epanechnikov kernel}]
\label{thm:EP}
Assume assumptions of Theorem~\ref{thm:NT}
and that the profile $k$ of the kernel $K$ satisfies 
$k(u^2/2)=k(0)+(k'(0+)/2)u^2$ for any $u\in[0,\beta']$ with 
some $\beta'\in(0,\beta]$ for the truncation point $\beta$ of $K$ 
(for example, the Epanechnikov kernel).
Then the BMS algorithm converges in a finite number of iterations:
There exist $\bar{\by}=\lim_{t\to\infty}\by_t$ and $\tau\in\bbN$ such that 
all the BMS graphs $(\calG_{\by_t})_{t\ge\tau}$ are the same as $\calG_{\bar{\by}}$,
and there exist $M=M_{\by_\tau}$ different points $\bz_1,\ldots,\bz_M\in\bbR^d$ such that
$\by_t=\bar{\by}$, 
$\by_{t,i}=\bz_m$ for all $i\in\calV_{\by_\tau,m}$ for all $m\in[M]$,
and $L(\by_t)=L(\bar{\by})$ for any $t\ge\tau$.
\end{theorem}
The `finite number of iterations' in this theorem is 
practically a few dozen as far as we have experienced;
refer to Section~S4 of the supplementary material.
This theorem serves as a multi-dimensional 
generalization of \cite[Theorem 14.2]{Cheng95}
that we reviewed in Section~\ref{sec:CW95}.

\section{Proof of Theoretical Results}
\label{sec:Proof}
\subsection{Proof of Proposition~\protect\ref{prop:KG}}
\label{sec:Proof-PropKG}
The statements \hyl{a1}, \hyl{a2}, \hyl{a4}, and \hyl{a6} of 
Proposition \ref{prop:KG} are proved in \cite[Section 24]{rockafellar1997convex},
and the remaining statements are proved from them under Assumption~\ref{asm:RS}.

\subsection{Proof of Propositions~\protect\ref{prop:BMScomp} and \protect\ref{prop:Number}}
\label{sec:Proof-Number}
Proposition~\ref{prop:BMScomp} is clear from 
Proposition~\ref{prop:KG} and the definitions of the BMS graph and its component.
One can prove the bound in Proposition~\ref{prop:Number}
according to the comparison between
the volume of the ball $\{\bv\in\bbR^d\mid\|\bv-\bu_i\|\le \beta h\}$ 
and that of the ball $\{\bv\in\bbR^d\mid\|\bv-\bu_1\|\le\gamma+\beta h/2\}$.
We additionally note that this bound is loose and may be improved, 
for example, by using more refined results on sphere packing
\cite{conway2013sphere,viazovska2017sphere,cohn2017sphere}.

\subsection{Proof of Theorems~\protect\ref{thm:Cheng} and \protect\ref{thm:variant}}
\label{sec:Proof-Cheng}
Although \cite{cheng1995mean} gives a proof of Theorem~\ref{thm:Cheng}, 
for the sake of completeness we provide proofs of Theorems~\ref{thm:Cheng} and \ref{thm:variant}
in Sections~S1.1 and S1.2, respectively, 
of the supplementary material of this paper.

\subsection{Proof of Propositions~\protect\ref{prop:TI} and \protect\ref{prop:TGM}}
\label{sec:Proof-PropTITGM}
Proposition~\ref{prop:TI} is proved via a direct calculation, and 
Proposition~\ref{prop:TGM} is a direct consequence of Proposition~\ref{prop:KG} \hyl{a3}.

\subsection{Proof of Theorem~\protect\ref{thm:COS}}
\label{sec:Proof-COS}
Here, we provide a proof of Theorem~\ref{thm:COS}
that depends on Lemma~\ref{lem:A} introduced below:

\begin{lemma}
\label{lem:A}
Assume Assumption~\ref{asm:RS}.
Then there exists $\bar{a}\in(0,\infty)$ such that
\begin{align}
\label{eq:lem-a}
\begin{split}
	L(\by_{t+1})-L(\by_t)
	&\ge R(\by_{t+1}|\by_t)-R(\by_t|\by_t)\\
	&\ge \bar{a}\|\by_{t+1}-\by_t\|^2
\end{split}
\end{align}
holds for all $t\in\bbN$.
\end{lemma}

\begin{proof}[{Proof of Lemma~\ref{lem:A}}]
Let us introduce the abbreviation
$\bG_t\coloneq(G_{t,i,j})_{i,j\in[n]}\otimes\bI_d\in\bbR^{nd\times nd}$.
The $\by$-dependent part of the minorizer 
$R(\by|\by_t)$ times $-h^2$ can be calculated as
\begin{align}
\label{eq:SmG}
	\sum_{i,j=1}^n\frac{G_{t,i,j}}{2}\|\by_i-\by_j\|^2
	&=\sum_{i,j=1}^nG_{t,i,j}\|\by_i\|^2
	-\sum_{i,j=1}^nG_{t,i,j}\by_i^\top\by_j\nonumber\\
	&=\by^\top(\bS_t-\bG_t)\by.
\end{align}
Also, the update rule \eqref{eq:updaBMS} of the BMS algorithm is expressed as 
\begin{align}
\label{eq:updaBMS2}
	\by_{t+1}=\bS_t^{-1}\bG_t\by_t.
\end{align}
Therefore, one has
\begin{align}
\label{eq:lemAeq2}
	\begin{split}
	&h^2\{R(\by_{t+1}|\by_t)-R(\by_t|\by_t)\}\\
	&=\sum_{i,j=1}^n\frac{G_{t,i,j}}{2}\|\by_{t,i}-\by_{t,j}\|^2
	-\sum_{i,j=1}^n\frac{G_{t,i,j}}{2}\|\by_{t+1,i}-\by_{t+1,j}\|^2\\
	&={\by_t}^\top(\bS_t-\bG_t)\by_t
	-(\bS_t^{-1}\bG_t\by_t)^\top
	(\bS_t-\bG_t)
	(\bS_t^{-1}\bG_t\by_t).
	\end{split}
\end{align}
It generally holds that, for matrices $\bA, \bB\in\bbR^{N\times N}$
with $\bA$ invertible,
\begin{align}
	\begin{split}
	&(\bA-\bB)-\bB\bA^{-1}(\bA-\bB)\bA^{-1}\bB\\
	&=(\bB\bA^{-1}-\bI_N)(\bA+\bB)(\bA^{-1}\bB-\bI_N).
	\end{split}
\end{align}
Applying this equation with $\bA=\bS_t$ and $\bB=\bG_t$ 
to \eqref{eq:lemAeq2}, noting that $\bS_t$ and $\bG_t$ are symmetric, leads to
\begin{align}
\label{eq:lemAeq3}
	\begin{split}
	&h^2\{R(\by_{t+1}|\by_t)-R(\by_t|\by_t)\}\\
	&=\{(\bS_t^{-1}\bG_t-\bI_{nd})\by_t\}^\top
	(\bS_t+\bG_t)
	\{(\bS_t^{-1}\bG_t-\bI_{nd})\by_t\}.
	\end{split}
\end{align}
From the update rule \eqref{eq:updaBMS2} of the BMS algorithm one has 
$\by_{t+1}-\by_t=(\bS_t^{-1}\bG_t-\bI_{nd})\by_t$,
and thus one has 
\begin{align}
\label{eq:lemAeq4}
	&h^2\{R(\by_{t+1}|\by_t)-R(\by_t|\by_t)\}
	=(\by_{t+1}-\by_t)^\top
	(\bS_t+\bG_t)
	(\by_{t+1}-\by_t)
\nonumber\\
	&\ge\text{min-ev}(\bS_t+\bG_t)\cdot
	\|\by_{t+1}-\by_t\|^2,
\end{align}
where $\text{min-ev}(\bA)$ is the minimum eigenvalue of a symmetric matrix $\bA$.
Since it holds that
\begin{align}
	\begin{split}
	&\text{min-ev}(\bS_t+\bG_t)\\
	&=\inf_{\|\by\|=1}\by^\top(\bS_t+\bG_t)\by
	=\inf_{\|\by\|=1}\sum_{i,j=1}^n\frac{G_{t,i,j}}{2}\|\by_i+\by_j\|^2\\
	&=\inf_{\|\by\|=1}\biggl(\sum_{i=1}^n\frac{G_{t,i,i}}{2}\|2\by_i\|^2
	+\sum_{i\neq j}\frac{G_{t,i,j}}{2}\|\by_i+\by_j\|^2\biggr)\\
	&\ge\inf_{\|\by\|=1}\sum_{i=1}^n\frac{G_{t,i,i}}{2}\|2\by_i\|^2
	=2g(0)\inf_{\|\by\|=1}\sum_{i=1}^n\|\by_i\|^2\\
	&=2g(0)\inf_{\|\by\|=1}\|\by\|^2
	=2g(0),
	\end{split}
\end{align}
one has 
\begin{align}
\label{eq:select-a}
	R(\by_{t+1}|\by_t)-R(\by_t|\by_t)
	\ge\bar{a}	\|\by_{t+1}-\by_t\|^2
\end{align}
with $\bar{a}=\frac{2g(0)}{h^2}\in(0,\infty)$
from Proposition \ref{prop:KG} \hyl{a3}.
Additionally, from the properties of the minorizer one has 
$L(\by_t)=R(\by_t|\by_t)$ and $R(\by_{t+1}|\by_t)\le L(\by_{t+1})$
as in \eqref{eq:MI}, which imply
\begin{align}
\label{eq:lemAeq1}
	L(\by_{t+1})-L(\by_t)
	\ge R(\by_{t+1}|\by_t)-R(\by_t|\by_t).
\end{align}
Putting \eqref{eq:select-a} and \eqref{eq:lemAeq1} 
together completes the proof.
\end{proof}

Theorem~\ref{thm:COS} follows straightforwardly from Lemma~\ref{lem:A}:
\begin{proof}[{Proof of Theorem \ref{thm:COS}}]
The function $L$ is bounded under Assumption~\ref{asm:RS},
and Lemma~\ref{lem:A} shows that 
$(L(\by_t))_{t\in\bbN}$ is non-decreasing.
Thus, the convergence of $(L(\by_t))_{t\in\bbN}$ is ensured, 
since a bounded monotone sequence converges in general.
Furthermore, 
$L(\by_{t+1})=L(\by_t)$ implies $\by_{t+1}=\by_t$ via \eqref{eq:lem-a}. 
\end{proof}

\subsection{Proof of Propositions~\protect\ref{prop:SA} and \protect\ref{prop:GSprop}}
\label{sec:Proof-PropSAGSProp}
Proposition~\ref{prop:SA} is given by \cite{lojasiewicz1965ensembles, kurdyka1994wf}.
\cite[before Example 1.1.4]{valette2022subanalytic} describes 
that any globally semianalytic functions are semianalytic, 
and any semianalytic functions with a bounded graph are globally semianalytic.
\cite[after Definition 2.2]{bolte2007lojasiewicz} describes 
a `subanalytic'-version of this statement.
Also, semianalytic functions are subanalytic (as seen from Definition~\ref{def:GSF}),
globally semianalytic functions are globally subanalytic \cite[Definition 1.1.6]{valette2022subanalytic},
and semialgebraic functions are globally semianalytic \cite[Example 1.1.4]{valette2022subanalytic}.
Proposition~\ref{prop:GSprop} summarizes these relations 
and the stability of the global subanalyticity under 
the summation \cite[Properties 1.1.8]{valette2022subanalytic}.

\subsection{Proof of Theorems~\protect\ref{thm:BMS-GCG}, 
\protect\ref{thm:BMS-CG}, and \protect\ref{thm:BMS-CRE-GEN}}
\label{sec:Proof-BMS-Loj}
Theorems~\ref{thm:BMS-GCG} and \ref{thm:BMS-CRE-GEN}
can be proved by applying abstract convergence theorems,
\cite[Theorem 3.2]{attouch2013convergence} and 
\cite[Theorems 3.1 and 3.5]{frankel2015splitting}.
We describe proofs of these theorems in Sections S2.2 and S2.4 of the supplementary material.
Also, Theorem~\ref{thm:BMS-CG} can be proved 
as a corollary of Theorem~\ref{thm:BMS-GCG};
see the description in Section~\ref{sec:CPSS-CG}, 
as well as that in Section S2.3 of the supplementary material.

\subsection{Proof of Theorem~\protect\ref{thm:Trivial}}
\label{sec:Proof-Trivial}
We here provide a proof of Theorem~\ref{thm:Trivial}.

\begin{proof}[{Proof of Theorem~\ref{thm:Trivial}}]
The statement \hyl{d2} immediately implies \hyl{d1}. 
We thus prove in the following that \hyl{d1} implies \hyl{d2}
via proof by contraposition.
More concretely, we deduce that \hyl{d1} does not hold 
under the assumption that \hyl{d2} is not satisfied,
namely, there exists a non-singular component $\calG_{\bu,m}$. 
For any $i\in\calV_{\bu,m}$,
since $G(\frac{\bu_i-\bu_j}{h})=0$ for all $j\in[n]\backslash\calV_{\bu,m}$,
one has $\sum_{j=1}^n(\bu_i-\bu_j)G(\frac{\bu_i-\bu_j}{h})
=\sum_{j\in\calV_{\bu,m}}(\bu_i-\bu_j)G(\frac{\bu_i-\bu_j}{h})$.
Thus, we prove in the following that if the component $\calG_{\bu,m}$
is non-singular then there exists $i\in\calV_{\bu,m}$ such that 
$\sum_{j\in\calV_{\bu,m}}(\bu_i-\bu_j)G(\frac{\bu_i-\bu_j}{h})\neq\bm{0}_d$.

Let $\calP=\Conv(\{\bu_i\}_{i\in\calV_{\bu,m}})$, 
which is a polytope in $\bbR^d$. 
It is compact \cite[Theorem 2.8]{Brondsted1983}, 
and any extreme point of $\calP$ is in 
$\{\bu_i\}_{i\in\calV_{\bu,m}}$ \cite[Theorem 7.2]{Brondsted1983}.
Let $\bu_i$, $i\in\calV_{\bu,m}$ be an extreme point of $\calP$. 
It is an exposed point \cite[Theorem 7.5]{Brondsted1983} of $\calP$, 
which means that there exists a proper supporting hyperplane 
$\calR$ of $\calP$ with which $\{\bu_i\}=\calR\cap\calP$.
Let $\calQ$ be the (closed) supporting halfspace of $\calP$ 
bounded by $\calR$, that is, $\calP\subset\calQ$. 
The supporting halfspace $\calQ$ of $\calP$ at $\bu_i$
is represented with a normal $\bq$ of $\calR$ as 
$\calQ=\{\bv\in\bbR^d\mid\bq^\top\bv\ge\bq^\top\bu_i\}$. 
The above fact furthermore means that
$\calP\backslash\{\bu_i\}\subset\calQ\backslash\calR$ holds,
and in particular, for any $j\in\calV_{\bu,m}$ satisfying 
$\bu_j\neq\bu_i$ one has $\bq^\top\bu_j>\bq^\top\bu_i$.
Also, as the component $\calG_{\bu,m}$ is non-singular, 
there exists $k\in\calV_{\bu,m}\backslash\{i\}$ that
satisfies $\bu_k\neq\bu_i$ and has an edge with $i\in\calV_{\bu,m}$. 
Then,
\begin{align}
	\bq^\top\biggl\{\sum_{j\in\calV_{\bu,m}}
	(\bu_i-\bu_j)G\biggl(\frac{\bu_i-\bu_j}{h}\biggr)\biggr\}
	&\le\bq^\top(\bu_i-\bu_k)G\biggl(\frac{\bu_i-\bu_k}{h}\biggr)
\nonumber\\
	&<0,
\end{align}
which implies the failure of \hyl{d1},
completing the proof.
\end{proof}

\subsection{Proof of Theorem~\protect\ref{thm:BMS-CRE-EXP}}
\label{sec:Proof-Quick}
We here give a proof of Theorem~\ref{thm:BMS-CRE-EXP} which 
provides convergence rate bounds using the local approximation.

\begin{proof}[{Proof of Theorem~\ref{thm:BMS-CRE-EXP}}]
For the case \hyl{e1},
the convergence guarantee of Theorem~\ref{thm:BMS-GCG} 
ensures that there exists $\tau\in\bbN$ such that
the BMS graphs $\calG_{\by_\tau},\calG_{\by_{\tau+1}},\ldots$ 
are the same as the BMS graph $\calG_{\bar{\by}}$ and are hence closed.
The exponential rate convergence of 
the sequences $(\by_t)_{t\in\bbN}$ and $(\by_{t,i})_{t\in\bbN}$'s
in the former half of Theorem~\ref{thm:BMS-CRE-EXP} follows straightforwardly 
from Theorem~\ref{thm:variant} and closedness of $\calG_{\by_\tau}$,
we prove in the following the latter half, 
which states the cubic convergence of those sequences.

The invariance of the BMS graphs $(\calG_{\by_t})_{t\ge\tau}$
allows us to adopt the convention $\calV_m\coloneq
\calV_{\by_\tau,m}=\calV_{\by_{\tau+1},m}=\cdots=\calV_{\bar{\by},m}$
for every $m\in[M]$.
Define the diameter for components of the BMS graph $\calG_{\bar{\by}}$ as
\begin{align}
\label{eq:rhot}
	\rho_t\coloneq
	\max\{\|\by_{t,i_1}-\by_{t,i_2}\|\mid i_1,i_2\in\calV_m,m\in[M]\}.
\end{align}
Since the convergence guarantee is available,
one has $\lim_{t\to\infty}\rho_t=0$.
In the proof we bound $\rho_{t+1}$ from above
in terms of $\rho_t$.
For this purpose we consider the quantity
$\|\by_{t+1,i_1}-\by_{t+1,i_2}\|$ for 
$t\ge\tau, i_1,i_2\in\calV_m, m\in[M]$, which is bounded from above as 
\begin{align}
	&\|\by_{t+1,i_1}-\by_{t+1,i_2}\|
	=\biggl\|\frac{\sum_{j=1}^n G_{t,i_1,j}\by_{t,j}}{\sum_{j=1}^n G_{t,i_1,j}}
	-\frac{\sum_{j=1}^n G_{t,i_2,j}\by_{t,j}}{\sum_{j=1}^n G_{t,i_2,j}}\biggr\|
\nonumber\\
&=\biggl\|\frac{\sum_{j_1,j_2=1}^nG_{t,i_1,j_1}G_{t,i_2,j_2}\by_{t,j_1}
  -\sum_{j_1,j_2=1}^nG_{t,i_1,j_1}G_{t,i_2,j_2}\by_{t,j_2}}
	{(\sum_{j=1}^n G_{t,i_1,j})(\sum_{j=1}^n G_{t,i_2,j})}\biggr\|
\nonumber\\
&=\biggl\|\frac{\sum_{j_1,j_2=1}^n(G_{t,i_1,j_1}G_{t,i_2,j_2}
  -G_{t,i_1,j_2}G_{t,i_2,j_1})\by_{t,j_1}}
	{(\sum_{j=1}^n G_{t,i_1,j})(\sum_{j=1}^n G_{t,i_2,j})}\biggr\|
\nonumber\\
	&=\biggl\|\frac{\sum_{1\le j_1<j_2\le n} (G_{t,i_1,j_1}G_{t,i_2,j_2}-G_{t,i_1,j_2}G_{t,i_2,j_1})(\by_{t,j_1}-\by_{t,j_2})}
	{(\sum_{j=1}^n G_{t,i_1,j})(\sum_{j=1}^n G_{t,i_2,j})}\biggr\|
\nonumber\\
\label{eq:difbou}
	&\le\frac{\sum_{1\le j_1<j_2\le n} Q_{t,i_1,i_2,j_1,j_2}}
	{\{g(0)\}^2}
	=\frac{\sum_{(j_1,j_2)\in\calV_m^2,j_1<j_2} Q_{t,i_1,i_2,j_1,j_2}}
	{\{g(0)\}^2}
\end{align}
with $Q_{t,i_1,i_2,j_1,j_2}\coloneq|G_{t,i_1,j_1}G_{t,i_2,j_2}-G_{t,i_1,j_2}G_{t,i_2,j_1}|\|\by_{t,j_1}-\by_{t,j_2}\|$.
The proof proceeds by further bounding $Q_{t,i_1,i_2,j_1,j_2}$'s
from above in terms of $\rho_t$, as detailed in the following.

First, the convergence guarantee of Theorem~\ref{thm:BMS-GCG} 
ensures that, for any $\epsilon>0$, 
there exists $\tau'\in\bbN$ with $\tau'\ge\tau$ such that 
for any $t\ge\tau'$ and $i_1, i_2\in\calV_m$, $m\in[M]$ one has 
$\|\by_{t,i_1}-\by_{t,i_2}\|<\epsilon$.
This allows us to apply the following local approximation
under Assumption~\ref{asm:DDif}: the function $g$ satisfies 
$g(0)-c_1 u\le g(u)\le g(0)$ for any $u\in[0,v]$ with some constant $c_1>0$.
Let $\epsilon$ be such that $(\epsilon/h)^2/2\le v$ holds,
and take $\tau'$ to satisfy the above condition. 
For every $i_1, i_2, j_1,j_2\in\calV_m$ with some $m\in[M]$, 
since $\|\by_{t,j_1}-\by_{t,j_2}\|<\rho_t$ and $(\rho_t/h)^2/2\le(\epsilon/h)^2/2\le v$ for any $t\ge\tau'$, 
one has 
\begin{align}
\label{eq:Qineq1}
\begin{split}
	Q_{t,i_1,i_2,j_1,j_2}
	&\le\{g(0)^2-(g(0)-c_1(\rho_t/h)^2/2)^2\}\rho_t\\
	&\le (c_1g(0)/h^2)\rho_t^3.
\end{split}
\end{align}
Therefore, one has from \eqref{eq:difbou} that
\begin{align}
\label{eq:ExpTow}
	\|\by_{t+1,i_1}-\by_{t+1,i_2}\|
	\le c_2\rho_t^3
\end{align}
with $c_2\coloneq c_1|\calV_m|(|\calV_m|-1)/(2g(0)h^2)$
for any $t\ge\tau'$ and every $i_1,i_2\in\calV_m$,
which implies
\begin{align}
\label{eq:ExpTow2}
	\rho_{t+1}\le c_2\rho_t^3,\quad\forall t\ge\tau'.
\end{align}
Also, 
the definition of $\rho_t$, as well as the closedness of $(\calG_{\by_t})_{t\ge\tau}$, 
implies that $\|\by_{t,i}-\bz_m\|\le\rho_t$ for all $i\in\calV_m$ for all $m\in[M]$ for any $t\ge\tau$,
which yields 
\begin{align}
\label{eq:rho-ineq}
	\rho_t\le\|\by_t-\bar{\by}\|\le\sqrt{n}\rho_t,\quad\forall t\ge\tau.
\end{align}
The inequalities \eqref{eq:ExpTow2} and \eqref{eq:rho-ineq} reveal
\begin{align}
\label{eq:ExpTow3}
	\|\by_{t+1}-\bar{\by}\|\le p\|\by_t-\bar{\by}\|^3
	\text{~with~}p=\sqrt{n}c_2,\quad\forall t\ge\tau'.
\end{align}
From \eqref{eq:ExpTow3}, one can show $\|\by_t-\bar{\by}\|=O(q^{3^t})$ 
via the argument in~\cite[Section 9.3.2]{Ortega2000}.
This result in turn implies $\|\by_{t,i}-\bz_m\|=O(q^{3^t})$,
since $\|\by_{t,i}-\bz_m\|\le\|\by_t-\bar{\by}\|$
for all $i\in\calV_m$ for all $m\in[M]$.

Additionally, for both of the former half and the latter half
of the theorem,
the local approximation of the function $L$ 
based on the Lipschitz-continuity of the profile $k$ of the kernel $K$ 
(Proposition~\ref{prop:KG} \hyl{a7} under Assumption~\ref{asm:RS})
reveals $L(\bar{\by})-L(\by_t)=O(\|\by_t-\bar{\by}\|^2)$.
\end{proof}

\subsection{Proof of Theorems~\protect\ref{thm:NT} and \protect\ref{thm:EP}}
\label{sec:Proof-NT}
In this subsection, we present proofs of Theorems~\ref{thm:NT} and \ref{thm:EP} that are 
about convergence properties of the BMS algorithm for a non-smoothly truncated kernel.
\begin{proof}[{Proof of Theorem~\ref{thm:NT}}]
We first argue that for any $\epsilon>0$
there exists $\tau\in\bbN$ such that
the following holds: 
\begin{align}
\label{eq:eq1}
	\|\by_{\tau+1,i}-\by_{\tau,i}\|<\sqrt{\frac{\epsilon}{\bar{a}}},
	\quad\forall i\in[n].
\end{align}
Indeed, from the fact that $(L(\by_t))_{t\in\bbN}$ is monotonic
and bounded from above, for any $\epsilon>0$ there exists $\tau\in\bbN$
such that the inequality $L(\by_{\tau+1})-L(\by_\tau)<\epsilon$ holds.
Lemma~\ref{lem:A} then states that one has
\begin{align}
	L(\by_{\tau+1})-L(\by_\tau)\ge\bar{a}\|\by_{\tau+1}-\by_\tau\|^2,
\end{align}
with $\bar{a}=2g(0)/h^2$.
Combining these two, as well as the inequality
$\|\bm{u}\|\ge\|\bm{u}_i\|$ for any configuration
$\bm{u}=(\bu_1^\top,\ldots,\bu_n^\top)^\top\in\bbR^{nd}$
and any $i\in[n]$, one obtains~\eqref{eq:eq1}.

Equation~\eqref{eq:eq1} means that at iteration $(\tau+1)$
there is no blurred data point which moves more
than $\sqrt{\epsilon/\bar{a}}$.
In the following, we will use this fact with a small enough $\epsilon$
to prove that the BMS graph $\calG_{\by_\tau}$ is closed,
which allows us to apply Theorem~\ref{thm:variant} or \ref{thm:BMS-CRE-EXP},
depending on the assumptions made in Theorem~\ref{thm:NT},
to conclude the desired convergence rate bounds.

Letting $\calG_{\by_t,*}$ be a component in $\calG_{\by_t}$ 
and $\calV_{\by_t,*}$ be its vertex set, 
one can rewrite, for any $i\in\calV_{\by_t,*}$, the BMS update rule \eqref{eq:updaBMS} as
\begin{align}
\label{eq:eq3}
	\by_{t+1,i}-\by_{t,i}
	=\frac{\sum_{j\in\calV_{\by_t,*}}G_{t,i,j}
	(\by_{t,j}-\by_{t,i})}
	{\sum_{j\in\calV_{\by_t,*}}G_{t,i,j}}.
\end{align}
One can see that the quantity
$G_{t,i,j}(\by_{t,j}-\by_{t,i})$
represents the influence of the blurred data point $\by_{t,j}$
on the update of $\by_{t,i}$. 
A unique characteristic of a non-smoothly truncated kernel
is that, for $\bu$ satisfying $G(\bu/h)>0$, 
$G(\bu/h)\bu$ is small only when $\bu$ itself is small.
Consequently, if $\|\by_{t+1,i}-\by_{t,i}\|$ is small, then either:
\begin{itemize}
\item $\|\by_{t,j}-\by_{t,i}\|$ for $j$ joined to $i$ in $\calG_{\by_t,*}$ are all small, or
\item cancellation among $G_{t,i,j}(\by_{t,j}-\by_{t,i})$ 
for $j$ joined to $i$ in $\calG_{\by_t,*}$ takes place,
even when some $\|\by_{t,j}-\by_{t,i}\|$ are not small.
\end{itemize}
The following lemma formalizes the idea described in the second item above.

\begin{lemma}
\label{lem:LEM2}
Assume $\|\by_{t+1,i}-\by_{t,i}\|<\epsilon'$ for some $\epsilon'>0$.
Let $\bu\in\bbR^d$ be an arbitrary unit vector, and assume that there exists
$j\in\calV_{\by_t,*}$ such that $\bu^\top(\by_{t,i}-\by_{t,j})\ge\delta$
holds with some $\delta>0$.
Then there exists $\ell\in\calV_{\by_t,*}\backslash\{j\}$
such that $\bu^\top(\by_{t,\ell}-\by_{t,i})\ge\psi(\delta)$
with $\psi(\delta)=\alpha\delta/n-\epsilon'$,
where $\alpha$ is as defined in \eqref{eq:c2}. 
\end{lemma}
\begin{proof}[{Proof of Lemma~\ref{lem:LEM2}}]
From the assumption one has 
\begin{align}
\label{eq:eq4}
	-\epsilon'\le\bu^\top(\by_{t+1,i}-\by_{t,i}).
\end{align}
The term $\bu^\top(\by_{t+1,i}-\by_{t,i})$ can be rewritten
via the BMS update rule \eqref{eq:eq3} as
\begin{align}
\label{eq:eq5}
\begin{split}
	&\bu^\top(\by_{t+1,i}-\by_{t,i})
	=\frac{\sum_{j'\in\calV_{\by_t,*}}G_{t,i,j'}\bu^\top(\by_{t,j'}-\by_{t,i})}
	{\sum_{j'\in\calV_{\by_t,*}}G_{t,i,j'}}
\\
	&\le\frac{\sum_{j'\in\calV_{\by_t,*}\backslash\{j\}}G_{t,i,j'}\bu^\top(\by_{t,j'}-\by_{t,i})
	-\alpha g(0)\delta}
	{\sum_{j'\in\calV_{\by_t,*}}G_{t,i,j'}}.
\end{split}
\end{align}
Combining \eqref{eq:eq4} and \eqref{eq:eq5}, one obtains
\begin{align}
\label{eq:eq6}
	\sum_{j'\in\calV_{\by_t,*}\backslash\{j\}}G_{t,i,j'}\bu^\top(\by_{t,j'}-\by_{t,i})
	&\ge\alpha g(0)\delta-\epsilon'\sum_{j'\in\calV_{\by_t,*}}G_{t,i,j'}
	\nonumber\\
	&\ge(\alpha\delta-\epsilon'n)g(0),
\end{align}
where we used the facts $G_{t,i,j'}\le g(0)$ and $|\calV_{\by_t,*}|\le n$.
Let $\ell\in\calV_{\by_t,*}$ be such that
$G_{t,i,\ell}\bu^\top(\by_{t,\ell}-\by_{t,i})$ is maximum
among the summands $\{G_{t,i,j'}\bu^\top(\by_{t,j'}-\by_{t,i})\mid
j'\in\calV_{\by_t,*}\backslash\{j\}\}$ 
on the left-hand side of \eqref{eq:eq6}.
One then has
\begin{align}
\begin{split}
	\bu^\top(\by_{t,\ell}-\by_{t,i})
	&\ge\frac{G_{t,i,\ell}}{g(0)}\bu^\top(\by_{t,\ell}-\by_{t,i})
\\
	&\ge\frac{1}{ng(0)}\sum_{j'\in\calV_{\by_t,*}\backslash\{j\}}G_{t,i,j'}\bu^\top(\by_{t,j'}-\by_{t,i})
\\
	&\ge\frac{\alpha\delta}{n}-\epsilon'=\psi(\delta),
\end{split}
\end{align}
where the last inequality follows from \eqref{eq:eq6}. 
\end{proof}

One can apply Lemma~\ref{lem:LEM2} repeatedly to prove the next lemma.
\begin{lemma}
\label{lem:LEM3}
Assume $\|\by_{t+1,i}-\by_{t,i}\|\le\epsilon'$, 
$\forall i\in\calV_{\by_t,*}$, for some $\epsilon'>0$.
Let $\bu\in\bbR^d$ be an arbitrary unit vector,
and assume that there exist $j_0,j_1\in\calV_{\by_t,*}$
such that $\bu^\top(\by_{t,j_1}-\by_{t,j_0})\ge\delta$ with some $\delta>0$.
Then there exists a sequence $j_2,j_3,\ldots$
with $j_k\in\calV_{\by_t,*}$ such that 
for any $k\in\bbN$ one has
$\bu^\top(\by_{t,j_{k+1}}-\by_{t,j_k})\ge\psi^k(\delta)$. 
\end{lemma}

Noting that
\begin{align}
	\psi^k(\delta)
	=\left(\frac{\alpha}{n}\right)^{k-1}\delta
	-\epsilon'\frac{1-(\alpha/n)^{k-1}}{1-\alpha/n}
	>\left(\frac{\alpha}{n}\right)^{k-1}\delta
	-\frac{\epsilon'}{1-\alpha},
\end{align}
$\delta$ in Lemma~\ref{lem:LEM3} cannot be larger than $\delta_0$ defined as 
\begin{align}
	\delta_0=\left(\frac{n}{\alpha}\right)^{n-1}\frac{\epsilon'}{1-\alpha}.
\end{align}
Indeed, if otherwise, one would have
$\psi^k(\delta)>0$ for all $k\in[n-1]$.
It would in turn imply that the sequence 
$\bu^\top\by_{t,j_0},\bu^\top\by_{t,j_1},\ldots,\bu^\top\by_{t,j_n}$
is strictly increasing, leading to contradiction
since there are only $n$ blurring data points in total.
Noting that $\|\by\|=\sup_{\bu:\|\bu\|=1}\bu^\top\by$, 
the above argument have shown that if $\|\by_{t+1,i}-\by_{t,i}\|\le\epsilon'$, 
$\forall i\in\calV_{\by_t,*}$, then for any $i,j\in\calV_{\by_t,*}$
the inequality $\|\by_{t,i}-\by_{t,j}\|<\delta_0$ should hold. 

Collecting the arguments so far, 
if one takes $\epsilon$ small enough so that the inequality 
\begin{align}
	\delta_0=\left(\frac{n}{\alpha}\right)^{n-1}\frac{\sqrt{\epsilon}{\bar{a}}}{1-\alpha}<\beta h
\end{align}
holds, then by taking $\tau\in\bbN$ that satisfies \eqref{eq:eq1},
there should be no indices $i,j$ belonging to
the same component of $\calG_{\by_\tau}$
such that $\|\by_{\tau,i}-\by_{\tau,j}\|\ge\beta h$.
It implies that all the components of $\calG_{\by_\tau}$
are complete, and therefore $\calG_{\by_\tau}$ is closed.
This completes the proof. 
\end{proof}

\begin{proof}[{Proof of Theorem \ref{thm:EP}}]
Under the additional assumption of Theorem \ref{thm:EP},
it holds that $g(u^2/2)=g(0)$ for any $u\in[0,\beta']$.
Thus, once $\rho_\tau\le\beta' h$ holds,
$\by_{\tau+1}=\by_{\tau+2}=\cdots$ holds. 
\end{proof}

\section{Conclusion}
\label{sec:Conclusion}
In this study, we analyzed convergence properties of the BMS algorithm by utilizing 
its interpretation as an optimization procedure for the objective function $L$,
which is known but has been underutilized in existing convergence studies.
Consequently, we presented a convergence guarantee 
when the blurred data point sequences can 
converge to multiple points and yield multiple clusters,
when the algorithm is based on a $C^1$ subanalytic kernel 
(which includes the biweight and triweight kernels), or a non-smoothly 
truncated kernel (which includes the cosine and Epanechnikov kernels).
We also showed that the algorithm for these kernels 
typically achieves the cubic convergence,
and the finite-time convergence for the Epanechnikov kernel.
On the other hand, it is unanswered whether the case \hyl{e2} actually 
occurs, for which the algorithm has only a weaker convergence rate bound.
These results are summarized in Table~\ref{tab:Kernel}.

\begin{table}[!t]
{\centering%
\renewcommand{\tabcolsep}{1pt}
\renewcommand{\arraystretch}{0.4}
\caption{%
Fulfillment of assumptions of kernels 
(satisfying Assumptions~\ref{asm:LP} and \ref{asm:DDif}),
applicable theorems, and convergence of the BMS algorithm.}
\label{tab:Kernel}
\scalebox{0.775}{%
\begin{tabular}{c|c|cc|c|rc}
\toprule
Kernel & Profile $k(u)\propto$ & 
	Asm.\,\ref{asm:RS} & Asm.\,\ref{asm:LCG} & 
	Applicable thms. & 
	\multicolumn{2}{c}{Convergence}\\
\midrule
Epanechnikov & $(1-u)_+$ & 
	\checkmark & $\times$ &
	Thms.\,\ref{thm:Cheng}--\ref{thm:BMS-GCG}, \ref{thm:BMS-CRE-GEN}, \ref{thm:Trivial}, \ref{thm:NT}, \ref{thm:EP} & 
	$\circledcirc$ & Thm.\,\ref{thm:EP}\\
Cosine & $\cos(\frac{\pi u^{1/2}}{2})\bbI\!(u\!\le\!1)$ & 
	\checkmark & $\times$ &
	Thms.\,\ref{thm:Cheng}--\ref{thm:BMS-GCG}, \ref{thm:BMS-CRE-GEN}, \ref{thm:Trivial}, \ref{thm:NT} & 
	$\oplus$ & Thm.\,\ref{thm:NT}\\
Quadweight & $\{(1-u)_+\}^4$ & 
	\checkmark & \checkmark &
	Thms.\,\ref{thm:Cheng}--\ref{thm:BMS-CRE-EXP} & 
	$\ominus$ & Thms.\,\ref{thm:BMS-CG}, \ref{thm:BMS-CRE-GEN}, \ref{thm:BMS-CRE-EXP}\\
Triweight & $\{(1-u)_+\}^3$ & 
	\checkmark & \checkmark &
	Thms.\,\ref{thm:Cheng}--\ref{thm:BMS-CRE-EXP} & 
	$\ominus$ & Thms.\,\ref{thm:BMS-CG}, \ref{thm:BMS-CRE-GEN}, \ref{thm:BMS-CRE-EXP}\\
Biweight & $\{(1-u)_+\}^2$ & 
	\checkmark & \checkmark &
	Thms.\,\ref{thm:Cheng}--\ref{thm:BMS-CRE-EXP} & 
	$\ominus$ & Thms.\,\ref{thm:BMS-CG}, \ref{thm:BMS-CRE-GEN}, \ref{thm:BMS-CRE-EXP}\\
-- & $\{(1-u)_+\}^{3/2}$ & 
	\checkmark & $\times$ &
	Thms.\,\ref{thm:Cheng}--\ref{thm:BMS-GCG}, \ref{thm:BMS-CRE-GEN}--\ref{thm:BMS-CRE-EXP} & 
	$\bigtriangleup$ & Thms.\,\ref{thm:BMS-GCG}, \ref{thm:BMS-CRE-GEN}, \ref{thm:BMS-CRE-EXP}\\
Gaussian & $e^{-u}$ & 
	\checkmark & \checkmark &
	Thms.\,\ref{thm:Cheng}--\ref{thm:BMS-CRE-EXP} & 
	$\square$ & Thm.\,\ref{thm:BMS-CRE-EXP}\\
Logistic & $\frac{1}{e^{u^{1/2}}+2+e^{-u^{1/2}}}$ &
	\checkmark & \checkmark &
	Thms.\,\ref{thm:Cheng}--\ref{thm:BMS-CRE-EXP} & 
	$\square$ & Thm.\,\ref{thm:BMS-CRE-EXP}\\
Cauchy & $\frac{1}{1+u}$ & 
	\checkmark & \checkmark &
	Thms.\,\ref{thm:Cheng}--\ref{thm:BMS-CRE-EXP} & 
	$\square$ & Thm.\,\ref{thm:BMS-CRE-EXP}\\
Tricube & $\{(1-u^{3/2})_+\}^3$ & 
	$\times$ & \checkmark &
	-- & 
	$\star$ & --\\
\bottomrule
\end{tabular}}
\par\smallskip}
{\small%
Convergent points $\{\bar{\by}_i\}_{i\in[n]}$ 
/ Convergence rate bound of $(\by_t)_{t\in\bbN}$:\par
$\circledcirc$:
possibly multiple points / finite number of iterations;\par
$\oplus$:
possibly multiple points / cubic;\par
$\ominus$:
possibly multiple points / cubic in most cases,\par
\hphantom{$\ominus$:}
polynomial at worst;\par
$\bigtriangleup$:
possibly multiple points / cubic in most cases,\par
\hphantom{$\bigtriangleup$:}
polynomial at worst, under conditions \hyl{b1} and \hyl{b2};\par
$\square$:
single point / cubic.\par
$\star$:
Not only convergence of $(\by_t)_{t\in\bbN}$ but even that of $(L(\by_t))_{t\in\bbN}$ is not ensured.}
\end{table}

Our proof strategies can be employed to show 
that these results hold as well for variants of the BMS algorithm derived for 
the generalized objective function $\sum_{i,j=1}^n w_{i,j}K(\frac{\by_i-\by_j}{h_{i,j}})$ 
with weights $w_{i,j}=w_{j,i}\in(0,\infty)$ and bandwidths $h_{i,j}=h_{j,i}\in(0,\infty)$.


\section*{Acknowledgment}
This work was supported by Grant-in-Aid for JSPS Fellows, Number 20J23367.
\bibliographystyle{IEEEtran}
\bibliography{bibtex}
\clearpage
\onecolumn
\renewcommand{\theequation}{S\arabic{equation}}
\renewcommand{\thelemma}{S\arabic{lemma}}
\renewcommand{\theclaim}{S\arabic{claim}}
\renewcommand{\thefigure}{S\arabic{figure}}
\setcounter{equation}{0}
\setcounter{lemma}{0}
\setcounter{claim}{0}
\setcounter{figure}{0}
\section*{S1\quad Proof of Theorems~\protect\ref{thm:Cheng} and \protect\ref{thm:variant}}
\label{sec:Ape-Cheng}
\subsection*{S1.1\quad Proof of Theorem~\protect\ref{thm:Cheng}}
We here provide a proof of Theorem~\ref{thm:Cheng} 
for the sake of completeness. 
\begin{proof}[{Proof of Theorem~\ref{thm:Cheng}}]
\textbf{Proof of (\ref{eq:Claim1}):}
Proposition \ref{prop:KG} \hyl{a5} ensures 
$G_{t,i,j}\ge0$ for any $t\in\bbN$ and all $i, j\in[n]$,
and hence one has 
\begin{align}
	a_{t+1,\bu}
	&=\min\biggl\{\bu^\top\frac{\sum_{j=1}^n G_{t,i,j}\by_{t,j}}{\sum_{j=1}^n G_{t,i,j}}\biggr\}_{i\in[n]}
	\quad(\because\eqref{eq:updaBMS}, \eqref{eq:ab})
\nonumber\\
	&\ge\min\biggl\{\frac{\sum_{j=1}^n G_{t,i,j}\min\{\bu^\top\by_{t,k}\}_{k\in[n]}}{\sum_{j=1}^n G_{t,i,j}}\biggr\}_{i\in[n]}
	\quad(\because G_{t,i,j}\ge0)
\nonumber\\
	&=\min\{a_{t,\bu}\}_{i\in[n]}
	=a_{t,\bu}.
\end{align}
One can derive $b_{t+1,\bu}\le b_{t,\bu}$ in a similar way.
These results, together with $a_{t,\bu}\le b_{t,\bu}$ 
ensured from the definition of $a_{t,\bu}$ and $b_{t,\bu}$,
prove $[a_{t,\bu},b_{t,\bu}]\supseteq
[a_{t+1,\bu},b_{t+1,\bu}]$.

\textbf{Proof of (\ref{eq:Claim3}):}
As any non-empty closed convex set is equal to
the intersection of its supporting halfspaces~\cite[Theorem~4.5]{Brondsted1983}, 
the convex hull $\Conv(\{\by_{t,i}\}_{i\in[n]})$ of the set 
$\{\by_{t,i}\}_{i\in[n]}$ of blurred data points is represented as
\begin{align}
	\Conv(\{\by_{t,i}\}_{i\in[n]})
	=\bigcap_{\|\bu\|=1}\{\bv\in\bbR^d\mid a_{t,\bu}\le\bu^\top\bv\le b_{t,\bu}\}.
\end{align}
According to this representation and \eqref{eq:Claim1}, 
one has $\Conv(\{\by_{t,i}\}_{i\in[n]})\supseteq
\Conv(\{\by_{t+1,i}\}_{i\in[n]})$,
since, for any $\bv\in\Conv(\{\by_{t+1,i}\}_{i\in[n]})$, 
one has $\bv\in\{\bv\in\bbR^d\mid a_{t+1,\bu}\le\bu^\top\bv\le b_{t+1,\bu}\}
\subseteq\{\bv\in\bbR^d\mid a_{t,\bu}\le\bu^\top\bv\le b_{t,\bu}\}$ 
for any $\bu\in\bbR^d$ with $\|\bu\|=1$.
Also, due to the initialization rule $\by_{1,1}=\bx_1,\ldots,\by_{1,n}=\bx_n$,
one has $\Conv(\{\bx_i\}_{i\in[n]})=\Conv(\{\by_{1,i}\}_{i\in[n]})$.

\textbf{Proof of (\ref{eq:diamrate}):}
Take an arbitrary $\bu\in\bbR^d$ with $\|\bu\|=1$,
and let $c_{t,\bu}\coloneq(a_{t,\bu}+b_{t,\bu})/2$.
We consider two cases according to the signs of 
$\sum_{j=1}^n\bbI(\bu^\top\by_{t,j}\le c_{t,\bu})
-\sum_{j=1}^n\bbI(\bu^\top\by_{t,j}>c_{t,\bu})$.

We first consider the case
$\sum_{j=1}^n\bbI(\bu^\top\by_{t,j}\le c_{t,\bu})
-\sum_{j=1}^n\bbI(\bu^\top\by_{t,j}>c_{t,\bu})\ge0$.
Then, one has that, for every $i\in[n]$, 
\begin{align}
\label{eq:Cheng3-3}
\begin{split}
	b_{t,\bu}-\bu^\top\by_{t+1,i}
	&=\frac{\sum_{j=1}^n G_{t,i,j}(b_{t,\bu}-\bu^\top\by_{t,j})}{\sum_{j=1}^n G_{t,i,j}}
\\
	&\ge\frac{\sum_{j=1}^n g((d_t/h)^2/2) (b_{t,\bu}-\bu^\top\by_{t,j})}{\sum_{j=1}^n g(0)}
	\quad(\because g((d_t/h)^2/2)\le G_{t,i,j}\le g(0),\; b_{t,\bu}-\bu^\top\by_{t,j}\ge0)
\\
	&\ge\frac{g((d_t/h)^2/2) \sum_{j=1}^n \bbI(\bu^\top\by_{t,j}\le c_{t,\bu})(b_{t,\bu}-\bu^\top\by_{t,j})}{n g(0)}
	\quad(\because 0\le \bbI(\bu^\top\by_{t,j}\le c_{t,\bu})\le 1)
\\
	&\ge\frac{g((d_t/h)^2/2) \sum_{j=1}^n \bbI(\bu^\top\by_{t,j}\le c_{t,\bu})(b_{t,\bu}-c_{t,\bu})}{n g(0)}
	\quad(\because\bbI(u\le v)u\le \bbI(u\le v)v, \forall u,v\in\bbR)
\\
	&\ge\frac{g((d_t/h)^2/2) (b_{t,\bu}-c_{t,\bu})}{2 g(0)}
	\quad(\because{\textstyle\sum_{j=1}^n\bbI(\bu^\top\by_{t,j}\le c_{t,\bu}})
	\ge\tfrac{n}{2})
\\
	&=\frac{g((d_t/h)^2/2) (b_{t,\bu}-a_{t,\bu})}{4 g(0)}
	\quad(\because b_{t,\bu}-c_{t,\bu}=\tfrac{b_{t,\bu}-a_{t,\bu}}{2}).
\end{split}
\end{align}
Therefore, one has
\begin{align}
\label{eq:Cheng3-4}
	\begin{split}
	b_{t+1,\bu}-a_{t+1,\bu}
	&\le b_{t+1,\bu}-a_{t,\bu}
	\quad(\because\eqref{eq:Claim1})
\\
	&=(b_{t,\bu}-a_{t,\bu})-(b_{t,\bu}-\max\{\bu^\top\by_{t+1,i}\}_{i\in[n]})
\\
	&\le(b_{t,\bu}-a_{t,\bu})-\frac{g((d_t/h)^2/2) (b_{t,\bu}-a_{t,\bu})}{4 g(0)}
	\quad(\because\eqref{eq:Cheng3-3})
\\
	&=\biggl(1-\frac{g((d_t/h)^2/2)}{4g(0)}\biggr)(b_{t,\bu}-a_{t,\bu}).
	\end{split}
\end{align}
In the remaining case $\sum_{j=1}^n\bbI(\bu^\top\by_{t,j}\le c_{t,\bu})-\sum_{j=1}^n\bbI(\bu^\top\by_{t,j}>c_{t,\bu})\le0$,
one can prove \eqref{eq:Cheng3-4} in a similar manner.
Thus, one proves \eqref{eq:diamrate} via letting 
$\bu=\bu_{t+1}\coloneq\arg\max_{\|\bu\|=1}(b_{t+1,\bu}-a_{t+1,\bu})$, as 
\begin{align}
\label{eq:devo}
  d_{t+1}
  =b_{t+1,\bu_{t+1}}-a_{t+1,\bu_{t+1}}
	\le\biggl(1-\frac{g((d_t/h)^2/2)}{4g(0)}\biggr)(b_{t,\bu_{t+1}}-a_{t,\bu_{t+1}})
	\le
  \biggl(1-\frac{g((d_t/h)^2/2)}{4g(0)}\biggr) d_t.
\end{align}

\textbf{Proof of the convergence of $(d_t)_{t\in\bbN}$:}
\eqref{eq:devo} is rewritten as $\frac{d_{t+1}}{d_t}\le1-\frac{g((d_t/h)^2/2)}{4g(0)}\le1$, 
which implies that $(d_t)_{t\in\bbN}$ is non-increasing.
Thus, since $g(u)$ is non-increasing in $u$, one has
$g((d_t/h)^2/2)\ge g((d_1/h)^2/2)$, and therefore 
$\frac{d_{t+1}}{d_t}\le1-\frac{g((d_1/h)^2/2)}{4g(0)}$,
implying that $(d_t)_{t\in\bbN}$ approaches 0 when $g((d_1/h)^2/2)>0$.

\textbf{Proof of the convergence of $(\by_t)_{t\in\bbN}$ and $(\by_{t,i})_{t\in\bbN}$'s:}
For any $\epsilon>0$, letting $\tau$ be a positive integer larger than 
$\log(\frac{\epsilon}{\sqrt{n}d_1})/\log(1-\frac{g((d_1/h)^2/2)}{4g(0)})-1$
so that $\epsilon>\sqrt{n}d_1(1-\frac{g((d_1/h)^2/2)}{4g(0)})^{\tau-1}\ge \sqrt{n}d_\tau$,
one has $\|\by_t-\by_{t'}\|\le \sqrt{n}d_\tau<\epsilon$ for any $t,t'\in\bbN$ such that $t, t'\ge \tau$
since $\by_{t,i},\by_{t',i}\in\Conv(\{\by_{\tau,j}\}_{j\in[n]})$ and $\|\by_{t,i}-\by_{t',i}\|\le d_\tau$ for all $i\in[n]$.
This result implies that the configuration sequence $(\by_t)_{t\in\bbN}$
is a Cauchy sequence, and hence converges.
Also, the convergence of $(\by_t)_{t\in\bbN}$ implies 
the convergence of $(\by_{t,i})_{t\in\bbN}$, $i\in[n]$.
If convergent points of $(\by_{t,i})_{t\in\bbN}$, $i\in[n]$ were different,
$d_t$ does not converge to 0 as $t\to\infty$,
which is a contradiction.
Therefore, $(\by_{t,i})_{t\in\bbN}$'s converge to a single point.

\textbf{Proof of the convergence rate bound of $(\by_t)_{t\in\bbN}$ and $(\by_{t,i})_{t\in\bbN}$'s:}
Let $\bz_1\in\bbR^d$ denote the point to which
the blurred data point sequences
$(\by_{t,1})_{t\in\bbN},\ldots,(\by_{t,n})_{t\in\bbN}$ converge. 
$\bz_1\in\Conv(\{\by_{t,i}\}_{i\in[n]})$ implies $\|\by_{t,i}-\bz_1\|\le d_t$.
Thus, one can find that its convergence rate is at least the exponential rate:
$\|\by_{t,i}-\bz_1\|=O((1-\frac{g((d_1/h)^2/2)}{4g(0)})^t)$ for all $i\in[n]$,
which implies $\|\by_t-\bz_1\otimes\bm{1}_n\|=O((1-\frac{g((d_1/h)^2/2)}{4g(0)})^t)$.

\end{proof}

\subsection*{S1.2\quad Proof of Theorem~\protect\ref{thm:variant}}
\begin{proof}[{Proof of Theorem~\ref{thm:variant}}]
The assumption that the BMS graph $\calG_{\by_\tau}$
is closed implies that it consists of $M:=M_{\by_\tau}$
components which are complete.
If $M=1$, then one can directly apply Theorem~\ref{thm:Cheng}
to conclude the proof. 
In the following we assume $M\ge2$,
and let $\{\calG_{\by_\tau,m}\}_{m\in[M]}$ be
the components in $\calG_{\by_\tau}$ 
and $\calV_{\by_\tau,m}$ be the vertex set
of $\calG_{\by_\tau,m}$ for $m\in[M]$.

We claim: For any $m,m'\in[M]$ with $m\not=m'$, 
one has $\Conv(\{\by_{\tau,i}\}_{i\in\calV_{\by_\tau,m}})\cap\Conv(\{\by_{\tau,i}\}_{i\in\calV_{\by_\tau,m'}})=\emptyset$.
Once we admit it, then 
one has $\Conv(\{\by_{\tau+1,i}\}_{i\in\calV_{\by_\tau,m}})\cap\Conv(\{\by_{\tau+1,i}\}_{i\in\calV_{\by_\tau,m'}})=\emptyset$, 
since one has
$\Conv(\{\by_{\tau,i}\}_{i\in\calV_{\by_\tau,m}})\supseteq\Conv(\{\by_{\tau+1,i}\}_{i\in\calV_{\by_\tau,m}})$ 
from Proposition~\ref{prop:BMScomp}. 
By induction, one furthermore has
$\Conv(\{\by_{t,i}\}_{i\in\calV_{\by_\tau,m}})\cap\Conv(\{\by_{t,i}\}_{i\in\calV_{\by_\tau,m'}})=\emptyset$ for all $t\ge\tau$.
This implies that the time evolution of the blurred data points 
with indices in $\calV_{\by_\tau,m}$ 
is not affected by those with indices in $\calV_{\by_\tau,m}$, $m'\not=m$,
so that one can apply Theorem~\ref{thm:Cheng} in a component-wise
manner to conclude the proof. 

Therefore, all that remains is to prove that  
$\Conv(\{\by_{\tau,i}\}_{i\in\calV_{\by_\tau,m}})\cap\Conv(\{\by_{\tau,i}\}_{i\in\calV_{\by_\tau,m'}})=\emptyset$ 
for any $m,m'\in[M]$, $m\not=m'$, under the assumption that 
all the components $\calG_{\by_\tau,1},\ldots,\calG_{\by_\tau,M}$
of $\calG_{\by_\tau}$ are complete.
We prove this claim using Lemma~\ref{lem:MD} below: 
Assume to the contrary that 
$\Conv(\{\by_{\tau,i}\}_{i\in\calV_{\by_\tau,m}})\cap\Conv(\{\by_{\tau,i}\}_{i\in\calV_{\by_\tau,m'}})\not=\emptyset$ holds for some $m,m'$ with $m\not=m'$. 
Letting $\calP_1=\{\by_{\tau,i}\}_{i\in\calV_{\by_\tau,m}}$ 
and $\calP_2=\{\by_{\tau,i}\}_{i\in\calV_{\by_\tau,m'}}$,
the above counterfactual assumption is rewritten as
$\calC_1\cap\calC_2\not=\emptyset$,
with $\calC_k:=\Conv\calP_k$, $k=1,2$. 
On the other hand, since $\calG_{\by_\tau,m}$ and $\calG_{\by_\tau,m'}$
are complete by assumption, one has $L_1,L_2\le\beta h$
for $L_k:=\min_{\bu,\bv\in\calP_k}\|\bu-\bv\|$, $k=1,2$. 
Then from Lemma~\ref{lem:MD} one would have
$\min_{\bu\in\calP_1, \bv\in\calP_2}\|\bu-\bv\|<\beta h$,
which implies that there should be at least one edge
joining $\calG_{\by_\tau,m}$ and $\calG_{\by_\tau,m'}$. 
This, however, contradicts the assumption that
$\calG_{\by_\tau,m}$ and $\calG_{\by_\tau,m'}$ are components,
completing the proof.
\end{proof}


\begin{lemma}
\label{lem:MD}
Let $\calP_1,\calP_2$ be sets of points in $\bbR^d$,
and let $\calC_1,\calC_2$ be the convex hulls of $\calP_1,\calP_2$, respectively.
For $k=1,2$, let $L_k\coloneq\max_{\bu,\bv\in \calP_k}\|\bu-\bv\|\ge0$. 
Assume that $\calC_1\cap \calC_2\not=\emptyset$.
One then has $\min_{\bu\in \calP_1,\bv\in\calP_2}\|\bu-\bv\|=0$ when $L_1=L_2=0$
and $\min_{\bu\in \calP_1,\bv\in\calP_2}\|\bu-\bv\|<\sqrt{(L_1^2+L_2^2)/2}$ otherwise.
\end{lemma}
\begin{proof}[Proof of Lemma~\ref{lem:MD}]
When $L_1=L_2=0$, all points in $\calP_1, \calP_2$ are the same,
and thus $\min_{\bu\in \calP_1,\bv\in\calP_2}\|\bu-\bv\|=0$ holds trivially.
In the following, we consider the other case, where $L_1>0$ or $L_2>0$ holds.
  Since we have assumed $\calC_1\cap \calC_2\not=\emptyset$,
  there exists $\bw\in \calC_1\cap \calC_2$.
  For $k=1,2$, $\bw\in\calC_k$ can be written as the convex combination
  of at most $(d+1)$ extreme points in $\calP_k$
  according to the following lemma:

\begin{lemma}[Carath\'{e}odory's theorem]
\label{lem:CT}
  If a point $\bu$ lies in the convex hull $\calC$ of a set $\calP\subseteq\bbR^d$,
  then $\bu$ can be written as the convex combination of at most
  $(d+1)$ extreme points in $\calP$. 
\end{lemma}

  Let $\bar{\calP}_1\coloneq\{\bu_1,\ldots,\bu_{n_1}\}\subseteq \calP_1$
  and $\bar{\calP}_2\coloneq\{\bv_1,\ldots,\bv_{n_2}\}\subseteq \calP_2$
  be the sets of these points, where $n_1,n_2\in[d+1]$ denote
  the numbers of the points in $\bar{\calP}_1$ and $\bar{\calP}_2$, respectively.
  One can then let
  \begin{align}
    \bw=\sum_{i=1}^{n_1}\lambda_i\bu_i
    =\sum_{j=1}^{n_2}\mu_j\bv_j
  \end{align}
  with $\lambda_i,\mu_j>0$ and $\sum_{i=1}^{n_1}\lambda_i=\sum_{j=1}^{n_2}\mu_j=1$.

  By assumption, we have $\|\bu_i-\bu_j\|^2\le L_1^2$
  for $\bu_i,\bu_j\in\bar{\calP}_1$. 
  Multiplying it with $\lambda_i\lambda_j>0$ and summing up that result over
  $i,j\in[n_1]$, one has 
  \begin{align}
    \label{eq:ineq1}
	\begin{split}
    2\sum_{i=1}^{n_1}\lambda_i\|\bu_i\|^2-2\|\bw\|^2
    &=\sum_{i,j=1}^{n_1}\lambda_i\lambda_j\|\bu_i-\bu_j\|^2\\
    &=\sum_{i=1}^{n_1}\lambda_i\sum_{j\not=i}\lambda_j\|\bu_i-\bu_j\|^2\\
    &\le L_1^2\sum_{i=1}^{n_1}\lambda_i(1-\lambda_i)\\
    &=L_1^2(1-\|\bm{\lambda}\|^2),
	\end{split}
  \end{align}
  where $\bm{\lambda}\coloneq(\lambda_1,\ldots,\lambda_{n_1})^\top$. 
  Similarly, since $\|\bv_i-\bv_j\|^2\le L_2^2$ 
  for $\bv_i,\bv_j\in\bar{\calP}_2$, 
  one has
  \begin{align}
    \label{eq:ineq2}
    2\sum_{j=1}^{n_2}\mu_j\|\bv_j\|^2-2\|\bw\|^2
	\le L_2^2(1-\|\bm{\mu}\|^2),
  \end{align}
  where $\bm{\mu}\coloneq(\mu_1,\ldots,\mu_{n_2})^\top$. 

  Assume to the contrary that for any $i\in[n_1]$ and
  $j\in[n_2]$ one had $\|\bu_i-\bv_j\|^2\ge \frac{L_1^2+L_2^2}{2}$.
  Multiplying it with $\lambda_i\mu_j>0$ and summing up that result over
  $i\in[n_1]$ and $j\in[n_2]$, one would have 
  \begin{align}
    \sum_{i=1}^{n_1}\lambda_i\|\bu_i\|^2
    +\sum_{j=1}^{n_2}\mu_j\|\bv_j\|^2
    -2\|\bw\|^2
	=\sum_{i=1}^{n_1}\sum_{j=1}^{n_2}\lambda_i\mu_j\|\bu_i-\bv_j\|^2
	\ge \frac{L_1^2+L_2^2}{2}.
  \end{align}
  Combining it with \eqref{eq:ineq1} and \eqref{eq:ineq2},
  one would obtain $L_1^2\|\bm{\lambda}\|^2+L_2^2\|\bm{\mu}\|^2\le0$,
  which is a contradiction
since $L_1>0$ or $L_2>0$ and $\|\bm{\lambda}\|^2,\|\bm{\mu}\|^2>0$.
\end{proof}

\section*{S2\quad Proof of Theorems~\protect\ref{thm:BMS-GCG}, \protect\protect\ref{thm:BMS-CG}, and \protect\ref{thm:BMS-CRE-GEN}}
\label{sec:Ape-BMS-Loj}
\subsection*{S2.1\quad Technical Lemmas}
We first provide two technical lemmas used along with Lemma \ref{lem:A}
to prove Theorems~\ref{thm:BMS-GCG}, \ref{thm:BMS-CG}, and \ref{thm:BMS-CRE-GEN}:

\begin{lemma}
\label{lem:B}
Assume Assumption~\ref{asm:RS}.
\begin{itemize}
\item
Assume furthermore Assumption \hyl{b1} in Theorem~\ref{thm:BMS-GCG}
(i.e., that there exists $\tau\in\bbN$ such that
the function $L$ is differentiable on $\cl(\Conv(\{\by_t\}_{t\ge\tau}))$). 
Then there exists $\bar{b}\in(0,\infty)$ such that
\begin{align}
\label{eq:lem-b}
	\|\by_{t+1}-\by_t\|\ge\bar{b}\|\nabla L(\by_t)\|
\end{align}
holds for any $t\ge\tau$.
\item
Instead, assume further Assumption~\ref{asm:LCG}
(i.e., that the kernel $K$ is differentiable). 
Then there exists $\bar{b}\in(0,\infty)$ such that \eqref{eq:lem-b} holds for any $t\in\bbN$.
\end{itemize}
\end{lemma}

\begin{proof}[{Proof of Lemma~\ref{lem:B}}]
Equation \eqref{eq:BMS-all} shows
\begin{align}
	\label{eq:BMS-all3}
	\nabla L(\by_t)
	=\frac{2}{h^2}\bS_t(\by_{t+1}-\by_t),
\end{align}
which implies
\begin{align}
	\|\nabla L(\by_t)\|\le
	\frac{2}{h^2}\cdot\text{max-ev}(\bS_t)\|\by_{t+1}-\by_t\|,
\end{align}
where $\text{max-ev}(\bA)$ is the maximum eigenvalue of a symmetric matrix $\bA$.
Since $\bS_t$ is a diagonal matrix, 
eigenvalues of $\bS_t$ are values of its diagonal elements.
Therefore, one has 
\begin{align}
	\text{max-ev}(\bS_t)	
	=\max\biggl\{\sum_{j=1}^n G_{t,i,j}\biggr\}_{i\in[n]}
	\le n g(0),
\end{align}
and hence
\begin{align}
\label{eq:select-b}
	\|\by_{t+1}-\by_t\|
	\ge\bar{b}\|\nabla L(\by_t)\|
	\text{ with }\bar{b}=\frac{h^2}{2n g(0)}
	\in(0,\infty)
	\quad(\because\text{Proposition~\ref{prop:KG} \hyl{a3}}).
\end{align}
\end{proof}

\begin{lemma}
\label{lem:C}
Assume Assumption~\ref{asm:RS}.
\begin{itemize}
\item
Assume furthermore Assumption \hyl{b1} in Theorem~\ref{thm:BMS-GCG}
(i.e., that there exists $\tau\in\bbN$ such that
the function $L$ is differentiable and has a Lipschitz-continuous 
gradient on $\cl(\Conv(\{\by_t\}_{t\ge\tau}))$).
Then there exists $\bar{c}\in(0,\infty)$ such that
\begin{align}
\label{eq:lem-c}
	\|\by_{t+1}-\by_t\|\ge\bar{c}\|\nabla L(\by_{t+1})\|
\end{align}
holds for any $t\ge\tau$.
\item
Instead, assume further Assumption~\ref{asm:LCG}
(i.e., that the kernel $K$ is differentiable and has a Lipschitz-continuous gradient).
Then there exists $\bar{c}\in(0,\infty)$ such that \eqref{eq:lem-c} holds for any $t\in\bbN$.
\end{itemize}
\end{lemma}

\begin{proof}[{Proof of Lemma~\ref{lem:C}}]
With the Lipschitz constant $C\ge0$ of $\nabla L$, 
one can find the relation
\begin{align}
\label{eq:select-c}
	\begin{split}
	\|\nabla L(\by_{t+1})\|
	&\le\|\nabla L(\by_t)\|+\|\nabla L(\by_{t+1})-\nabla L(\by_t)\|
	\quad(\because\text{Triangle inequality})\\
	&\le\frac{1}{\bar{b}}\|\by_{t+1}-\by_t\|+C\|\by_{t+1}-\by_t\|
	\quad(\because\text{Lemma \ref{lem:B}, \hyl{b1}})\\
	&=\frac{1}{\bar{c}}\|\by_{t+1}-\by_t\|
	\text{ with }\bar{c}=\biggl(\frac{1}{\bar{b}}+C\biggr)^{-1}\in(0,\infty).
	\end{split}
\end{align}
Also, under Assumption \ref{asm:LCG}, one has that, 
with the Lipschitz constant $C'$ of $\nabla K$,
\begin{align}
\label{eq:GRALIP}
	\begin{split}
	\|\nabla L(\bu)-\nabla L(\bu')\|
	&=\left\|\sum_{i,j=1}^n
	\begin{pmatrix}\vdots\\
	\tfrac{\partial}{\partial \bu_i} K\bigl(\tfrac{\bu_i-\bu_j}{h}\bigr)\\\vdots\\
	\tfrac{\partial}{\partial \bu_j} K\bigl(\tfrac{\bu_i-\bu_j}{h}\bigr)\\\vdots
	\end{pmatrix}
	-\sum_{i,j=1}^n
	\begin{pmatrix}\vdots\\
	\tfrac{\partial}{\partial \bu_i} K\bigl(\tfrac{\bu'_i-\bu'_j}{h}\bigr)\\\vdots\\
	\tfrac{\partial}{\partial \bu_j} K\bigl(\tfrac{\bu'_i-\bu'_j}{h}\bigr)\\\vdots
	\end{pmatrix}\right\|
\\
	&\le\sum_{i,j=1}^n
	\left\|\begin{pmatrix}
	\tfrac{\partial}{\partial \bu_i} K\bigl(\tfrac{\bu_i-\bu_j}{h}\bigr)-\tfrac{\partial}{\partial \bu_i} K\bigl(\tfrac{\bu'_i-\bu'_j}{h}\bigr)\\
	\tfrac{\partial}{\partial \bu_j} K\bigl(\tfrac{\bu_i-\bu_j}{h}\bigr)-\tfrac{\partial}{\partial \bu_j} K\bigl(\tfrac{\bu'_i-\bu'_j}{h}\bigr)\\
	\end{pmatrix}\right\|
	\quad(\because\text{Triangle inequality})
\\
	&=\frac{\sqrt{2}}{h}\sum_{i,j=1}^n \biggl\|\nabla K\biggl(\frac{\bu_i-\bu_j}{h}\biggr)
	-\nabla K\biggl(\frac{\bu'_i-\bu'_j}{h}\biggr)\biggr\|
\\
	&=\frac{\sqrt{2}C'}{h^2}\sum_{i,j=1}^n \|(\bu_i-\bu'_i)-(\bu_j-\bu'_j)\|
	\quad(\because\text{Assumption \ref{asm:LCG}})
\\
	&\le\frac{\sqrt{2}C'}{h^2}\sum_{i,j=1}^n (\|\bu_i-\bu'_i\|+\|\bu_j-\bu'_j\|)
	\quad(\because\text{Triangle inequality})
\\
	&\le\frac{\sqrt{2}C'}{h^2}\sum_{i,j=1}^n (\|\bu-\bu'\|+\|\bu-\bu'\|)
	\quad(\because\text{$\bu_i$ is a part of $\bu$})
\\
	&=\frac{2\sqrt{2}n^2C'}{h^2}\|\bu-\bu'\|.
	\end{split}
\end{align}
Thus, one can set $C\le\frac{2\sqrt{2}n^2C'}{h^2}$
and $\frac{h^2}{2ng(0)+2\sqrt{2}n^2C'}\le\bar{c}\le\frac{h^2}{2n g(0)}$.
\end{proof}

\subsection*{S2.2\quad Proof of Theorem~\protect\ref{thm:BMS-GCG}}
\textbf{Proof strategy of Theorems~\ref{thm:BMS-GCG}, 
\ref{thm:BMS-CG}, and \ref{thm:BMS-CRE-GEN}:}
Theorems~\ref{thm:BMS-GCG} and \ref{thm:BMS-CRE-GEN}
can be proved by applying abstract convergence theorems,
\cite[Theorem 3.2]{attouch2013convergence} and 
\cite[Theorems 3.1 and 3.5]{frankel2015splitting},
and Theorem~\ref{thm:BMS-CG} can be proved 
as a corollary of Theorem~\ref{thm:BMS-GCG}.
\cite{yamasaki2023ms} provided counterparts
of Theorems~\ref{thm:BMS-GCG}, \ref{thm:BMS-CG}, 
and \ref{thm:BMS-CRE-GEN} for the MS algorithm.
Proofs of Theorems~\ref{thm:BMS-GCG}, \ref{thm:BMS-CG}, 
and \ref{thm:BMS-CRE-GEN} for the BMS algorithm,
which we describe below, are almost the same as 
the proofs of theorems for the MS algorithm in 
supplementary material of \cite{yamasaki2023ms}.

\textbf{Preliminaries on {\L}ojasiewicz inequality:}
Let $\delta\in(0,\infty]$, and let $\varphi:[0,\delta)\to[0,\infty)$ 
be a continuous concave function such that $\varphi(0)=0$ 
and $\varphi$ is continuously differentiable on $(0,\delta)$ 
with $\varphi'(u)>0$.
The concavity of $\varphi$ implies that $\varphi'$ is non-increasing
on $(0,\delta)$.
The {\L}ojasiewicz inequality \eqref{eq:Lojasiewicz-ineq} 
holds trivially with $(\bu',\bu)$ satisfying $f(\bu')-f(\bu)=0$.
Also, it is known that the {\L}ojasiewicz inequality \eqref{eq:Lojasiewicz-ineq} with $(\bu',\bu)$ 
such that $\bu\in\bar{\calU}(\bu',f,\calT,\epsilon,\delta)\coloneq\{\bv\in \calT\mid \|\bu'-\bv\|<\epsilon, f(\bu')-f(\bv)\in(0,\delta)\}$
is a special case of
\begin{align}
\label{eq:Lojasiewicz-ineq2}
	\varphi'(f(\bu')-f(\bu))\|\nabla f(\bu)\|\ge 1
	\text{ at }(\bu',\bu)\text{ such that }\bu\in\bar{\calU}(\bu',f,\calT,\epsilon,\delta)
\end{align}
with $\varphi(u)=\frac{u^{1-\theta}}{c(1-\theta)}$ where $c$ is a positive constant.
(One technical subtlety with this extended definition is that
we have excluded those $\bu$ with $f(\bu)=f(\bu')$ from
$\bar{\calU}(\bu',f,\calT,\epsilon,\delta)$, as those points would
make the left-hand side of~\eqref{eq:Lojasiewicz-ineq2} indeterminate.)
Note that
\cite{attouch2013convergence,frankel2015splitting}
call the function $\varphi$ a desingularizing function
because of its role in \eqref{eq:Lojasiewicz-ineq2}, where $\varphi\circ f$ is
in a sense resolving criticality of $f$ at $\bu'$.
Note also that for the choice $\varphi(u)=\frac{u^{1-\theta}}{c(1-\theta)}$, 
one has $\varphi'(u)=\frac{u^{-\theta}}{c}$,
recovering the original definition (Definition~\ref{def:Loj}) of
the {\L}ojasiewicz property. 
The following proof of the convergence of the configuration sequence $(\by_t)_{t\in\bbN}$ (Theorem \protect\ref{thm:BMS-GCG})
is not restricted to the specific choice $\varphi(u)=\frac{u^{1-\theta}}{c(1-\theta)}$ but 
holds with the general form \eqref{eq:Lojasiewicz-ineq2} of the {\L}ojasiewicz inequality.
The specific choice $\varphi(u)=\frac{u^{1-\theta}}{c(1-\theta)}$,
on the other hand, will help derive the convergence rate bounds 
in Theorem~\ref{thm:BMS-CRE-GEN}.

\textbf{Proof of Theorem~\ref{thm:BMS-GCG}:}
We here provide a proof of Theorem~\ref{thm:BMS-GCG}
on the ground of \cite[Theorem~3.2]{attouch2013convergence} 
and \cite[Theorem~3.1]{frankel2015splitting}.

\begin{proof}[{Proof of Theorem~\ref{thm:BMS-GCG}}]
The objective sequence $(L(\by_t))_{t\in\bbN}$ converges 
under Assumption~\ref{asm:RS} since it is 
a bounded non-decreasing sequence (Theorem~\ref{thm:COS}).
Also, $\by_t$ lies in the convex hull
$\Conv(\{\bx_i\}_{i\in[n]})^n$, which is a compact set. 
Thus, there exist an accumulation point $\tilde{\by}\in\Conv(\{\bx_i\}_{i\in[n]})^n$
of the configuration sequence $(\by_t)_{t\in\bbN}$ and a subsequence 
$(\by_{t'})_{t'\in \calN}$ of $(\by_t)_{t\in\bbN}$ (with $\calN\subseteq\bbN$)
that converges to the accumulation point $\tilde{\by}$ as $t'\to\infty$.
Also, $\tilde{\by}\in\cl(\Conv(\{\by_t\}_{t\ge\tau}))$ obviously holds for any $\tau\in\bbN$.
When there exists $t'\in \calN$ such that $L(\tilde{\by})=L(\by_{t'})$, 
Lemma \ref{lem:A} obviously shows the convergence of $(\by_t)_{t\in\bbN}$ to $\tilde{\by}$: 
Assume $\by_{t'+1}\not=\by_{t'}$.
One then has $L(\by_{t'+1})\ge L(\by_{t'})+\bar{a}\|\by_{t'+1}-\by_{t'}\|^2>L(\tilde{\by})$
since $L(\by_{t'})=L(\tilde{\by})$ and $\bar{a}>0$.
It then follows from the monotonicity of $(L(\by_t))_{t\in\mathbb{N}}$
that $L(\tilde{\by})=\lim_{t'\in N, t\to\infty}L(\by_{t'})>L(\tilde{\by})$,
which is a contradiction.
On the other hand, if $\by_{t'+1}=\by_{t'}$, then
one has $\by_t=\by_{t'}$ for any $t\ge t'$ and hence $\tilde{\by}=\by_{t'}$. 
%
We, therefore, consider in what follows the remaining case where $L(\tilde{\by})>L(\by_t)$ for all $t\in\bbN$. 
The assumption~\hyl{b2} ensures 
that there exists a positive constant $\epsilon$ 
such that 
the function $L$ satisfies the {\L}ojasiewicz inequality \eqref{eq:Lojasiewicz-ineq} 
at least with any $(\bu',\bu)=(\tilde{\by},\by)$ 
such that $\by\in \calU(\tilde{\by},L,\cl(\Conv(\{\by_s\}_{s\ge\tau})),\epsilon)$
with some integer $\tau$. 

As we want to use the general form~\eqref{eq:Lojasiewicz-ineq2}
of the {\L}ojasiewicz inequality, we have to further restrict
the region where the {\L}ojasiewicz inequality to hold
from $\calU(\tilde{\by},L,\cl(\Conv(\{\by_s\}_{s\ge\tau})),\epsilon)$
to $\bar{\calU}(\tilde{\by},L,\cl(\Conv(\{\by_s\}_{s\ge\tau})),\epsilon,\delta)$
in order to ensure that
$L(\tilde{\by})-L(\by)$ is in the domain $[0,\delta)$
of the desingularizing function $\varphi$. 
Denoting $l_t\coloneq L(\tilde{\by})-L(\by_t)>0$, the convergence of 
the objective sequence $(L(\by_t))_{t\in\bbN}$ and the definition of $\tilde{\by}$
imply that the sequence $(l_t)_{t\in\bbN}$ is positive, non-increasing,
and converging to 0 as $t\to\infty$. 
The facts, $\by_{t'}\to\tilde{\by}$ and $l_t\to0$, as well as the continuity of $\varphi$, 
imply the existence of a finite integer $\tau'\ge\tau$ in $\calN$ 
such that $l_t\in[0,\delta)$ holds for any $t\ge\tau'$, and
that the inequality 
\begin{align}
\label{eq:At4}
	\|\tilde{\by}-\by_{\tau'}\|+2\sqrt{\frac{l_{\tau'}}{\bar{a}}}+\frac{1}{\bar{a}\bar{c}}\varphi(l_{\tau'})<\epsilon
\end{align}
holds.
It should be noted that if the assumptions~\hyl{b1} and \hyl{b2}
hold with some $\tau\in\bbN$, they also hold with the above $\tau'$
since $\{\by_s\}_{s\ge\tau'}\subseteq\{\by_s\}_{s\ge\tau}$ with $\tau'\ge\tau$. 
Using the {\L}ojasiewicz property of the function $L$ on $\bar{\calU}(\tilde{\by},L,\cl(\Conv(\{\by_s\}_{s\ge\tau'})),\epsilon,\delta)$,
the inequality~\eqref{eq:At4}, and assumption \hyl{b1},
we prove below that the configuration sequence $(\by_t)_{t\in\bbN}$ does not endlessly wander 
and does converge to $\tilde{\by}$ and that $\tilde{\by}$ is a critical point of the function $L$.

\textbf{Two key claims:}
We will establish the following two claims 
for any $t\ge\tau'+1$, which are the key to proving Theorem~\ref{thm:BMS-GCG}. 
\begin{claim}
\label{claim:At6}
$\by_t$ satisfies
\begin{align}
\label{eq:At6}
\by_t\in\bar{\calU}(\tilde{\by},L,\cl(\Conv(\{\by_s\}_{s\ge\tau'})),\epsilon,\delta).
\end{align}
In other words, the {\L}ojasiewicz inequality~\eqref{eq:Lojasiewicz-ineq2} with $(\bu',\bu)=(\tilde{\by},\by_t)$ holds. 
\end{claim}
\begin{claim}
\label{claim:At7}
$\{\by_s\}_{s\in\{\tau',\ldots,t+1\}}$ satisfies 
\begin{align}
\label{eq:At7}
	\sum_{s=\tau'+1}^t\|\by_{s+1}-\by_s\|
	+\|\by_{t+1}-\by_t\|
	\le\|\by_{\tau'+1}-\by_{\tau'}\|
	+\frac{1}{\bar{a}\bar{c}}\{\varphi(l_{\tau'+1})-\varphi(l_{t+1})\}.
\end{align}
\end{claim}

\textbf{Auxiliary results:}
We here provide two auxiliary results to be used in the succeeding proof.
First, one has
\begin{align}
\label{eq:At8}
	\begin{split}
	\|\by_{\tau'+1}-\by_{\tau'}\|
	&\le\sqrt{\frac{l_{\tau'}-l_{\tau'+1}}{\bar{a}}}
	\quad(\because\text{Lemma~\ref{lem:A}})\\
	&\le\sqrt{\frac{l_{\tau'}}{\bar{a}}}
	\quad(\because l_{\tau'+1}\ge0).
	\end{split}
\end{align}
Second, we have the following auxiliary lemma,
which will be used in proving~\eqref{eq:At7}
from~\eqref{eq:At6} via making use of the {\L}ojasiewicz property. 
\begin{lemma}
\label{lem:lem}
If $\by_t$ with $t\ge\tau$ satisfies Claim~\ref{claim:At6}, that is, if $\by_t\in\bar{\calU}(\tilde{\by},L,\cl(\Conv(\{\by_s\}_{s\ge\tau'})),\epsilon,\delta)$ holds, then
\begin{align}
\label{eq:At9}
	2\|\by_{t+1}-\by_t\|
	\le\|\by_t-\by_{t-1}\|+\frac{1}{\bar{a}\bar{c}}\{\varphi(l_t)-\varphi(l_{t+1})\}.
\end{align}
\end{lemma}
\begin{proof}[Proof of Lemma~\ref{lem:lem}]
Since \eqref{eq:At9} holds trivially if $\by_t=\by_{t-1}$,
we consider the case $\by_t\neq\by_{t-1}$.
When $\by_t\in\bar{\calU}(\tilde{\by},L,\cl(\Conv(\{\by_s\}_{s\ge\tau'})),\epsilon,\delta)$,
the {\L}ojasiewicz inequality~\eqref{eq:Lojasiewicz-ineq2}
with $(\bu',\bu)=(\tilde{\by},\by_t)$ holds. 
Noting that $0<l_{t+1}\le l_t<\delta$ holds, one has 
\begin{align}
\label{eq:Direct}
	\begin{split}
	\varphi(l_t)-\varphi(l_{t+1})
	&=\int_{l_{t+1}}^{l_t}\varphi'(u)\,du\\ 
	&\ge\varphi'(l_t)(l_t-l_{t+1})
	\quad(\because\text{$\varphi'$ is positive and non-increasing})\\
	&\ge\varphi'(l_t)\bar{a}\|\by_{t+1}-\by_t\|^2
	\quad(\because\text{Lemma~\ref{lem:A}})\\
	&\ge\frac{1}{\|\nabla L(\by_t)\|}\bar{a}\|\by_{t+1}-\by_t\|^2
	\quad(\because\mbox{{\L}ojasiewicz inequality~\eqref{eq:Lojasiewicz-ineq2} with $(\bu',\bu)=(\tilde{\by},\by_t)$})\\
	&\ge\bar{a}\bar{c}\frac{\|\by_{t+1}-\by_t\|^2}{\|\by_t-\by_{t-1}\|}
	\quad(\because\text{Lemma~\ref{lem:C}}).
	\end{split}
\end{align}
The inequality $2\sqrt{\alpha\beta}\le\alpha+\beta$ for $\alpha,\beta\ge0$ yields
\begin{align}
	\begin{split}
	2\|\by_{t+1}-\by_t\|
	&=2\sqrt{\|\by_{t+1}-\by_t\|^2}\\
	&\le2\sqrt{\|\by_t-\by_{t-1}\|\frac{1}{\bar{a}\bar{c}}\{\varphi(l_t)-\varphi(l_{t+1})\}}
	\quad(\because\text{\eqref{eq:Direct}})\\
	&\le \|\by_t-\by_{t-1}\|+\frac{1}{\bar{a}\bar{c}}\{\varphi(l_t)-\varphi(l_{t+1})\}.
	\end{split}
\end{align}
This concludes the proof of Lemma~\ref{lem:lem}.
\end{proof}

\textbf{Proof that Claims~\ref{claim:At6} and \ref{claim:At7} hold for $t=\tau'+1$:}
Here we prove Claims~\ref{claim:At6} and \ref{claim:At7} for $t=\tau'+1$.
One has
\begin{align}
	\begin{split}
	\|\tilde{\by}-\by_{\tau'+1}\|
	&\le\|\tilde{\by}-\by_{\tau'}\|
	+\|\by_{\tau'+1}-\by_{\tau'}\|
	\quad(\because\text{Triangle inequality})\\
	&\le\|\tilde{\by}-\by_{\tau'}\|
	+\sqrt{\frac{l_{\tau'}}{\bar{a}}}
	\quad(\because\text{\eqref{eq:At8}})\\
	&<\epsilon
	\quad(\because\text{\eqref{eq:At4}}),
	\end{split}
\end{align}
which, together with $0<l_{\tau'+1}\le l_{\tau'}<\delta$,
implies \eqref{eq:At6} with $t=\tau'+1$,
proving Claim~\ref{claim:At6} for $t=\tau'+1$. 
Also, Claim~\ref{claim:At6} with $t=\tau'+1$ implies, via Lemma~\ref{lem:lem},
the inequality \eqref{eq:At9} with $t=\tau'+1$, which reads 
\begin{align}
	2\|\by_{\tau'+2}-\by_{\tau'+1}\|
	\le\|\by_{\tau'+1}-\by_{\tau'}\|
	+\frac{1}{\bar{a}\bar{c}}\{\varphi(l_{\tau'+1})-\varphi(l_{\tau'+2})\},
\end{align}
which is nothing other than~\eqref{eq:At7} with $t=\tau'+1$,
thereby proving Claim~\ref{claim:At7} for $t=\tau'+1$. 

\textbf{Proof that Claims~\ref{claim:At6} and \ref{claim:At7} hold for $t\ge\tau'+1$:}
Now that we have seen that Claim~\ref{claim:At7} holds for $t=\tau'+1$,
we next prove Claim~\ref{claim:At7} to hold for every $t\ge\tau'+1$ by induction. 
For this purpose, we prove Claims~\ref{claim:At6} and \ref{claim:At7} for $t=u+1$ under
the assumption that Claims~\ref{claim:At6} and \ref{claim:At7} hold for $t=u\ge\tau'+1$.
One has 
\begin{align}
\label{eq:Atabove}
	\begin{split}
	\|\tilde{\by}-\by_{u+1}\|
	&\le\|\tilde{\by}-\by_{\tau'}\|
	+\|\by_{\tau'+1}-\by_{\tau'}\|
	+\sum_{s=\tau'+1}^u\|\by_{s+1}-\by_s\|
	\quad(\because\text{Triangle inequality})\\
	&\le\|\tilde{\by}-\by_{\tau'}\|
	+2\|\by_{\tau'+1}-\by_{\tau'}\|
	+\frac{1}{\bar{a}\bar{c}}\{\varphi(l_{\tau'+1})-\varphi(l_{u+1})\}
	-\|\by_{u+1}-\by_u\|
	\quad(\because\eqref{eq:At7}\text{ with }t=u)\\
	&\le\|\tilde{\by}-\by_{\tau'}\|
	+2\|\by_{\tau'+1}-\by_{\tau'}\|
	+\frac{1}{\bar{a}\bar{c}}\varphi(l_{\tau'+1})
	\quad(\because\|\by_{u+1}-\by_u\|\ge0\text{ and }\varphi(l_{u+1})\ge0)\\
	&\le\|\tilde{\by}-\by_{\tau'}\|
	+2\sqrt{\frac{l_{\tau'}}{\bar{a}}}
	+\frac{1}{\bar{a}\bar{c}}\varphi(l_{\tau'})
	\quad(\because\text{\eqref{eq:At8} and }\varphi(l_{\tau'})\ge\varphi(l_{\tau'+1}))\\
	&<\epsilon
	\quad(\because\text{\eqref{eq:At4}}),
	\end{split}
\end{align}
which, together with Claim~\ref{claim:At6} for $t=u$ and $0<l_{u+1}\le l_u<\delta$,
implies Claim~\ref{claim:At6} to hold for $t=u+1$.
Also, this result ensures, via Lemma~\ref{lem:lem},
that \eqref{eq:At9} holds with $t=u+1$.
Adding \eqref{eq:At9} with $t=u+1$ to \eqref{eq:At7} with $t=u$
then implies that \eqref{eq:At7} holds with $t=u+1$,
proving Claim~\ref{claim:At7} to hold for $t=u+1$.
As Claims~\ref{claim:At6} and \ref{claim:At7} have been shown to hold for $t=\tau'+1$,
the above argument proves, by induction, that
Claim~\ref{claim:At7} holds for every $t\ge\tau'+1$.

\textbf{Claim~\ref{claim:At7} for every $t\ge\tau'+1$ implies convergence:}
From~\eqref{eq:At7}, one has for any $t\ge\tau'+1$ 
\begin{align}
	\begin{split}
	\sum_{s=\tau+1}^t\|\by_{s+1}-\by_s\|
	&\le\|\by_{\tau'+1}-\by_{\tau'}\|
	+\frac{1}{\bar{a}\bar{c}}\{\varphi(l_{\tau'+1})-\varphi(l_{t+1})\}
	-\|\by_{t+1}-\by_t\|\\
	&\le\|\by_{\tau'+1}-\by_{\tau'}\|
	+\frac{1}{\bar{a}\bar{c}}\varphi(l_{\tau'+1})
	\quad(\because\|\by_{t+1}-\by_t\|\ge0\text{ and }\varphi(l_{t+1})\ge0).
	\end{split}
\end{align}
Taking the limit $t\to\infty$ yields
\begin{align}
  \sum_{s=\tau'+1}^\infty\|\by_{s+1}-\by_s\|
  \le\|\by_{\tau'+1}-\by_{\tau'}\|
	+\frac{1}{\bar{a}\bar{c}}\varphi(l_{\tau'+1}),
\end{align}
which implies
\begin{align}
	\sum_{s=1}^\infty\|\by_{s+1}-\by_s\|
	=\sum_{s=1}^{\tau'}\|\by_{s+1}-\by_s\|
	+\sum_{s=\tau'+1}^\infty\|\by_{s+1}-\by_s\|
	=\sum_{s=1}^{\tau'}\|\by_{s+1}-\by_s\|
	+\|\by_{\tau'+1}-\by_{\tau'}\|
	+\frac{1}{\bar{a}\bar{c}}\varphi(l_{\tau'+1})
	<\infty.
\end{align}
This shows that the trajectory of $(\by_t)_{t\in\bbN}$ is of finite length,
which in turn implies 
that $(\by_t)_{t\in\bbN}$ converges. 
As the limit $\lim_{t\to\infty}\by_t$ is a unique accumulation point
of $(\by_t)_{t\in\bbN}$, it must be $\tilde{\by}$
since $\by_{t'}\to\tilde{\by}$.
Additionally, from Lemma~\ref{lem:B}, one has 
\begin{align}
	\sum_{s=\tau'}^\infty\|\nabla L(\by_s)\|
	\le\frac{1}{\bar{b}}\sum_{s=\tau'}^\infty\|\by_{s+1}-\by_s\|<\infty,
\end{align}
which implies $\lim_{t\to\infty}\|\nabla L(\by_t)\|=0$.
Since the gradient of the function $L$ is Lipschitz-continuous 
with a Lipschitz constant $L\ge0$ 
on $\cl(\Conv(\{\by_y\}_{t\ge\tau'}))$ 
due to the assumption~\hyl{b1}, one has that
\begin{align}
	\|\nabla L(\tilde{\by})\|
	\le\lim_{t\to\infty}\{\|\nabla L(\by_t)\|+\|\nabla L(\tilde{\by})-\nabla L(\by_t)\|\}	
	\le\lim_{t\to\infty}\{\|\nabla L(\by_t)\|+L\|\tilde{\by}-\by_t\|\}
	=0,
\end{align}
which implies that the limit $\tilde{\by}=\lim_{t\to\infty}\by_t$
is a critical point of $L$. 
\end{proof}

\subsection*{S2.3\quad Proof of Theorem~\protect\protect\ref{thm:BMS-CG} and Discussion}
Theorem~\ref{thm:BMS-CG} is a direct corollary of Theorem~\ref{thm:BMS-GCG}.
We described the reason why the conditions 
\hyl{b1} and \hyl{b2} of Theorem~\ref{thm:BMS-GCG} can be 
replaced with Assumptions~\ref{asm:LCG} and \ref{asm:LP}
in the main text of this paper.
Note that we cannot avoid assuming the smoothness of the kernel $K$:
\cite{bolte2007lojasiewicz, bolte2007clarke} provided 
an extension of the {\L}ojasiewicz property for non-smooth functions,
and succeeding studies such as \cite{attouch2009convergence, attouch2013convergence, 
noll2014convergence, frankel2015splitting, bolte2016majorization}
used it to construct abstract convergence theorems 
for various optimization algorithms.
Although we also attempted the convergence analysis according to such
a general framework that allows non-smooth objective functions,
we could not avoid the smoothness assumption,
\hyl{b1} or Assumptions~\ref{asm:LCG},
in the framework based on the {\L}ojasiewicz inequality
even for the BMS algorithm.
Such difficulty is also discussed in \cite[Section 6]{noll2014convergence, frankel2015splitting}.

\subsection*{S2.4\quad Proof of Theorem~\protect\ref{thm:BMS-CRE-GEN}}
\textbf{Preliminaries on {\L}ojasiewicz inequality:}
For a desingularizing function $\varphi(u)$, 
define $\Phi(u)$ to be a primitive function (indefinite integral)
of $-(\varphi')^2$.
For the specific choice of the desingularizing function
$\varphi(u)=\frac{u^{1-\theta}}{c(1-\theta)}$,
one has
\begin{align}
\label{eq:RES0}
	\varphi'(u)
	=\frac{u^{-\theta}}{c},\quad
	\Phi(u)
	=\begin{cases}
	-\frac{u^{1-2\theta}}{c^2(1-2\theta)}&\text{if }\theta\in[0,\frac{1}{2}),\\
	-\frac{\log(u)}{c^2}&\text{if }\theta=\frac{1}{2},\\
	-\frac{u^{1-2\theta}}{c^2(1-2\theta)}&\text{if }\theta\in(\frac{1}{2},1),
	\end{cases}\quad
	\Phi^{-1}(u)
	=\begin{cases}
	\exp(-c^2 u)&\text{if }\theta=\frac{1}{2},\\
	\{c^2(2\theta-1)u\}^{-\frac{1}{2\theta-1}}&\text{if }\theta\in(\frac{1}{2},1).
	\end{cases}
\end{align}
These functional forms will be used in proving Theorem~\ref{thm:BMS-CRE-GEN}.

\textbf{Proof of Theorem~\ref{thm:BMS-CRE-GEN}:}
Now, we provide a proof of Theorem~\ref{thm:BMS-CRE-GEN}, 
which is based on the proof of \cite[Theorem~3.5]{frankel2015splitting}.

\begin{proof}[{Proof of Theorem~\ref{thm:BMS-CRE-GEN}}]
In the proof of Theorem~\ref{thm:BMS-GCG} we have established the following facts: 
If there exists $t'\in\bbN$ such that $L(\bar{\by})=L(\by_{t'})$
then $\by_t=\bar{\by}$ for any $t\ge t'$, that is,
$(\by_t)_{t\in\bbN}$ converges in a finite number of iterations.
If otherwise, then 
there exists $\tau\in\bbN$ such that Claim~\ref{claim:At6} holds
for any $t\ge\tau$, that is,
$\by_t\in\bar{\calU}(\bar{\by},L,\cl(\Conv(\{\by_s\}_{s\ge\tau})),\epsilon,\delta)$ 
holds for any $t\ge\tau$,
or equivalently, the {\L}ojasiewicz inequality \eqref{eq:Lojasiewicz-ineq2}
with $(\bu',\bu)=(\bar{\by},\by_t)$ holds for any $t\ge\tau$. 
If $\by_t$ with $t\ge\tau$ satisfies Claim~\ref{claim:At6},
then one has
\begin{align}
\label{eq:INEQ2}
\begin{split}
	\Phi(l_{t+1})-\Phi(l_t)
	&=\int_{l_{t+1}}^{l_t}\{\varphi'(u)\}^2\,du
	\quad(\because\text{Definition of $\Phi$})\\
	&\ge\{\varphi'(l_t)\}^2 (l_t-l_{t+1})
	\quad(\because\text{$\varphi'$ is positive and non-increasing})\\
	&\ge\{\varphi'(l_t)\}^2 \bar{a}\|\by_{t+1}-\by_t\|^2
	\quad(\because\text{Lemma~\ref{lem:A}})\\
	&\ge\{\varphi'(l_t)\}^2 \bar{a} \{\bar{b}\|\nabla L(\by_t)\|\}^2
	\quad(\because\text{Lemma~\ref{lem:B}})\\
	&\ge\bar{a}\bar{b}^2
	\quad(\because\text{{\L}ojasiewicz inequality~\eqref{eq:Lojasiewicz-ineq2}
    with $(\bu',\bu)=(\bar{\by},\by_t)$}).
\end{split}
\end{align}
Now Claim~\ref{claim:At6} holds for any $t\ge \tau$, 
which implies
\begin{align}
\label{eq:INEQ3}
	\Phi(l_t)-\Phi(l_\tau)
	=\sum_{s=\tau}^{t-1}\{\Phi(l_{s+1})-\Phi(l_s)\}
	\ge\bar{a}\bar{b}^2(t-\tau-2).
\end{align}

We discuss the two cases $\theta\in[0,\frac{1}{2})$
and $\theta\in[\frac{1}{2},1)$ separately.

\textbf{The case $\theta\in[0,\frac{1}{2})$:}
We claim that in this case the algorithm converges in a finite number
of iterations.
If otherwise, the inequality~\eqref{eq:INEQ3} should hold for any $t\ge\tau$.
When we take the limit $t\to\infty$, 
the right-hand side of~\eqref{eq:INEQ3} goes to infinity, 
which contradicts the fact that the left-hand side remains finite, 
by noting that one has $\lim_{u\to0}\Phi(u)=0$
with $\Phi(u)=-\frac{u^{1-2\theta}}{c^2(1-2\theta)}$
and that $l_t\to0$ as $t\to\infty$.
This contradiction implies the finite-time convergence of $(\by_t)_{t\in\bbN}$. 

\textbf{The case $\theta\in[\frac{1}{2},1)$:}
We may suppose that $l_t>0$ holds for any $t\in\bbN$,
and so \eqref{eq:INEQ3} holds for any $t\ge\tau$.
Recalling the functional form of $\Phi(u)$ as in~\eqref{eq:RES0}, 
one has $\lim_{t\to\infty}\Phi(l_t)=\infty$
with $\theta\in[\frac{1}{2},1)$.
Assume $\Phi(l_\tau)\ge0$. (If it is not the case
one can always redefine $\tau$ to a larger value with which
$\Phi(l_\tau)\ge0$ is satisfied.) 
One then has $\Phi(l_t)\ge\bar{a}\bar{b}^2(t-\tau-2)+\Phi(l_\tau)\ge\bar{a}\bar{b}^2(t-\tau-2)$,
which allows us to obtain the convergence rate bound for $l_t=L(\bar{\by})-L(\by_t)$, 
namely,
\begin{align}
\label{eq:RES1}
	l_t\le\Phi^{-1}(\bar{a}\bar{b}^2(t-\tau-2)).
\end{align}
With the explicit form of $\Phi^{-1}$ given in~\eqref{eq:RES0}, 
one has 
\begin{align}
	l_t\le
	\begin{cases}
	\exp(-c^2\bar{a}\bar{b}^2(t-\tau-2))
	=O(q^{2t})
	&\text{if }\theta=\frac{1}{2},\\
	\{c^2(2\theta-1)\bar{a}\bar{b}^2(t-\tau-2)\}^{-\frac{1}{2\theta-1}}
	=O(t^{-\frac{1}{2\theta-1}})
	&\text{if }\theta\in(\frac{1}{2},1),
	\end{cases}
\end{align}
where $q=\exp(-c^2\bar{a}\bar{b}^2/2)\in(0,1)$.
For the convergence rate bound for $\|\by_t-\bar{\by}\|$,
we have 
\begin{align}
	\begin{split}
	\varphi(l_t)-\varphi(l_{t+1})
	&=\int_{l_{t+1}}^{l_t}\varphi'(u)\,du\\ 
	&\ge\varphi'(l_t)(l_t-l_{t+1})
	\quad(\because\text{$\varphi'$ is positive and non-increasing})\\
	&\ge\varphi'(l_t)\bar{a}\|\by_{t+1}-\by_t\|^2
	\quad(\because\text{Lemma~\ref{lem:A}})\\
	&\ge\varphi'(l_t)\bar{a}\|\by_{t+1}-\by_t\|\bar{b}\|\nabla L(\by_t)\|
	\quad(\because\text{Lemma~\ref{lem:B}})\\
	&\ge\bar{a}\bar{b}\|\by_{t+1}-\by_t\|
	\quad(\because\eqref{eq:Lojasiewicz-ineq2}
	\text{ since }\by_t\in\bar{\calU}(\bar{\by},L,\cl(\Conv(\{\by_s\}_{s\ge\tau})),\epsilon,\delta)),
	\end{split}
\end{align}
which in turn yields
\begin{align}
\label{eq:RES2}
	\|\by_t-\bar{\by}\|
	\le\sum_{s=t}^\infty \|\by_{s+1}-\by_s\|
	\le\frac{1}{\bar{a}\bar{b}}\sum_{s=t}^\infty\{\varphi(l_s)-\varphi(l_{s+1})\}
	\le\frac{1}{\bar{a}\bar{b}}\varphi(l_t)
	\le\frac{1}{\bar{a}\bar{b}}\varphi(\Phi^{-1}(\bar{a}\bar{b}^2(t-\tau-2)))
\end{align}
from \eqref{eq:RES1}.
According to the calculation \eqref{eq:RES0} for $\varphi(u)=\frac{u^{1-\theta}}{c(1-\theta)}$,
one can obtain the exponential-rate convergence when $\theta=\frac{1}{2}$
and polynomial-rate convergence when $\theta\in(\frac{1}{2},1)$:
\begin{align}
	\|\by_t-\bar{\by}\|
	\le\frac{\{\Phi^{-1}(\bar{a}\bar{b}^2(t-\tau-2))\}^{1-\theta}}{\bar{a}\bar{b}c(1-\theta)}
	=\begin{cases}
	\frac{\{\exp(-c^2\bar{a}\bar{b}^2(t-\tau-2))\}^{\frac{1}{2}}}{\bar{a}\bar{b}c(1-\theta)}
	=O(q^t)
	&\text{if }\theta=\frac{1}{2},\\
	\frac{\{\{c^2(2\theta-1)\bar{a}\bar{b}^2(t-\tau-2)\}^{-\frac{1}{2\theta-1}}\}^{1-\theta}}{\bar{a}\bar{b}c(1-\theta)}
	=O(t^{-\frac{1-\theta}{2\theta-1}})
	&\text{if }\theta\in(\frac{1}{2},1).
	\end{cases}
\end{align}
This concludes the proof for all the cases, \hyl{b1}, \hyl{b2}, and \hyl{b3}.
\end{proof}

\section*{S3\quad Review of Previous Work}
\subsection*{S3.1\quad Theorem 4 of \cite{cheng1995mean}}
\label{sec:Rev-Th4}
\cite[Theorem 4]{cheng1995mean} claimed finite-time convergence
of the BMS algorithm with an additional assumption.
The claim is reproduced in the following:
\begin{claim}[Theorem 4 of \cite{cheng1995mean}]
  If the blurred data points cannot move arbitrarily close to each other,
  and $G$ is either zero or larger than a fixed positive constant,
  then the BMS algorithm reaches a fixed point in finitely many iterations.
\end{claim}
One may be able to criticize this claim from at least two aspects.
First, as for the assumption ``the blurred data points cannot move
arbitrarily close to each other,'' it appears that
the finite-precision nature of the floating-point representations
on computers is supposed to be the cause that makes the assumption valid
(see the second paragraph of \cite[Section IV-D]{cheng1995mean}).
If it is the case, then one has to modify the BMS update rule
in such a way that it incorporates the rounding
due to the finite-precision representation.
The argument in~\cite{cheng1995mean}, however, does not include
such a modification of the BMS algorithm.
Second, one might alternatively be able to regard the claim
as being applicable to the original BMS algorithm without any
modifications with regard to rounding.
If it is the case, then one should argue under what conditions
the assumption holds true.
However, \cite{cheng1995mean} does not provide any argument about this,
which would make the significance of this claim quite obscure.

\subsection*{S3.2\quad Theorem 1 of \cite{chen2015convergence}}
\label{sec:Ape-Chen}
To cover situations where the sequences 
$(\by_{t,1})_{t\in\bbN},\ldots,(\by_{t,n})_{t\in\bbN}$ can converge to multiple points, 
\cite{chen2015convergence} also performed convergence analysis.
However, he assumed that the function $g$ is decreasing:
there, $g$ is positive under Assumption~\ref{asm:RS} (or his assumptions),
and Theorem~\ref{thm:Cheng} ensures the convergence to a single point.
Furthermore, his proof of \cite[Theorem~1]{chen2015convergence} has another flaw:
he argued that his proof for the one-dimensional case, 
which is based on an intermediate claim like 
``$g(|\frac{v_{t,j}-v_{t,i}}{h}|^2/2)\le g(|\frac{v_{t,j'}-v_{t,i}}{h}|^2/2)$ 
since $|v_{t,j}-v_{t,i}|>|v_{t,j'}-v_{t,i}|$ for all $i\in[n]$'' in 
\cite[ll.\;2--3 in p.\;164]{chen2015convergence}, 
is also valid for the multi-dimensional case, 
by letting each $v_{t,i}$ be a projection of $\by_{t,i}$ for some $i\in\calV_{\by_t,m}$ 
onto a supporting hyperplane through some extreme point 
of the convex hull $\Conv(\{\by_{t,i}\}_{i\in\calV_{\by_t,m}})$;
however, to attain his final claim in the multi-dimensional case, one has to prove 
``$g(\|\frac{\by_{t,j}-\by_{t,i}}{h}\|^2/2)\le g(\|\frac{\by_{t,j'}-\by_{t,i}}{h}\|^2/2)$ for all $i\in[n]$'', 
which does not hold in general.

\section*{S4\quad Demonstration of BMS-based Data Clustering}
Figure~\ref{fig:MSBMS} shows the demonstration 
of the BMS-based data clustering.
It demonstrates that 
once the BMS graph becomes closed,
one would only need a few more iterations
for the blurred data point sequences, represented in a 
floating-point format on a computer, to converge.

\begin{figure}
\centering%
\renewcommand{\arraystretch}{0.5}
\begin{tabular}{cc}
{\hskip-13pt}
{\scriptsize~~~Truncated-flat function $G$}&{\hskip-20pt}
{\scriptsize~~Truncated-quadratic function $G$}\\
{\hskip-13pt}
\begin{tabular}{cc}
\includegraphics[height=4cm, bb=0 0 551 528]{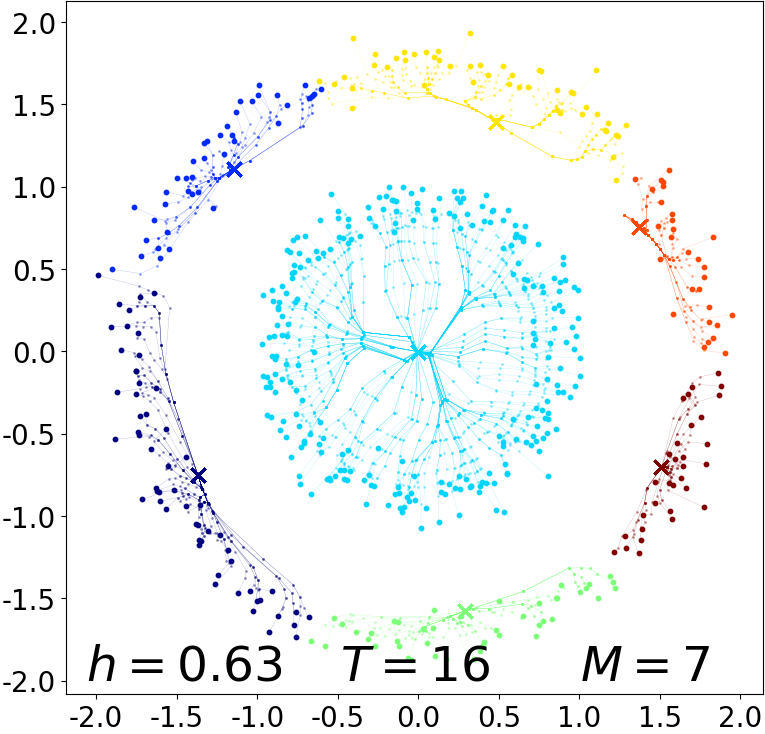}&{\hskip-13pt}
\includegraphics[height=4cm, bb=0 0 551 528]{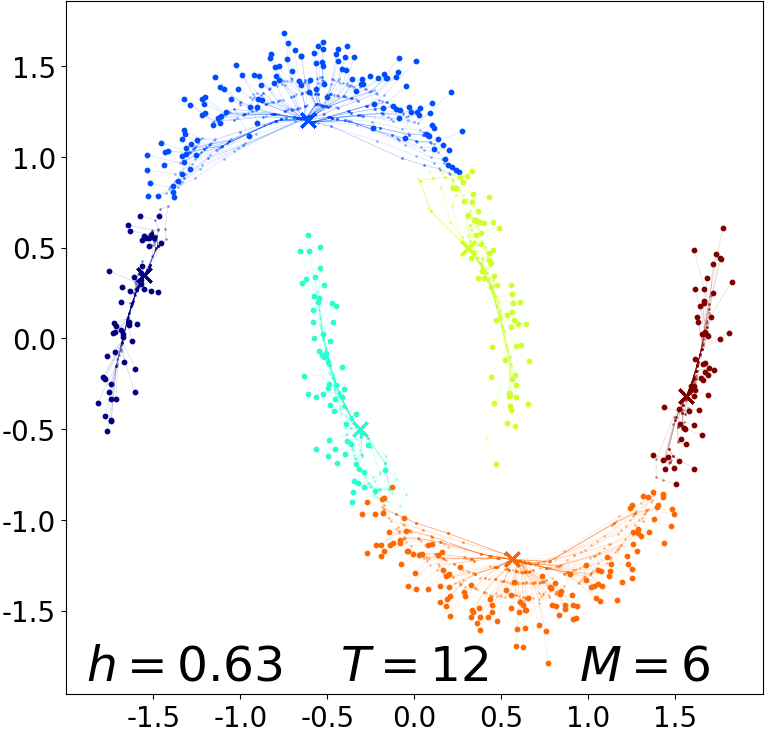}\\
{\scriptsize~~~noisy circles}&{\hskip-13pt}{\scriptsize~~~noisy moons}\\
\includegraphics[height=4cm, bb=0 0 551 528]{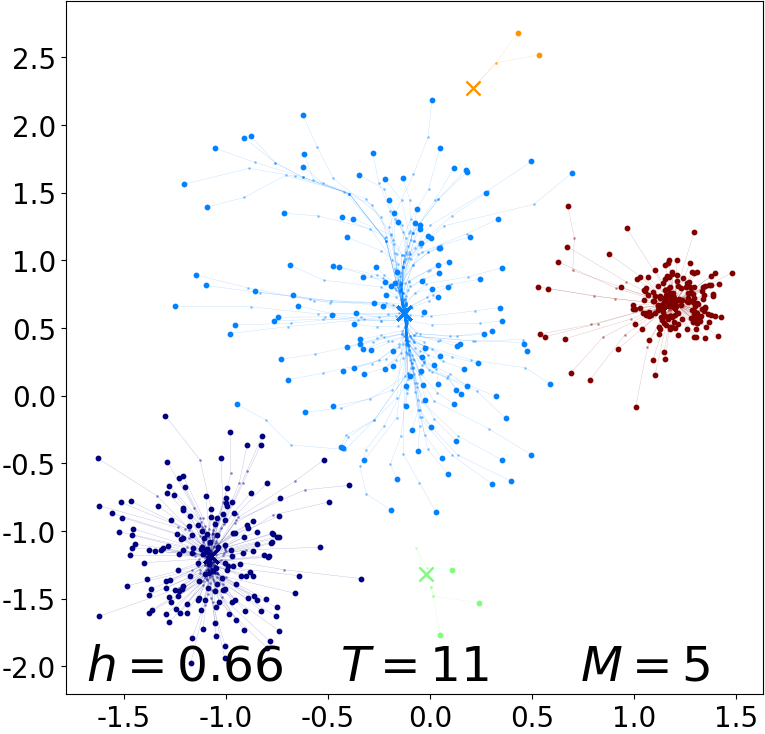}&{\hskip-13pt}
\includegraphics[height=4cm, bb=0 0 551 528]{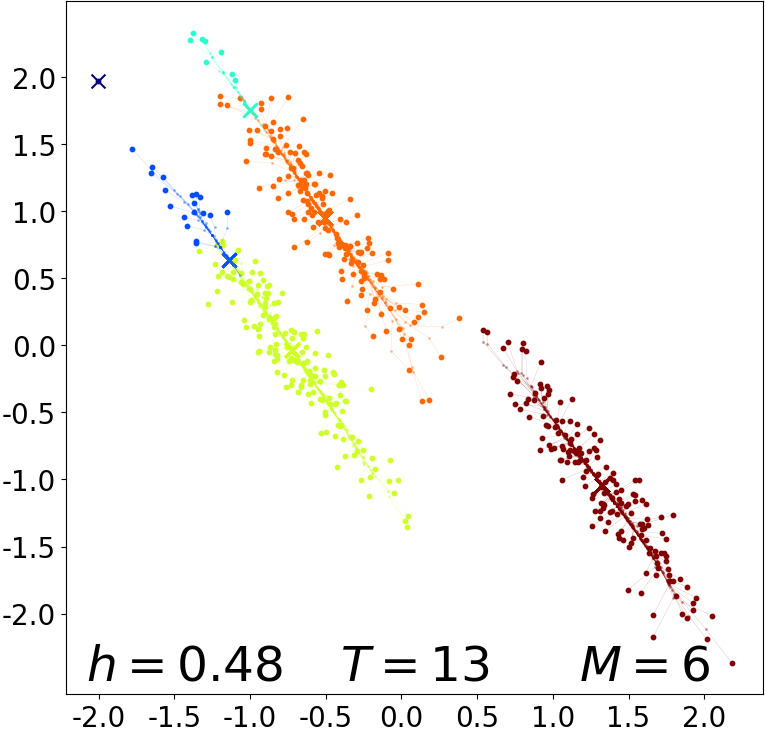}\\
{\scriptsize~~~varied}&{\hskip-13pt}{\scriptsize~~~aniso}\\
\includegraphics[height=4cm, bb=0 0 551 528]{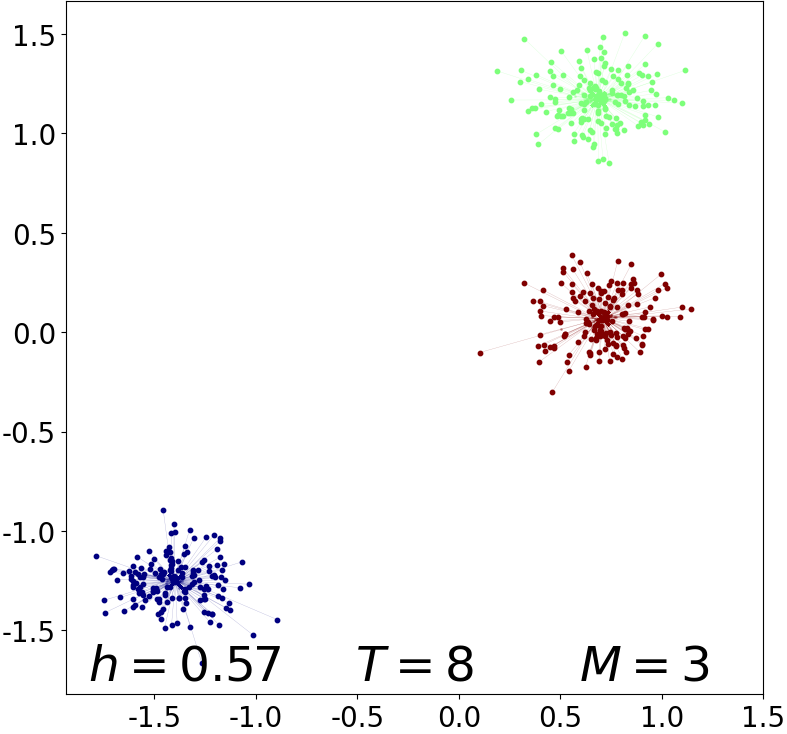}&{\hskip-13pt}
\includegraphics[height=4cm, bb=0 0 551 528]{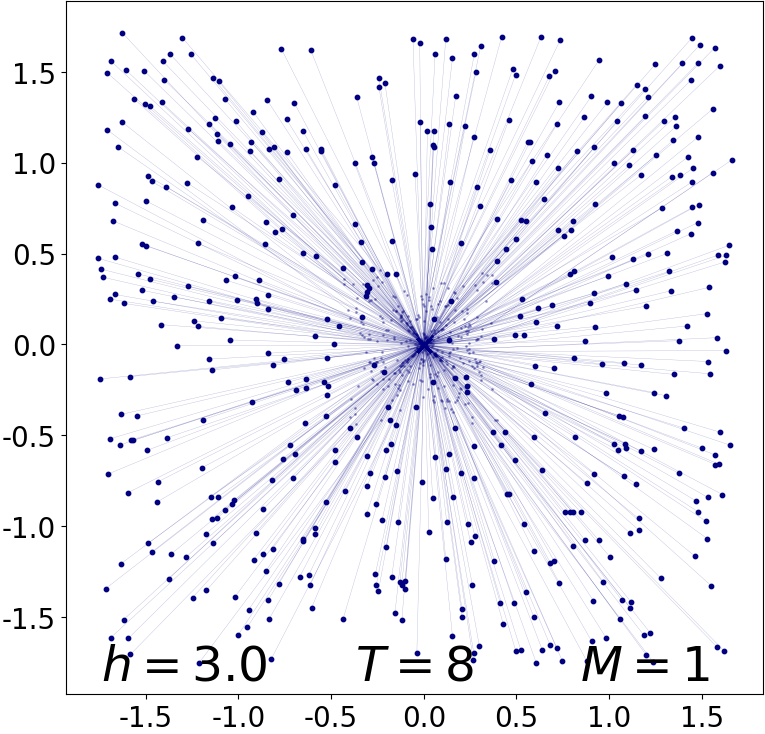}\\
{\scriptsize~~~blobs}&{\hskip-13pt}{\scriptsize~~~no structure}
\end{tabular}&{\hskip-20pt}
\begin{tabular}{cc}
\includegraphics[height=4cm, bb=0 0 551 528]{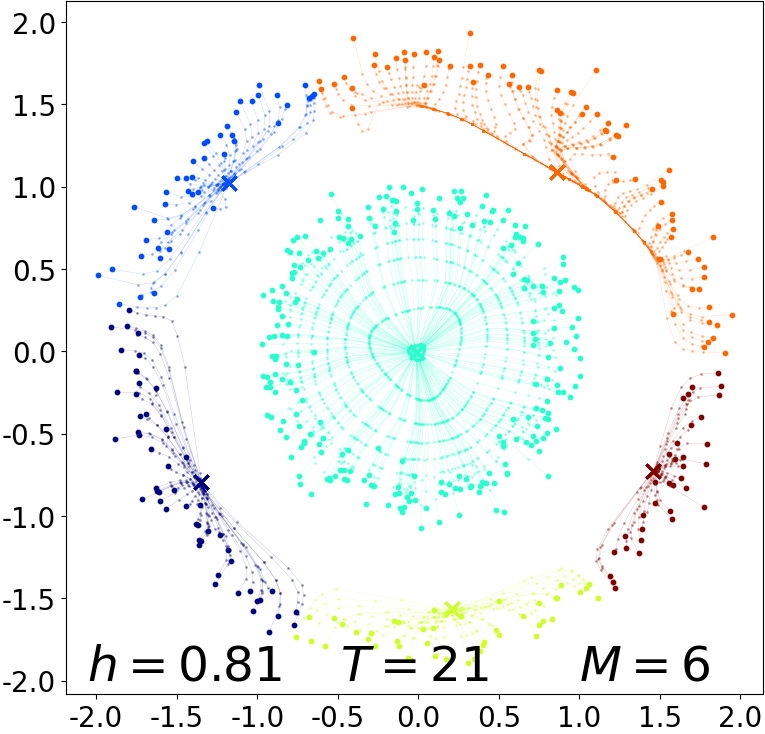}&{\hskip-13pt}
\includegraphics[height=4cm, bb=0 0 551 528]{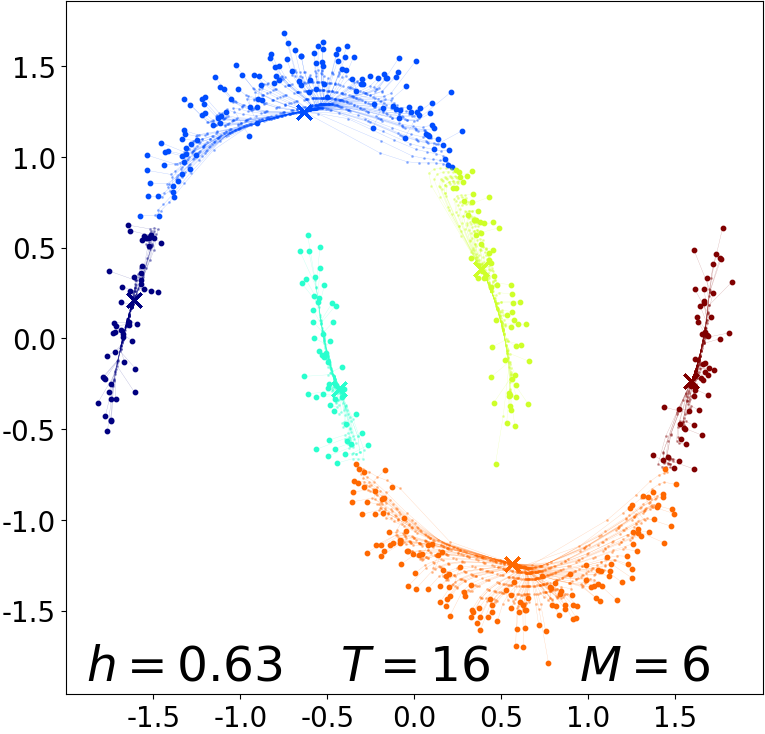}\\
{\scriptsize~~~noisy circles}&{\hskip-13pt}{\scriptsize~~~noisy moons}\\
\includegraphics[height=4cm, bb=0 0 551 528]{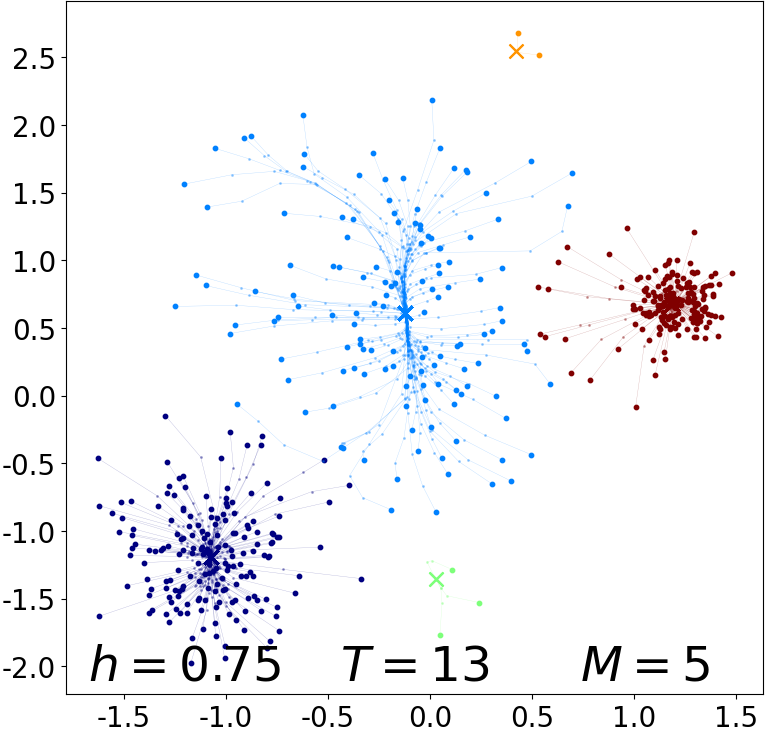}&{\hskip-13pt}
\includegraphics[height=4cm, bb=0 0 551 528]{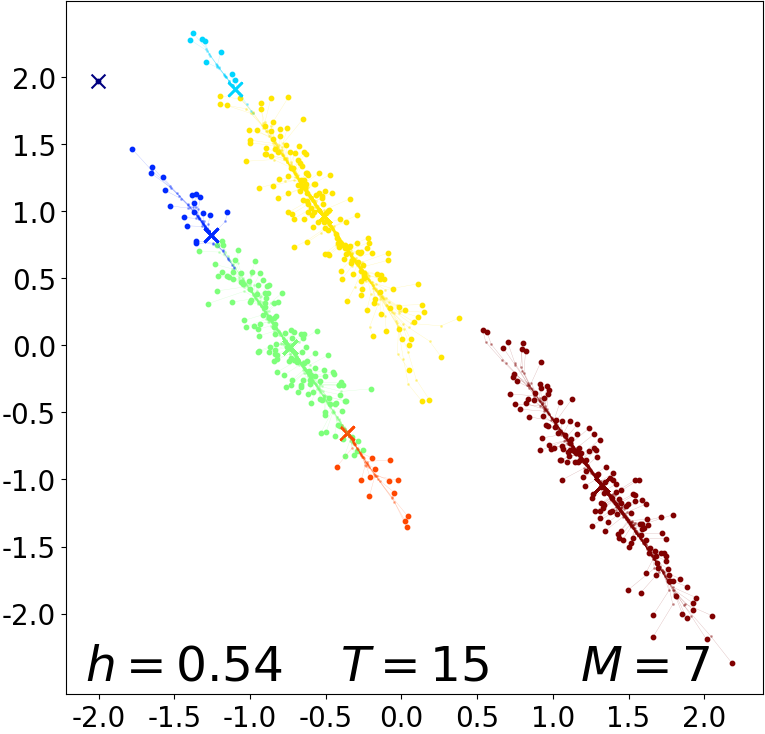}\\
{\scriptsize~~~varied}&{\hskip-13pt}{\scriptsize~~~aniso}\\
\includegraphics[height=4cm, bb=0 0 551 528]{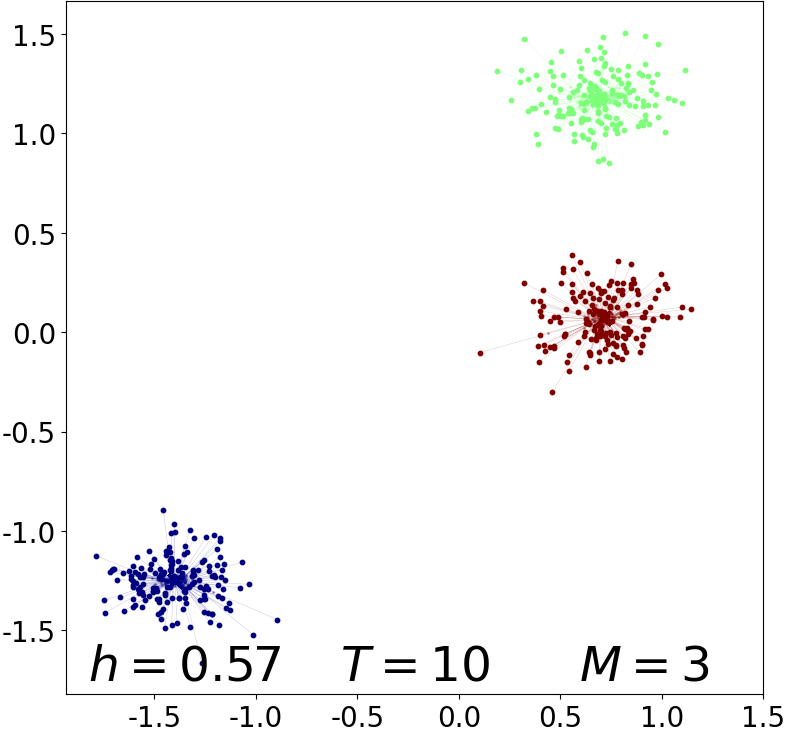}&{\hskip-13pt}
\includegraphics[height=4cm, bb=0 0 551 528]{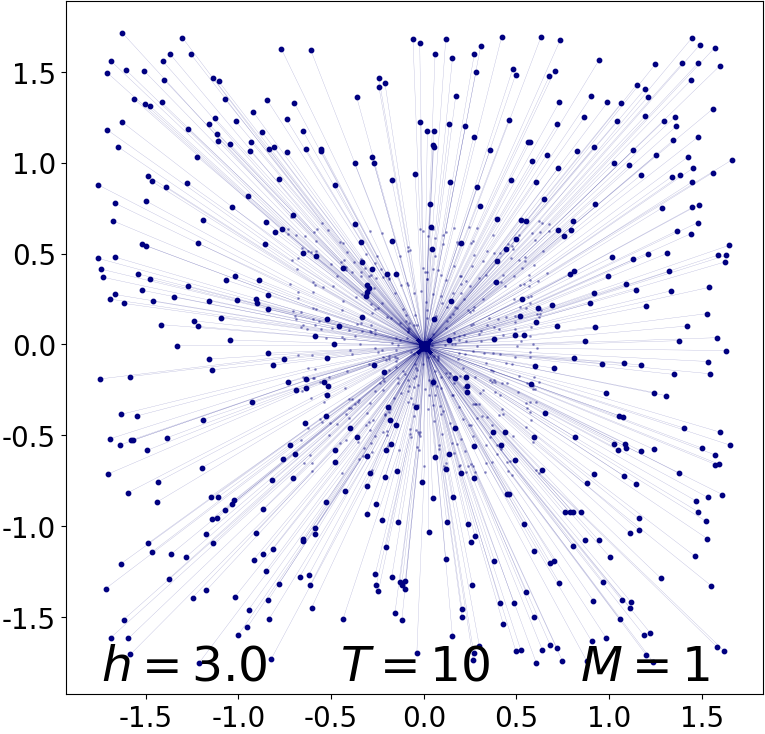}\\
{\scriptsize~~~blobs}&{\hskip-13pt}{\scriptsize~~~no structure}
\end{tabular}
\end{tabular}
\caption{%
Data clustering results by the BMS algorithm for six datasets in 
\url{https://scikit-learn.org/stable/modules/clustering.html}.
Each dataset has $n=500$ data points in 
$\bbR^2$ (i.e., $d=2$) and is standardized for use.
We used the double-precision floating-point number format,
the truncated-flat function $G(\bu)\propto\bbI(\|\bu\|\le1)$
and truncated-quadratic function $G(\bu)\propto(1-\|\bu\|^2)_+$,
and selected the bandwidth $h$,
which yielded the smallest number of clusters without clustering data 
from clearly different underlying clusters into the same cluster, 
among the candidates $\{0.03, 0.06, \ldots, 2.97, 3.0\}$
(note that clusters with close $\by_{T,i}$'s can be 
further integrated with proper post-processing).
Each plate shows $h$,
terminated step $T=\min\{t\in\bbN\mid\by_{t,i}=\by_{t+1,i}, \forall i\in[n]\}$,
and number $M$ of clusters.
It also shows initial points $\by_{1,i}$'s (data points $\bx_i$'s) by large dots, 
intermediate points $\by_{t,i}$'s by small dots, 
terminated points $\by_{T,i}$'s by crosses, 
and trajectories by polylines, 
in different colors corresponding to clusters.}
\label{fig:MSBMS}
\end{figure}

\section*{S5\quad Demonstration of Convergence Rate}
\label{sec:SIMP}
In this section, we demonstrate the cubic convergence of the BMS algorithm
under a specific finite-sample setting. 
The whole process of the BMS algorithm can be described explicitly 
when data points are located at the vertices of a regular simplex.
In this case, the blurred data points move toward the center
of the simplex, and every BMS update shrinks the simplex
by a certain scale factor. 
Assume that $n$ blurred data point $\by_{t,1},\ldots,\by_{t,n}\in\bbR^d$
with $n\le d+1$
are vertices of a $(n-1)$-dimensional regular simplex
centered at the origin, 
and $\by_{t,1}=(r_t/\sqrt{n},0,\ldots,0)^\top$ with $r_t>0$
so that $\|\by_{t,i}-\by_{t,j}\|=\sqrt{\frac{2}{n-1}}r_t$ for all $i,j\in[n]$ such that $i\neq j$
and $\|\by_t-\bm{0}_{n(n-1)}\|=r_t$.
Under this setting, the update of $\by_{t,i}$ becomes
\begin{align}
	\begin{split}
	\by_{t+1,i}
	&=\frac{\sum_{j=1}^n G\bigl(\frac{\by_{t,i}-\by_{t,j}}{h}\bigr)\by_{t,j}}
	{\sum_{j=1}^n G\bigl(\frac{\by_{t,i}-\by_{t,j}}{h}\bigr)}
\\
	&=\frac{g(0)\by_{t,i}+ g((r_t/h)^2/(n-1))\sum_{j\not=i}\by_{t,j}}
	{g(0)+(n-1)g((r_t/h)^2/(n-1))}
	\quad(\because \text{$\|\tfrac{\by_{t,i}-\by_{t,j}}{h}\|^2/2=(r_t/h)^2/(n-1)$ for $j\neq i$})
\\
	&=\frac{g(0)-g((r_t/h)^2/(n-1))}
	{g(0)+(n-1)g((r_t/h)^2/(n-1))}\by_{t,i}
	\quad(\because {\textstyle\sum_{j\not=i}\by_{t,j}=-\by_{t,i}}).
	\end{split}
\end{align}
This shows that the BMS update shrinks the $n$ blurred data points
toward the origin by the scale factor
\begin{equation}
  \frac{g(0)-g((r_t/h)^2/(n-1))}
	{g(0)+(n-1)g((r_t/h)^2/(n-1))}.
\end{equation}
It furthermore shows that when the data points are located
at the vertices of a regular simplex,
the blurred data points at any $t$ are also located at the vertices
of a simplex shrunk by a certain factor. 
Then, one has
\begin{align}
\label{eq:rre}
	r_{t+1}=\frac{g(0)-g((r_t/h)^2/(n-1))}{g(0)+(n-1)g((r_t/h)^2/(n-1))}r_t
	\text{ for all $t\in\bbN$}.
\end{align}
When $g((r_1/h)^2/(n-1))=0$, the BMS graph $\calG_{\by_1}$
consists of $n$ isolated vertices, so that one has $r_t=r_1$ for all $t$. 
When $g((r_1/h)^2/(n-1))>0$, on the other hand, the BMS graph $\calG_{\by_1}$
consists of a single component that is complete, so that
one can apply Theorem~\ref{thm:Cheng} to show the exponential rate convergence. 
Furthermore, under Assumption~\ref{asm:DDif}
(there exists a constant $c\in(0,\infty)$ such that 
$g$ satisfies $|g(u)-g(0)|\le c u$ for any $u\in[0,v]$),
if $(r_t/h)^2/(n-1)\le v$,
then from~\eqref{eq:rre} one can show 
\begin{align}
	r_{t+1}\le\frac{g(0)-\{g(0)-c (r_t/h)^2/(n-1)\}}{g(0)}r_t
	\le\frac{c}{(n-1)g(0)h^2}r_t^3
\end{align}
which implies the cubic convergence when $g((r_1/h)^2/(n-1))>0$.
The speed of the cubic convergence under this particular setting 
can be demonstrated by a simulation experiment
as shown in Figure~\ref{fig:SIMP}.

\begin{figure}
\centering%
\includegraphics[height=6cm, bb=0 0 735 385]{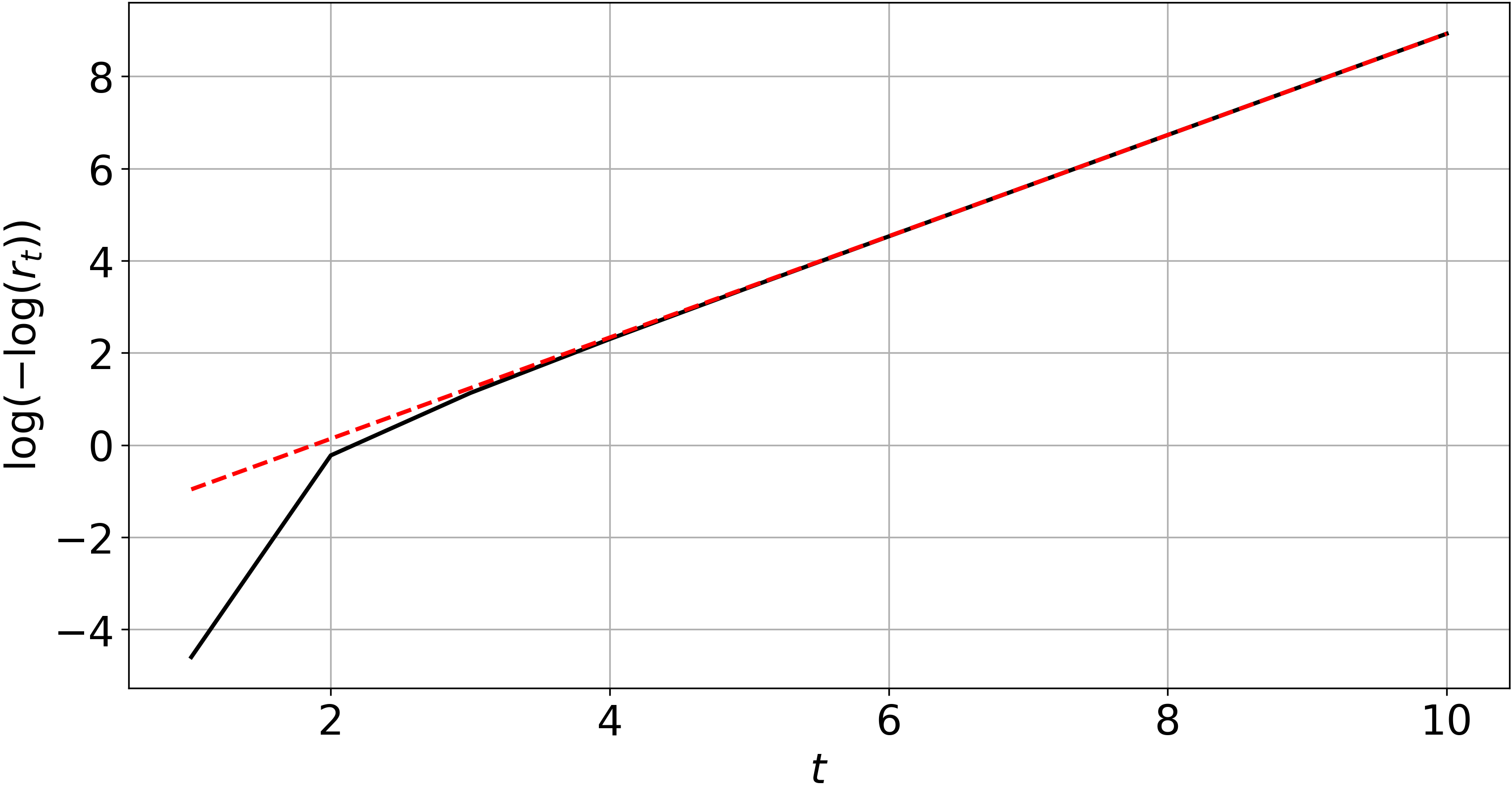}
\caption{%
Behavior of the sequence $(r_t)_{t\in[10]}$ under the setting 
discussed in Section~\protect\ref{sec:SIMP}
with the Gaussian kernel $K$, $n=2$, $h=1$, $r_1=0.99$,
computed with 5000-digit precision.
The black solid curve shows the simulation result, 
and the red dotted line shows the predicted slope $\log 3$
of the cubic convergence.}
\label{fig:SIMP}
\end{figure}

\end{document}